\documentclass[10pt,journal,compsoc]{IEEEtran}

\usepackage{amsmath,amssymb}
\usepackage{array}

\usepackage{theorem}
\usepackage{graphicx,psfrag}
\usepackage{enumerate,placeins,wrapfig,chngcntr}

\usepackage{diagbox}
\usepackage[normalem]{ulem}
\usepackage{cancel}

\usepackage{microtype}
\usepackage{subfigure}
\usepackage{booktabs} 
\usepackage{multirow}
\usepackage{threeparttable}
\usepackage{hyperref}

\RequirePackage{fancyhdr}
\RequirePackage[table]{xcolor}
\RequirePackage{algorithm}
\RequirePackage{algorithmic}

\RequirePackage{eso-pic} 
\RequirePackage{forloop}
\usepackage{epstopdf} 
\usepackage{bbding} 

\newtheorem{pro}{Proposition}
\newtheorem{defi}{Definition}
\newtheorem{asu}{Assumption}
\newtheorem{theorem}{Theorem}
\newtheorem{corollary}{Corollary}
\newtheorem{lemma}{Lemma}
\newtheorem{remark}{Remark}
\newenvironment{proof}{{\noindent\it Proof:}\quad}{
$\hfill \square$ \newline \par}

\def\D{\mathcal{D}}
\def\x{\mathbf{x}}
\def\y{\mathbf{y}}
\def\z{\mathbf{z}}
\def\p{\mathbf{p}}
\def\q{\mathbf{q}}
\def\Z{\mathbf{Z}}

\def\Y{\mathbb{R}^n}
\def\S{\mathcal{S}}
\def\X{\mathcal{X}}
\def\BB{\mathcal{B}}
\def\D{\mathcal{D}}

\def\c{\mathbf{c}}

\def \Z {\mathbf{Z}}

\def \u {\mathbf{u}}
\def \v {\mathbf{v}}

\def \PP {P}

\def \eta {\sigma^{(2)}}

\ifCLASSINFOpdf
\else
\fi
\begin{document}
%

\title{Value-Function-based Sequential Minimization for Bi-level Optimization}

%
%

\author{
	Risheng~Liu,~\IEEEmembership{Member,~IEEE,}
        Xuan~Liu,
        Shangzhi~Zeng,
        Jin~Zhang,
        and~Yixuan~Zhang 
\IEEEcompsocitemizethanks{
	\IEEEcompsocthanksitem R. Liu and X. Liu are with the DUT-RU International School of Information Science $\&$ Engineering, Dalian University of Technology, and the Key Laboratory for Ubiquitous Network and Service Software of Liaoning Province, Dalian, Liaoning, China. 
	R. Liu is also with the Pazhou Lab, Guangzhou, Guangdong, China.
	E-mail: rsliu@dlut.edu.cn, liuxuan\_16@126.com. 

\IEEEcompsocthanksitem S. Zeng is with 
the Department of Mathematics and Statistics, University of Victoria, Victoria, B.C., Canada.
E-mail: zengshangzhi@uvic.ca.

\IEEEcompsocthanksitem J. Zhang is with the Department of Mathematics, SUSTech International Center for Mathematics, Southern University of Science and Technology, 
National Center for Applied Mathematics Shenzhen, and Peng Cheng
Laboratory,
Shenzhen, Guangdong, China. (Corresponding author, E-mail:
zhangj9@sustech.edu.cn.)

\IEEEcompsocthanksitem Y. Zhang is with the Department of Applied Mathematics, the Hong Kong
Polytechnic University, Hong Kong, China. E-mail: \\
yi-xuan.zhang@connect.polyu.hk.


}

\thanks{Manuscript received April 19, 2005; revised August 26, 2015.}
}

%
%

\markboth{Journal of \LaTeX\ Class Files,~Vol.~14, No.~8, August~2015}%
{Shell \MakeLowercase{\textit{et al.}}: Bare Advanced Demo of IEEEtran.cls for IEEE Computer Society Journals}
%



\IEEEtitleabstractindextext{%
\begin{abstract}
	Gradient-based Bi-Level Optimization (BLO) methods have been widely applied to handle modern learning tasks. However, most existing strategies are theoretically designed based on restrictive assumptions (e.g., convexity of the lower-level sub-problem), and computationally not applicable for high-dimensional tasks. Moreover, there are almost no gradient-based methods able to solve BLO in those challenging scenarios, such as BLO with functional constraints and pessimistic BLO. In this work, by reformulating BLO into approximated single-level problems, we provide a new algorithm, named Bi-level Value-Function-based Sequential Minimization (BVFSM), to address the above issues. Specifically, BVFSM constructs a series of value-function-based approximations, and thus avoids repeated calculations of recurrent gradient and Hessian inverse required by existing approaches, time-consuming especially for high-dimensional tasks. We also extend BVFSM to address BLO with additional functional constraints. More importantly, BVFSM can be used for the challenging pessimistic BLO, which has never been properly solved before. In theory, we prove the \textcolor{black}{asymptotic} convergence of BVFSM on these types of BLO, in which the restrictive lower-level convexity assumption is  discarded. To our best knowledge, this is the first gradient-based algorithm that can solve different kinds of BLO (e.g., optimistic, pessimistic, and with constraints) with solid convergence guarantees. Extensive experiments verify the theoretical investigations and demonstrate our superiority on various real-world applications.

\end{abstract}

}

\maketitle

\IEEEdisplaynontitleabstractindextext

%
\IEEEpeerreviewmaketitle


\ifCLASSOPTIONcompsoc
\IEEEraisesectionheading{\section{Introduction}\label{sec:introduction}}
\else
\section{Introduction}
\label{sec:introduction}
\fi



\IEEEPARstart{C}{urrently}, 
{{a number of}} 	{important machine} {learning and} {deep learning tasks} can be captured by hierarchical models, 
such as hyper-parameter optimization \cite{liu2021value,franceschi2017forward,okuno2021lp,mackay2018self},
neural architecture search \cite{liu2018darts,liang2019darts+,chen2019progressive}, 
meta learning~\cite{franceschi2018bilevel,rajeswaran2019meta,zugner2018adversarial},
Generative Adversarial Networks (GAN) \cite{metz2016unrolled,pfau2016connecting}, 
reinforcement learning \cite{yang2019provably},
image processing \cite{liu2019convergence,ijcai2020-101,liu2020bilevel,liu2020investigating},
and so on.
In general, these hierarchical models can be formulated as the following
Bi-Level Optimization (BLO) problem \cite{dempe2020optimality,dempe2018bilevel,liu2021investigating}:
\begin{equation}
		``\min\limits_{\x \in \X} "  \  F(\x,\y),  \
		\mathrm{ \ s.t.\  } \y \in {\rm \S}(\x) 
		:=\mathop{\arg\min}_{\y} f(\x,\y),
	\label{eq:1}
\end{equation}
where $\x \in \X$ is the Upper-Level (UL) variable,
$\y \in \Y$ is the Lower-Level (LL) variable,
the UL objective $F(\x,\y):\X \times\mathbb{R}^n\rightarrow\mathbb{R}$
and the LL objective $f(\x,\y):\mathbb{R}^m\times\mathbb{R}^n\rightarrow\mathbb{R}$, 
are continuously differentiable and jointly continuous functions, 
and the UL constraint $\X \subset \mathbb{R}^m$ is a compact set. 
Nevertheless, the model in Eq.~\eqref{eq:1} cannot be solved directly.
Some existing works only consider the case that the LL solution set $\S(\x)$ is a singleton.
However, since this may not be satisfied and~$\y \in {\rm \S}(\x)$ may not be unique,
Eq.~\eqref{eq:1} is not a rigorous BLO model in mathematics, 
and thus we use the quotation marks around ``min" to denote the slightly imprecise definition of the UL objective \cite{dempe2020optimality,alves2019new}. 

Strictly, people usually focus on an extreme situation
of the BLO model,
i.e., the optimistic BLO \cite{dempe2020optimality}:
\begin{equation}
		\min\limits_{\x \in \X} 
		\min\limits_{\y \in \Y} \ F(\x,\y), \
		\mathrm{ \ s.t.\  } \y \in {\rm \S}(\x).
		\label{eq:OBO}
\end{equation}
It can be found from the above expression that in optimistic BLO, $\x$ and $\y$ are in a cooperative relationship, aiming to minimize $F(\x,\y)$ at the same time.
Therefore, it can be applied to a variety of learning and vision tasks, 
such as hyper-parameter optimization, meta learning, and so on. 
Sometimes we also need to study BLO problems with inequality constraints on the UL or LL for capturing constraints in real tasks.
Another situation one can consider is the pessimistic BLO, which changes the $\min_{\y \in \Y}$ in Eq.~\eqref{eq:OBO} 
into $\max_{\y \in \Y}$~\cite{dempe2020optimality}.
In the pessimistic case, $\x$ and $\y$ are in an adversarial relationship,
and hence solving pessimistic BLO can be applied to adversarial learning and GAN. 

Actually, BLO is challenging to solve,
because in the hierarchical structure, we need to solve $\S(\x)$ governed by the fixed $\x$,
and select an appropriate $\y$ from $\S(\x)$ to optimize the UL $F(\x,\y)$,
making $\x$ and $\y$ intricately dependent of each other, 
especially when $\S(\x)$ is not a singleton \cite{jeroslow1985polynomial}.
In classical optimization, KKT condition is utilized to characterize the problem,
but this method is not applicable to machine learning tasks of large scale due to the use of too many multipliers \cite{1982An,luo1996mathematical}.
In the machine learning community, a class of mainstream and popular methods are gradient-based methods, divided into Explicit Gradient-Based Methods (EGBMs) \cite{franceschi2017forward,franceschi2018bilevel,maclaurin2015gradient,shaban2019truncated,liu2018darts} 
and Implicit Gradient-Based Methods (IGBMs) \cite{pedregosa2016hyperparameter,rajeswaran2019meta,lorraine2020optimizing},
according to divergent ideas of calculating the gradient needed for implementing gradient descent.
EGBMs implement this process via unrolled differentiation,
and IGBMs use the implicit function theorem to obtain the gradient.
Both of them usually 
deal with the problem where the LL solution set $\S(\x)$ is a singleton,
which is a quite restrictive condition in real application tasks. 
In dealing with this, Liu et al.~\cite{liu2020generic,liu2022general} proposed Bi-level Descent Aggregation (BDA) as a new EGBM, which removes this assumption and solves the model from the perspective of optimistic BLO.

Nevertheless, there still exists a bottleneck hard to break through, that the LL problems in real learning tasks are usually too complex for EGBMs and IGBMs.
In theory, all of the EGBMs and IGBMs require 
the convexity of the LL problem, or the Lower-Level Convexity, denoted as LLC for short,
which is a 
\textcolor{black}{strong} condition and not satisfied in many complicated real-world tasks.
For example, 
since the layer of chosen network is usually greater than one,
LLC is not satisfied, so the convergence of these methods cannot be guaranteed.
In computation, additionally,
EGBMs using unrolled differentiation request large time and space complexity,
while IGBMs need to approximate the inverse of a matrix, also with high computational cost,
especially when the LL variable $\y$ is of large scale, which means the dimension of $\y$ is large, generating matrices and vectors of high dimension during the calculating procedure.
Furthermore, it has been rarely discussed how to handle machine learning tasks by solving an optimization problem with functional constraints on the UL and LL, or by solving a pessimistic BLO. 
However, these problems are worth discussing, because pessimistic BLO can be used to capture min-max bi-level structures, which is suitable for GAN and so on,
and optimization problems with constraints can be used to represent learning tasks more accurately.
Unfortunately, existing methods including EGBMs and IGBMs, are not able to handle these problems.

To address the above limitations of existing methods,
in this work, we propose a novel framework, named 
Bi-level Value-Function-based Sequential Minimization (BVFSM)
\footnote{A preliminary version of this work has been published in~\cite{liu2021value}.}.
To be specific, we start with reformulating BLO into a simple bi-level optimization problem by the value-function~\cite{outrata1990on,ye1995optimality} of UL objective.
After that, we further transform it into a single-level optimization problem with an inequality constraint through the value-function of LL objective.
Then, by using the smoothing technique via regularization 
and adding the constraint into the objective by an auxiliary function of penalty or barrier,
eventually the original problem can be transformed into a sequence of unconstrained differentiable single-level problems,
which can be solved by gradient descent.
Thanks to the re-characterization via the value-function of LL problem, our computational cost is the least to implement the algorithm, and simultaneously, 
BVFSM can be applied under more relaxed conditions.

Specifically, BVFSM avoids solving an unrolled dynamic system by recurrent gradient or approximating the inverse of Hessian during each iteration like existing methods.
Instead, we only need to calculate the first-order gradient in each iteration, reduces the computational complexity relative to the LL problem size by an order of magnitude compared to existing gradient-based BLO methods, and thus require less time and space complexity than EGBMs and IGBMs,
especially for complex high-dimensional BLO. 
\textcolor{black}{Besides, BVFSM enables to maintain the level of complexity when applying BLO to networks, thereby making it possible to use BLO in existing networks and expanding its range of applications significantly.
}
We illustrate the efficiency of BVFSM over existing methods through complexity analysis in theory and various experimental results in reality.
In addition, we \textcolor{black}{consider the asymptotic convergence different from some previous gradient-based methods inspired from the perspective of sequential minimization, and} prove that the solutions to the sequence of approximate sub-problems
converge to the true solution of the original BLO without the restrictive LLC assumption as before.
Also, BVFSM can be extended to more complicated and challenging scenarios, 
namely, BLO with functional constraints and pessimistic BLO problems.
We regard pessimistic BLO as a new viewpoint to deal with learning tasks, which has not been solved by gradient-based methods before to our best knowledge.
Specially, we use the experiment of GAN as an example to illustrate the application of our method for solving pessimistic BLO.
We summarize our contributions as follows.

\begin{itemize}
     \item By reformulating the original BLO as an approximated single-level problem based on the value-function, 
     \textcolor{black}{BVFSM breaks the traditional mindset in gradient-based methods, and establishes a competently new sequential minimization algorithmic framework,} 
     which not only can be used to address optimistic BLO, 
     but also has the ability to handle BLO in other more challenging scenarios (i.e., with functional constraints and pessimistic),
     which have seldom been discussed.
     

     \item \textcolor{black}{BVFSM significantly reduces the computational complexity by an order of magnitude compared to existing gradient-based BLO methods with the help of value-function-based reformulation which breaks the traditional mindset. 
     Also, BVFSM avoids the repeated calculation of recurrent gradients and Hessian inverse, which are the core bottleneck for solving high-dimensional BLO problems in existing approaches.
     The superiority allows BVFSM to be applied to large-scale networks and frontier tasks effectively. 
     }
     
     \item We rigorously analyze the \textcolor{black}{asymptotic} convergence behaviors of BVFSM on all types of BLO mentioned above.
     Our theoretical investigations successfully remove the restrictive LLC condition, 
     required in most existing works but actually too ambitious to satisfy in real-world applications. 
     
     \item In terms of experiments, we conduct extensive experiments to verify our theoretical findings and demonstrate the superiority of BVFSM on various learning tasks. 
     Especially, by formulating and solving GAN by BVFSM, we also show the application potential of our solution strategy on pessimistic BLO for  complex learning problems.
\end{itemize}

\section{Related Works}\label{sec:related works}
As aforementioned, BLO 
is challenging to solve due to its nested structures between UL and LL.
Early methods can only handle models with not too many hyper-parameters.
For example, to find appropriate parameters, the standard method is to use random search \cite{bergstra2012random} 
through randomly sampling, or to use Bayesian optimization \cite{hutter2011sequential}.  
However, in real learning tasks, the dimension of hyper-parameters is very large,
which early methods cannot deal with,
so gradient-based methods are proposed.
Here we first put forward a unified form of gradient-based methods,
and then discuss the existing methods for further comparing them with our proposed method.

Existing gradient-based methods mainly focus on the optimistic BLO only, 
so we use the optimistic scenario to illustrate our algorithmic framework clearly,
while in Section~\ref{sec:pessimistic},
we will discuss how to use our method to solve pessimistic BLO. 
For optimistic BLO, it can be found from Eq.~\eqref{eq:OBO} 
that the UL variable $\x$ and LL variable $\y$ will effect each other in a nested relationship.
To address this issue, one can transform it into the following form,
where $\varphi(\x)$ is the value-function of the sub-problem,
\begin{equation}
        \min\limits_{\x \in \X} \ \varphi(\x), \quad 
    		\varphi(\x):= \mathop{\min}_{\y} \Big\{  F(\x,\y) 
    		: \y \in {\rm \S}(\x)
        \Big\}.
	\label{eq:minvarphi}
\end{equation}
For a fixed $\x$, this sub-problem for solving $\varphi(\x)$ is an inner simple BLO task,
as it is only about one variable $\y$, with $\x$ as a parameter. 
Then, we hope to \textcolor{black}{minimize}
$\varphi(\x)$ through gradient descent.
However, as a value-function,~$\varphi(\x)$ is non-smooth, non-convex, even with jumps,
and thus ill-conditioned, 
so we use a smooth function to approximate~$\varphi(\x)$ and approach~$\frac{\partial\varphi(\x)}{\partial \x}$.
Existing methods can be classified into two categories according to divergent ways to calculate~$\frac{\partial\varphi(\x)}{\partial \x}$~\cite{liu2021investigating}, i.e., Explicit Gradient-Based Methods (EGBMs), which derives the gradient by Automatic Differentiation (AD), and Implicit Gradient-Based Methods (IGBMs), which apply implicit function theorem to deal with the optimality conditions of LL problems.

Note that both EGBMs and IGBMs require $\y \in \S(\x)$ to be unique (except BDA),
denoted as $\y^*(\x)$,
while for BDA, by integrating information from both the UL and LL sub-problem,
$\y_T(\x)$ is obtained by iterations to approach the appropriate $\y^*(\x)$.
Hence, $\varphi(\x)  = F(\x,\y^*(\x))$,
and therefore by the chain rule,
the approximated $\frac{\partial\varphi(\x)}{\partial \x}$ is split into direct and indirect gradients of $\x$,
\begin{equation}
	\label{eq:requ}
	\frac{\partial \varphi(\x)}{\partial \x} 
	= \frac{\partial F(\x,\y^*(\x))}{\partial \x} + G(\x),
\end{equation}
where $\frac{\partial F(\x,\y)}{\partial \x}$ is the direct gradient
and $G(\x)$ is the indirect gradient, $G(\x) = \left( \frac{\partial \y^*(\x)}{\partial \x} \right)^\top \frac{\partial F(\x,\y^*)}{\partial \y^*}$.
Then we need to compute $G(\x)$, in other words, the value of $\frac{\partial \y^*(\x)}{\partial \x}$.

\textbf{Explicit Gradient-Based Methods (EGBMs)}.
	Maclaurin et al. \cite{maclaurin2015gradient} and Franceschi et al. \cite{franceschi2017forward,franceschi2018bilevel} first proposed {Reverse Hyper-Gradient (RHG)} and {Forward Hyper-Gradient (FHG)} respectively, to implement a dynamic system, under the LLC assumption.
	Given an initial point $\y_0$, denote the iteration process to approach $\y^*(\x)$ as 
	$
	\y_{t+1}(\x)=\Phi_{t} (\x,\y_t(\x)),\ t=0,1,\cdots,T-1,
	$
	where 
	\textcolor{black}{$\Phi_t$ }
	is a smooth mapping performed to solve $\y_T(\x)$ 
	and $T$ is the number of iterations.
	In particular, for example, if the process is gradient descent, 
	$
	\Phi_{t} \left(\x,\y_t(\x) \right) 
	= \y_t(\x)- s_t \frac{\partial  f \left(\x,\y_{t}(\x) \right)}{\partial \y_t}, 
	$
	where $s_t>0$ is the corresponding step size. 
	Then $\varphi(\x)$ in Eq.~\eqref{eq:minvarphi} can be approximated by $\varphi(\x) 
	\approx \varphi_T(\x) =  F(\x,\y_T(\x)) $.
	As $T$ increases, $\varphi_T(\x)$ approaches $\varphi(\x)$ generally, and a sequence of unconstrained minimization problems is obtained.
	Thus, gradient-based methods can be regarded as a kind of sequential-minimization-type scheme \cite{fiacco1990nonlinear}.
	From the chain rule, we have
	$
		\frac{\partial \y_t(\x)}{\partial \x}
		=\left( \frac{\partial \Phi_{t-1} (\x,\y_{t-1})}{\partial \y_{t-1}} \right) ^\top
		\frac{\partial \y_{t-1}(\x)}{\partial \x}
		+\frac{\partial \Phi_{t-1} (\x,\y_{t-1})}{\partial \x},
	$
	and $\frac{\partial \y_T(\x)}{\partial \x}$ can be obtained from this unrolled procedure.
	However, FHG and RHG require calculating 
	the gradient of $\x$ composed of the first-order condition of LL problem by 
	AD during the entire trajectory,
	so the computational cost owing to the time and space complexity is very high.
	In dealing with this, Shaban et al. \cite{shaban2019truncated} proposed {Truncated Reverse Hyper-Gradient (TRHG)} to truncate the iteration, and thus TRHG only needs to store the last $I$ iterations, reducing the computational load.
	Nevertheless, it additionally requires $f$ to be strongly convex, 
	and the truncated path length is hard to determine.
	Another method Liu et al.~\cite{liu2018darts} tried is to use the difference of vectors to approximate the gradient,
	but the accuracy of using the difference is not promised and there is no theoretical guarantee for this method. 
	On the other hand, 
	from the viewpoint of theory, 
	for more relaxed conditions, Liu et al.~\cite{liu2020generic,liu2022general} proposed {Bi-level Descent Aggregation (BDA)} to remove the assumption that the LL solution set is a singleton, which is a simplification of real-world problems.
	Specifically, BDA uses information from both the UL and the LL problem as an aggregation during iterations. 
	However, the obstacle of LLC and computational cost still exists.

\textbf{Implicit Gradient-Based Methods (IGBMs)}.
IGBMs or implicit differentiation \cite{pedregosa2016hyperparameter,rajeswaran2019meta,lorraine2020optimizing}, 
can be applied to obtain $\frac{\partial \y^*(\x)}{\partial \x}$
under the LLC assumption.
If $ \frac{\partial^2 f(\x,\y^*(\x))}{\partial \y \partial \y} $ 
is assumed to be invertible in advance as an additional condition,
by using the implicit function theorem 
on the optimality condition $\frac{\partial f(\x,\y^*(\x))}{\partial \y} = 0$,
the LL problem is replaced with an implicit equation,
and then
$
	\frac{\partial \y^*(\x)}{\partial \x} 
	= -\left( \frac{\partial^2 f(\x,\y^*(\x))}{\partial \y \partial \y} \right) ^{-1}
	\frac{\partial^2 f(\x,\y^*(\x))}{\partial \y \partial \x }.
$
Unlike EGBMs relying on the first-order condition during the entire trajectory,
IGBMs only depends on the first-order condition once, 
which decouples the computational burden from the solution trajectory of the LL problem, 
but this leads to repeated computation of the inverse of Hessian matrix, which is still a heavy burden.
In dealing with this, to avoid direct inverse calculation,
the {Conjugate Gradient (CG) method}~\cite{pedregosa2016hyperparameter,rajeswaran2019meta}
changes it into solving a linear system,
and Neumann method~\cite{lorraine2020optimizing} uses the Neumann series to calculate the Hessian inverse.
However, after using these methods, the computational requirements are reduced but still large, because the burden of computing the inverse of matrix changes into computing Hessian-vector products. 
Additionally, 
the accuracy of solving a linear system highly depends on its condition number~\cite{grazzi2020iteration}, and the ill condition may result in numerical instabilities.
A large quadratic term is added on the LL objective to eliminate the ill-condition in~\cite{rajeswaran2019meta}, but this approach may change the solution set greatly. 

As discussed above, 
EGBMs and IGBMs need repeated calculations of recurrent gradient or Hessian inverse,
leading to 
high time and space complexity in numerical computation, 
and require 
the LLC assumption in theory.
Actually, when the dimension of $\y$ is very large, which happens in practical problems usually, the computational burden of massively computing the products of matrices and vectors might be too heavy to carry.
In addition, the LLC assumption is also not suitable for most complex real-world tasks.

\section{The Proposed Algorithm}
In this section, we illustrate our algorithmic framework, named 
Bi-level Value-Function-based Sequential Minimization (BVFSM).
Our method also follows the idea of constructing a sequence of unconstrained minimization problems to approximate the original bi-level problem,
but different from existing methods, 
BVFSM uses the 
re-characterization via the value-function of the LL problem.
Thanks to this strategy, our algorithm is able to handle problems with complicated non-convex high-dimensional LL, which existing methods are not able to deal with.



\subsection{Value-Function-based Single-level Reformulation}

BVFSM designs a sequence of single-level unconstrained minimization problems to approximate the original problem through a value-function-based reformulation.
We first present this procedure under the optimistic BLO case.

Recall the original optimistic BLO in Eq.~\eqref{eq:OBO} has been transformed into Eq.~\eqref{eq:minvarphi}, 
and we hope to compute $G(\x)$ in Eq.~\eqref{eq:requ}.
Note that the difficulty of computing $\frac{\partial\varphi(\x)}{\partial \x}$ comes from the ill-condition of $\varphi(\x)$,
owing to the nested structure of the bi-level sub-problem for solving $\varphi(\x)$. 
Hence, we introduce the value-function of the LL problem
$f^*(\x) := \min_{\y}   f(\x,\y) $ to transform it into a single-level problem.
Then the problem can be reformulated as
\begin{equation}
		\varphi(\x)= \mathop{\min}_{\y} \Big\{  F(\x,\y) 
		: f(\x,\y) \leq f^*(\x) \Big\} .
	\label{eq:f*}
\end{equation}
However, the inequality constraint $f(\x,\y) \leq f^*(\x)$ is still ill-posed,
because it does not satisfy any standard regularity condition
and $f^*(\x)$ is non-smooth.
In dealing with such difficulty, 
we approximate $f^*(\x)$ with regularization:
\begin{equation}
	f_{\mu}^*(\x)=\min\limits_{\y} 
	\left\{ 
	f(\x,\y)  + \frac{\mu}{2} \| \y \|^2
	\right\},
	\label{eq:fmu}
\end{equation}
where $\frac{\mu}{2} \| \y \|^2$ ($\mu > 0$) is the regularization term.

We further add an auxiliary function of the inequality constraints to the objective,
and obtain
\begin{equation} 
	\label{eq:phimu}
	\begin{aligned}	
		\varphi_{\mu,\theta,\sigma}(\x) \! = \! \min\limits_{\y} &
		\left\{ 
		F(\x,\y)  
		\! + \! \PP_{\sigma} \! \left( f(\x,\y) - f_{\mu}^*(\x) \right)
		\! + \! \frac{\theta}{2} \| \y \|^2
		\right\},
	\end{aligned}
\end{equation}
where $(\mu,\theta,\sigma) > 0$,
$\frac{\theta}{2} \| \y \|^2$ is the regularization term,
and $ \PP_{\sigma} : \mathbb{R}\rightarrow\mathbb{\overline{R}}$ 
(where $\mathbb{\overline{R}} = \mathbb{R} \cup \{\infty\}$) 
is the selected auxiliary function for the sequential unconstrained minimization method with parameter~$\sigma$, which will be 
\textcolor{black}{defined in Eq.~\eqref{eq:\PP} and}
 discussed in detail next. 
This reformulation changes the constrained problem Eq.~\eqref{eq:f*} into a sequence of unconstrained problems Eq.~\eqref{eq:phimu} under different~parameters.
\textcolor{black}{
The regularization terms in Eq.~\ref{eq:fmu} and Eq.~\ref{eq:phimu} are to guarantee the uniqueness of solution to these two problems, which is essential for the differentiability of~$\varphi_{\mu,\theta,\sigma}(\x)$, and will be discussed in Remark~\ref{remark 1} of Section~\ref{sec:algorithmic framework}.
Experiments in Section~\ref{sec:toy optimistic} 
also demonstrate that introducing the regularization terms for differentiability to avoid possible jumps matters to improve the computational~stability.
%
}


The sequential unconstrained minimization  method is mainly used for solving constrained nonlinear programming 
by changing the problem into a sequence of unconstrained minimization problems \cite{fiacco1990nonlinear,lasdon1972efficient,byrne2013alternating}. 
To be specific, we add to the objective a selected auxiliary function of the constraints with a sequence of parameters,
and obtain a series of unconstrained problems. 
The convergence of parameters makes the sequential unconstrained problems converge to the original constrained problem,
leading to the convergence of the solution. 
Based on the property of auxiliary functions, they are divided mainly into two types,
barrier functions and penalty functions~\cite{freund2004penalty,luenberger1984linear},
whose definitions are provided~here. 

\begin{defi}
	\label{defi:barrier}
	A continuous, differentiable, and non-decreasing function $ \rho : \mathbb{R}\rightarrow\mathbb{\overline{R}}$ is called a {standard barrier function} if $\rho(\omega ; \sigma)$ satisfies 
	$\rho(\omega ; \sigma) \geq 0$ and $\lim_{\sigma \rightarrow 0}\rho(\omega ; \sigma)=0 $, when $\omega < 0$; 
	and $\rho(\omega ; \sigma) \rightarrow \infty$ when $\omega \rightarrow 0$.
	It is called a {standard penalty function} if it satisfies 
	$\rho(\omega ; \sigma) = 0$ when $\omega \leq 0$; 
	and $\rho(\omega ; \sigma) > 0$ and $\lim_{\sigma \rightarrow 0}\rho(\omega ; \sigma)=\infty $ when $\omega > 0$.
	Here $\sigma > 0$ is the barrier or penalty parameter.
	In addition, if $\rho(\omega ; \sigma^{(1)})$ is a standard barrier function,
	then $\rho(\omega - \eta \; ; \sigma^{(1)}) \ $ 
	is called a {modified barrier function} $(\sigma^{(1)}, \eta > 0)$.
\end{defi}


For the simplicity of expression later, we denote the function $\PP_{\sigma}$ in Eq.~\eqref{eq:phimu} to be 
\begin{equation} \label{eq:\PP}
	\PP_{\sigma}(\omega) := 
	\left\{
	\begin{aligned}
		& \rho(\omega; \sigma), \text{if} \ \rho \ \text{ is a penalty} \\
		& \qquad\qquad\qquad \text{ or standard barrier function}, \\
		& \rho(\omega- \eta; \sigma^{(1)}), \text{if} \ \rho \ \text{ is a modified barrier function}.
	\end{aligned}
	\right.    
\end{equation}
Here for a modified barrier function, $\eta > 0$ is to guarantee that in Eq.~\eqref{eq:phimu},
$f(\x,\y) - f_{\mu}^*(\x) - \eta < 0$, and the barrier function is well-defined.


Classical examples of auxiliary functions  are the quadratic penalty function,
inverse barrier function
and log barrier function \cite{boukari1995survey,luenberger1984linear}. 
There are also some other popular examples, such as 
the polynomial penalty function \cite{freund2004penalty}
and truncated log barrier function \cite{auslender1999penalty}.
These examples of standard penalty and barrier functions are listed in Table~\ref{table penalty barrier}.
Note that we need the smoothness of $\varphi_{\mu,\theta,\sigma}(\x)$,
and will calculate the gradient of $\PP_{\sigma}$ afterwards, 
so we choose smooth auxiliary functions rather than non-smooth exact penalty functions.
\textcolor{black}{ 
	Note that when $\rho$ is a modified barrier function, $\sigma$ of $\PP_{\sigma}(\omega)$ in Eq.~\eqref{eq:\PP} has two components~$\sigma^{(1)}$ and~$\eta$, and the specific form of $\PP_{\sigma}(\omega)$ comes from substituting	
	$\omega - \eta$ and $\sigma^{(1)}$ for~$\omega$ and $\sigma$ in Table~\ref{table penalty barrier}.
}

\begin{table}[hbtp]
	\centering
	\caption{Some available standard penalty and barrier functions.} 
	\label{table penalty barrier}
	\vspace {-0.3cm}
	\begin{footnotesize}
		\begin{threeparttable}
\resizebox{\linewidth}{!}{
			\begin{tabular}{ |c|l| }
				\hline
				\multicolumn{2}{|c|}{}\\[-6pt]
				\multicolumn{2}{|c|}{Penalty functions} \\
				\multicolumn{2}{|c|}{}\\[-6pt]
				\hline
				&\\[-6pt]
				Quadratic 
				& $	\rho \! \left( \omega ; \sigma \right) = \frac{1}{2\sigma} \left( \omega^+ \right)^2$,
				where 
				$\omega^+ = \max \left\{ \omega,0 \right\}$ \\
				&\\[-6pt]
				\hline
				&\\[-6pt]
				\multirow{2}{*}{Polynomial}
				& $\rho \! \left( \omega ; \sigma \right) = \frac{1}{q \sigma} \left( \omega^+ \right)^q$,
				where $q$ is a positive integer \\ 
				& chosen such that $\rho \! \left( \omega ; \sigma \right)$ is differentiable\\

				\hline
				\multicolumn{2}{|c|}{}\\[-6pt]
				\multicolumn{2}{|c|}{Barrier functions} \\
				\multicolumn{2}{|c|}{}\\[-6pt]
				\hline
				&\\[-6pt]
				Inverse
				& $ 
				\rho \! \left( \omega ; \sigma \right) = 
				\left\{ 
				\begin{aligned}
					& - \frac{\sigma}{\omega}, & \omega<0 \\
					\\[-0.6cm]
					& \infty , & \omega \geq 0
				\end{aligned}
				\right.$ \\
				&\\[-6pt]
				
				
				\hline
				&\\[-6pt]
				\multirow{3}{*}{Truncated Log}
				& $  
				\rho \! \left( \omega ; \sigma \right) = 
				\left\{ 
				\begin{aligned}
					& - \sigma \left( \log \left( -\omega \right) + \beta_1 \right), & -\kappa \leq \omega < 0 \\
					\\[-0.5cm]
					& - \sigma \left( \beta_2 + \frac{\beta_3}{\omega^2} + \frac{\beta_4}{\omega} \right), & \omega < -\kappa \\ 
					\\[-0.6cm]
					& \infty , & \omega \geq 0
				\end{aligned}
				\right.
				$ \\
				& where $0 < \kappa \leq 1$, $\beta_1,\beta_2,\beta_3,\beta_4$ are chosen \\
				& such that $\rho \! \left( \omega ; \sigma \right) \geq 0$ and is twice differentiable\\
				\hline
				
			\end{tabular}
}
		\end{threeparttable}
	\end{footnotesize}
\end{table}

\subsection{Sequential Minimization Strategy}
\label{sec:algorithmic framework}

From the discussion above, we then hope to solve
\begin{equation}
	\label{eq:varphi approximation}
	\min\limits_{\x \in \X} \  \varphi_{\mu,\theta,\sigma}(\x),
\end{equation}
with $\varphi_{\mu,\theta,\sigma}(\x)$ in Eq.~\eqref{eq:phimu}.
First denote 
\begin{equation}
	\label{eq:z*}
	\z^*_{\mu}(\x) = \underset{\y}{\mathrm{argmin}}
	\left\{ 
	f(\x,\y) 
	+ \frac{\mu}{2} \| \y \|^2
	\right\},
\end{equation}
\begin{equation}\small
	\label{eq:y*}
	\begin{aligned}
		\y_{\mu,\theta,\sigma}^*(\x) 
		= \underset{\y}{\mathrm{argmin}}
		& \left\{ 
		F(\x,\y) 
		+  \PP_{\sigma} \! \left( f(\x,\y) - f_{\mu}^*(\x) \right)
		+ \frac{\theta}{2} \| \y \|^2
		\right\}.
	\end{aligned}
\end{equation}
The following proposition gives the smoothness of $\varphi_{\mu,\theta,\sigma} \left( \x \right)$ and the formula for computing $\frac{ \partial \varphi_{\mu,\theta,\sigma} \left( \x \right) }{\partial \x}$ or $G(\x)$, which serves as the ground for our algorithm.

\begin{pro}[Calculation of $G(\x)$]\label{prop1}
	Suppose $ F(\x, \y) $ and $f(\x, \y)$ are bounded below 
	and continuously differentiable. 
	Given~$\x \in \X$ and $\mu, \theta, \sigma > 0$, when
	$\z^*_{\mu}(\x)$ and $\y_{\mu,\theta,\sigma}^*(\x)$
	are unique, 
	then $\varphi_{\mu,\theta,\sigma} (\x)$ is differentiable and
	\begin{equation}
		\label{eq:def_G}
		\begin{aligned}
			G(\x) = \frac{ \partial \PP_{\sigma} \! \left( f \left(\x,\y \right) - f_{\mu}^*(\x) \right) }{ \partial \x } \big|_{ \y = \y_{\mu,\theta,\sigma}^*(\x)},		
		\end{aligned}
	\end{equation}
	where
	$
	\begin{aligned}
		f_{\mu}^*(\x) = 
		f \! \left(\x,\z^*_{\mu}(\x) \right) 
		+ \frac{\mu}{2} \| \z^*_{\mu}(\x) \|^2,
	\end{aligned}
	$
	and
	$
	\frac{ \partial f_{\mu}^*(\x) }{\partial \x} = 
	\frac{\partial f \! \left(\x,\y \right) }{\partial \x} 
	\big|_{\y=\z^*_{\mu}(\x)}.
	$
\end{pro}

\begin{proof}
	We first prove that for any $\bar{\x} \in \X$, 
	$f(\x,\y)  + \frac{\mu}{2}\|\y\|^2$ is level-bounded in $\y$ locally uniformly in $\bar{\x} \in \X$ (see~\cite[Definition 3]{liu2020generic}).
	That means for any $c \in \mathbb{R}$, there exist $\delta > 0$ and a bounded set \textcolor{black}{$\BB \subset \mathbb{R}^n$},
	such that
	$
		\left\{\y \in \mathbb{R}^n : f(\x,\y) + \frac{\mu}{2}\|\y\|^2 \leq c \right\} \subset \BB,  
	$
	for all $\x \in \BB_{\delta}(\bar{\x})\cap \X$,
	\textcolor{black}{where $\BB_{\delta}(\bar{\x})$ denotes the open ball with center at $\bar{\x}$ and radius $\delta$, i.e., $\BB_{\delta}(\bar{\x}) = \{ \hat{\x}\in\X : \|\bar{\x} - \hat{\x} \| < \delta \}$}. 
	Assume by contradiction that the above does not hold.
	Then there exist sequences $\{\x_k\}$ and $\{ \y_k \}$ satisfying $\x_k \rightarrow \bar{\x}$ and $\| \y_k \| \rightarrow + \infty$, such that
	$
		f(\x_k,\y_k) + \frac{\mu}{2} \| \y_k \|^2 \leq c.
	$
	As $f(\x,\y)$ is bounded below, 
	then $\|\y_k\| \rightarrow \infty$ implies $f(\x_k,\y_k) + \frac{\mu}{2} \| \y_k \|^2 \rightarrow \infty$, which contradicts with $	f(\x_k,\y_k) + \frac{\mu}{2} \| \y_k \|^2 \leq c$ and $c \in \mathbb{R}$.

	Hence, from the arbitrariness of $\bar{\x}$, we have $f(\x,\y) \! + \! \frac{\mu}{2}\|\y\|^2$ is level-bounded in $\y$ locally uniformly in $\x \!\!\! \in \!\!\! \X$,
	and then the inf-compactness condition in~\cite[Theorem 4.13]{bonnans2013perturbation} holds for $f(\x,\y)  + \frac{\mu}{2}\|\y\|^2$. 
	Since	
	$ {\mathrm{argmin}}_{\y \in \mathbb{R}^n} \left\{ f(\x,\y)  + \frac{\mu}{2}\|\y\|^2 \right\}$ is a singleton, it follows from~\cite[Theorem 4.13, Remark 4.14]{bonnans2013perturbation} that
	\[
	\begin{aligned}
		\textcolor{black}{\frac{ \partial f_{\mu}^*(\x) }{\partial \x}  } 
		& = \frac{\partial \left( f(\x,\y) + \frac{\mu}{2}\|\y\|^2    \right)}{\partial \x}  \big|_{\y = \z_{\mu}^*(\x)} 
		= \frac{\partial  f(\x,\y)  }{\partial \x}  \big|_{\y = \z_{\mu}^*(\x)}.
	\end{aligned}
	\]

	Next, 
	from definitions of penalty and barrier functions (Definition~\ref{defi:barrier}),
	we have $\rho(\omega ; \sigma) \geq 0$ for any $\omega$, so $\PP_{\sigma}(\omega) \geq 0$ holds.
	Then, 
	since $F(\x, \y)$ is assumed to be 
	bounded below, 
	similar to $f(\x,\y)  + \frac{\mu}{2}\|\y\|^2$,
	the inf-compactness condition in~\cite[Theorem 4.13]{bonnans2013perturbation} also holds for 
	$
	F(\x,\y) + \PP_{\sigma} \! \left( f(\x,\y) - f_{\mu}^*(\x) \right)	+ \frac{\theta}{2} \| \y \|^2
	$. 
	Combining with the fact that 
	\[
	{\mathrm{argmin}}_{\y \in \mathbb{R}^n} \left\{
	F(\x,\y) 
	+  \PP_{\sigma} \! \left( f(\x,\y) - f_{\mu}^*(\x) \right)
	+ \frac{\theta}{2} \| \y \|^2
	\right\}
	\]
	is a singleton,~\cite[Theorem 4.13, Remark 4.14]{bonnans2013perturbation} shows that
	\[
	\begin{aligned}
		& \frac{ \partial \varphi_{\mu,\theta,\sigma} \left( \x \right) }{\partial \x} \\
		& = \frac{ \partial \left( F(\x,\y)  +  \PP_{\sigma} \! \left( f(\x,\y) - f_{\mu}^*(\x)  \right) + \frac{\theta}{2} \| \y \|^2  \right) }{\partial \x}
		\big|_{\y = \y_{\mu,\theta,\sigma}^*(\x)} \\
		& = \left( \frac{ \partial F(\x,\y)  }{\partial \x}
		+ \frac{ \partial  \PP_{\sigma} \! \left( f(\x,\y) - f_{\mu}^*(\x) \right)   }{\partial \x} \right)
		\big|_{\y = \y_{\mu,\theta,\sigma}^*(\x)}. \\
	\end{aligned}
	\]
	Therefore, the conclusion in Eq.~\eqref{eq:def_G} follows immediately.
\end{proof}

\begin{remark}\label{remark 1}
	\color{black}{
	In Proposition~\ref{prop1} we require the uniqueness of $\z^*_{\mu}(\x)$ and $\y_{\mu,\theta,\sigma}^*(\x)$ to guarantee the differentiability of~$\varphi_{\mu,\theta,\sigma}(\x)$.
	This can be achieved by conditions much weaker than convexity, such as the convexity only on a level set.
	We start with the uniqueness of $\z^*_{\mu}(\x)$.
	For any given $\x \in \X$, consider a function $f(\x,\y)$ satisfying that there exists a constant $c > \min_{\y} f(\x,\y)$ such that $f(\x,\y)$ is convex in $\y$ on the level set $\{\y: f(\x,\y) \le c\}$.
	Suppose $\hat{\y}$ is a minimum of $f(\x,\y)$.
	Then $\inf_{\y} \{ f(\x,\y) + \mu/2 \|\y\|^2 \} \le f(\x,\hat{\y}) + \mu/2 \| \hat{\y} \|^2 = \min_{\y} f(\x,\y) + \mu/2 \| \hat{\y} \|^2 <c, $ for a sufficiently small $\mu > 0$.
	Thus, $\z^*_{\mu}(\x)$ locates inside the level set $\{\y: f(\x,\y) \le c\}$ on which $f(\x,\y)$ is convex.
	Hence, $f(\x,\y) + \mu/2 \|\y\|^2$ is strictly convex on $\{\y: f(\x,\y) \le c\}$, and the uniqueness of $\z^*_{\mu}(\x)$ follows.

	As for the uniqueness of $\y_{\mu,\theta,\sigma}^*(\x)$,
	suppose given $\x$, there exist constants $c_1 > \min_{\y \in \S(\x)} F(\x,\y)$ and $c_2 > \min_{\y} f(\x,\y)$ such that $F$ and $f$ are convex in $\y$ on the set $\{\y: F(\x,\y) \le c_1 \text{ and } f(\x,\y) \le c_2\}$.
	If we select a non-decreasing and convex auxiliary function $\PP_{\sigma}(\cdot)$ (such as those in Table~\ref{table penalty barrier}),  then $\PP_{\sigma} \! \left( f(\x,\y) - f_{\mu}^*(\x) \right)$ is convex in~$\y$ on the set (see~\cite{mordukhovich2013easy} Proposition 1.54).
	Or simply if there exists $c > \min_{\y} F(\x,\y) +  \PP_{\sigma} \! \left( f(\x,\y) - f_{\mu}^*(\x) \right)$ such that	
	$F(\x,\y) +  \PP_{\sigma} \! \left( f(\x,\y) - f_{\mu}^*(\x) \right)$ is convex on the set $\{\y: F(\x,\y) +  \PP_{\sigma} \! \left( f(\x,\y) - f_{\mu}^*(\x) \right) \le c\}$,
	then it derives the strict convexity of $ 
	F(\x,\y) 
	+  \PP_{\sigma} \! \left( f(\x,\y) - f_{\mu}^*(\x) \right)
	+ \frac{\theta}{2} \| \y \|^2
	$ in $\y$ on the set similarly,
	and the uniqueness of $\y_{\mu,\theta,\sigma}^*(\x)$ follows.
		
		
%
%
	}	
\end{remark}

Proposition~\ref{prop1} serves as the foundation for our algorithmic framework.
Denote $\varphi_{k}(\x) := \varphi_{\mu_k,\theta_k,\sigma_k}(\x)$.
Next, we will illustrate the implementation at the $k$-{th} step of the outer loop 
and the $l$-{th} step of the inner loop, that is, to calculate $\frac{\partial  \varphi_{k}(\x_l) }{\partial \x}$, as a guide.

We first calculate $\z^*_{\mu_k}(\x_l)$ in Eq.~\eqref{eq:z*} through $T_{\z}$ steps of gradient descent, and denote the output as $\z_{k,l}^{{T_{\z}}}$,
regarded as an approximation of $\z^*_{\mu_k}(\x_l)$.
After that, we calculate $\y_{\mu_k,\theta_k,\sigma_k}^*(\x_l)$ in Eq.~\eqref{eq:y*} through $T_{\y}$ steps of gradient descent, and denote the output as $\y_{k,l}^{T_{\y}}$.
\textcolor{black}{
Note that if the objective function is convex, running some number of steps of the method would lead to an approximation of the minimizer.
Meanwhile, the convexity of objective functions for approximating $\z^*_{\mu_k}(\x_l)$ and $\y_{\mu_k,\theta_k,\sigma_k}^*(\x_l)$ can be guaranteed if $f$ and $F$ are convex in $\y$ as discussed in Remark~\ref{remark 1}.
Also, if the objective function is not convex but all of its stationary points are minimizers, which is a weaker condition than convexity, gradient descent would still help to converge to minimizers.
}

Then, according to Proposition~\ref{prop1}, we can obtain 
\begin{equation}
	\label{eq:a_g_vphi}
	\frac{\partial  \varphi_{k}(\x_l) }{\partial \x}\approx \frac{\partial F(\x_l,\y_{k,l}^{T_{\y}})}{\partial \x}+G_{k,l} \ , 
\end{equation}
with
$
	G_{k,l} = \frac{ \partial \PP_{\sigma_k} \! \left( f \left( \x_l,\y_{k,l}^{T_{\y}} \right) - f_{k,l}^{T_{\z}}   \right) }{ \partial \x },		
$
where 
$
	f_{k,l}^{T_{\z}} = 
	f(\x_l, \z_{k,l}^{T_{\z}}) 
	+ \frac{\mu_{k}}{2} \| \z_{k,l}^{{T_{\z}}} \| ^2.
$
As a summary, the algorithm for solving Eq.~\eqref{eq:varphi approximation} 
is shown in Algorithms~\ref{alg:outer} and~\ref{alg:ours}.

\begin{algorithm}[!htpb]
	\caption{Our Solution Strategy for Eq.~\eqref{eq:varphi approximation}.}
	\label{alg:outer}
	\begin{algorithmic}[1]
		\REQUIRE $\x_0$, $(\mu_0,\theta_0,\sigma_0)$, step size $\alpha>0$.
		\ENSURE The optimal UL solution $\x^*$.
		
		\FOR {$k = 0 \rightarrow K-1$}
		\STATE Calculate $(\mu_k,\theta_k,\sigma_k)$.
		\FOR {$l=0\rightarrow L-1$}
		\STATE Calculate $\frac{\partial  \varphi_{k}(\x_l)}{\partial \x}$ by Algorithm~\ref{alg:ours}.
		\STATE $\x_{l+1}= \mathtt{Proj}_{\X} \left(
		\x_{l}-\alpha \frac{\partial  \varphi_{k}(\x_l)}{\partial \x} \right)$.
		\ENDFOR
		\STATE $\x_0 = \x_L$.
		\ENDFOR
		
		\RETURN $\x^*$.
		
	\end{algorithmic}  
\end{algorithm}

\begin{algorithm}[!htpb]
	\caption{Calculation of $\frac{\partial  \varphi_{k}(\x_l)}{\partial \x}$.}
	\label{alg:ours}
	\begin{algorithmic}[1]
		\REQUIRE $\x_l, (\mu_k,\theta_k,\sigma_k)$.
		\ENSURE $\frac{\partial  \varphi_{k}(\x_l)}{\partial \x}$.
		
		\STATE Calculate $\z_{k,l}^{{T_{\z}}}$ as an approximation of $\z^{*}_{\mu_k} \! (\x_l)$ by performing $T_\z$-step gradient descent on Eq.~\eqref{eq:z*}.
		\STATE Calculate $\y_{k,l}^{T_{\y}}$ as an approximation of $\y_{\mu_k,\theta_k,\sigma_k}^* \! (\x_l)$ by performing $T_\y$-step gradient descent on Eq.~\eqref{eq:y*}.
		\STATE Calculate an approximation of $\frac{\partial  \varphi_{k}(\x_l)}{\partial \x}$ by Eq.~\eqref{eq:a_g_vphi}.
		
	\end{algorithmic}  
\end{algorithm}

\subsection{Extension for BLO with Functional Constraints}\label{sec:constraint}

We consider the BLO with functional constraints on UL and LL problems in this subsection, which is a more general setting, and the above discussion without constraints can be extended to the case with inequality constraints.


The optimistic BLO problems in Eq.~\eqref{eq:OBO} with functional constraints are then
\begin{equation*} \small
	\begin{aligned}
		& \min\limits_{\x \in \X, \y \in \Y}  
		 \ F(\x,\y)    \\
		& \mathrm{ \ s.t.\  } 
		\left\{
		\begin{aligned}
			& H_{j} (\x,\y) \leq 0  , \text{for any } j  \ \\ 
			& \y \in {\rm \S}(\x) 
			:=\mathop{\arg\min}_{\y} 
			\Big\{  f(\x,\y): h_{j'}(\x,\y) \leq 0  ,	 
			\text{for any } j' 
			\Big\} 
		\end{aligned} 
		\right.	,		
	\end{aligned}
\end{equation*}
where the UL constraints $H_{j}(\x,\y) :\mathbb{R}^m\times\mathbb{R}^n\rightarrow\mathbb{R}$ $\left( \text{for any } {j} \in \left\{1,2,\cdots,J \right\} \right)$ 
and the LL constraints $h_{j'}(\x,\y) :\mathbb{R}^m\times\mathbb{R}^n\rightarrow\mathbb{R}$ $\left( \text{for any } {j'} \in \left\{1,2,\cdots,J' \right\} \right)$ 
are continuously differentiable functions.
This is equivalent to
$
	\min_{\x \in \X} \varphi(\x) ,
$
where $\varphi(\x)$, the value-function of the sub-problem in Eq.~\eqref{eq:minvarphi}, is adapted correspondingly to be
\[\varphi(\x):=  \mathop{\min}_{\y}  \Big\{  F(\x,\y) : 
H_{j} (\x,\y) \leq 0, \ \forall j, \text{ and } \y \in {\rm \S}(\x) \Big\}.
\]

Also, the counterpart for value-function of the LL problem is
$f^*(\x) := \min_{\y} \left\{  f(\x,\y): h_{j'}(\x,\y) \leq 0  , \forall \ {j'} \right\}$,
and following the technique within Eq.~\eqref{eq:f*} to transform the LL problem into an inequality constraint,
the problem can be reformulated as
\begin{equation}
	\label{eq:constrain varphi}
	\begin{aligned}
		\varphi(\x)= \mathop{\min}_{\y}  \Big\{  & F(\x,\y)  :
			H_{j} (\x,\y) \leq 0  , \forall \ j, \\
			& f(\x,\y) \leq f^*(\x),
			\text{ and } h_{j'}(\x,\y) \leq 0  , \ \forall \  {j'}
		\Big\}.
	\end{aligned} 
\end{equation}
After that, using the same idea of sequential minimization method, 
inspired by the regularized smoothing method in~\cite{borges2020a},
the value-function $f^*(\x)$ can be approximated with a barrier function (different from Eq.~\eqref{eq:fmu} due to the LL constraints) and a regularization term:
\begin{equation*}
	f_{\mu, \sigma_B}^*(\x) = \min\limits_{\y} 
	\Bigg\{ 
	f(\x,\y) + \sum_{{j'}=1}^{J'} \PP_{B,\sigma_B} \! \left( h_{j'}(\x,\y) \right) + \frac{\mu}{2} \| \y \|^2
	\Bigg\},
\end{equation*}
where $\PP_{B,\sigma_B} (\omega): \mathbb{R}\rightarrow\mathbb{\overline{R}}$ is the selected standard barrier function for the LL constraint $h_{j'}$ as defined in Eq.~\eqref{eq:\PP}, with $\sigma_B$ as the barrier parameter. 
Note that here we define $\PP_{B,\sigma_B}$ as a standard barrier function but not a modified barrier function. 
As for the approximation of $\varphi(\x)$, Eq.~\eqref{eq:phimu} is transferred into 
$\small
	\label{eq:constrain phimu}
	\begin{aligned}	
		& \varphi_{\mu,\theta,\sigma}(\x) =  \min\limits_{\y} 
		\Bigg\{ 
		F(\x,\y) + \sum_{j=1}^{J} \PP_{H,\sigma_H} \! \left( H_{j}(\x,\y) \right)\\
		& + \sum_{{j'}=1}^{J'} \PP_{h,\sigma_h} \! \left( h_{j'}(\x,\y) \right) 
		+  \ \PP_{f,\sigma_f} \! \left( f(\x,\y) - f_{\mu, \sigma_B}^*(\x) \right)
		+ \frac{\theta}{2} \| \y \|^2
		\Bigg\},
	\end{aligned}
$
where $\PP_{H,\sigma_H}, \PP_{h,\sigma_h}, \PP_{f,\sigma_f} : \mathbb{R}\rightarrow\mathbb{\overline{R}}$ are the selected auxiliary functions of penalty or modified barrier with parameters $\sigma_H$, $\sigma_h$ and $\sigma_f$, and $(\mu,\theta, \sigma) = (\mu,\theta, \sigma_B, \sigma_H, \sigma_h, \sigma_f) > 0$.

Then corresponding to Eq.~\eqref{eq:z*} and Eq.~\eqref{eq:y*}, by denoting
\begin{equation*}\small
	\z^*_{\mu,\sigma_B}(\x) \! = \! \underset{\y}{\mathrm{argmin}}
	\Bigg\{ 
	f(\x,\y) + \sum_{{j'}=1}^{J'} \PP_{B,\sigma_B} \! \left( h_{j'}(\x,\y) \right) + \frac{\mu}{2} \| \y \|^2
	\Bigg\},
\end{equation*}
$\small
	\begin{aligned}
		& \y_{\mu,\theta,\sigma}^*(\x) 
		= \underset{\y}{\mathrm{argmin}}
		\Bigg\{ 
		F(\x,\y) + \sum_{j=1}^{J} \PP_{H,\sigma_H} \! \left( H_{j}(\x,\y)  \right)\\
		& + \sum_{{j'}=1}^{J'} \PP_{h,\sigma_h} \! \left( h_{j'}(\x,\y) \right) 
		+  \ \PP_{f,\sigma_f} \! \left( f(\x,\y) - f_{\mu, \sigma_B}^*(\x) \right)
		+ \frac{\theta}{2} \| \y \|^2 
		\Bigg\},
	\end{aligned}
$
we have the next proposition, which follows the same idea from Proposition~\ref{prop1}.

\begin{pro}
	\label{prop constrain}
	Suppose $ F(\x, \y) $ and $f(\x, \y)$ are bounded below  and continuously differentiable. 
	Given $\x \in \X$ and $\mu, \theta, \sigma > 0$, when
	$\z^*_{\mu,\sigma_B}(\x)$ and $\y_{\mu,\theta,\sigma}^*(\x)$
	are unique, 
	then $\varphi_{\mu,\theta,\sigma} (\x)$ is differentiable and
	$
		\begin{aligned}
			G(\x) = & \Bigg[ \sum_{j=1}^{J} \frac{\partial \PP_{H,\sigma_H} \! \left( H_{j}(\x,\y)  \right)}{\partial \x} + 
			\sum_{{j'}=1}^{J'} \frac{ \partial \PP_{h,\sigma_h} \! \left( h_{j'}(\x,\y) \right)   }{ \partial \x } \\
			& + \frac{ \partial \PP_{f,\sigma_f}  \! \left( f \left(\x,\y \right) - f_{\mu, \sigma_B}^*(\x) \right) }{ \partial \x } \Bigg] \big|_{ \y = \y_{\mu,\theta,\sigma}^*(\x)},		
		\end{aligned}
	$
	where
	\begin{equation*} 
		\begin{aligned}
			f_{\mu, \sigma_B}^*(\x) = 
			f \! \left(\x,\z^*_{\mu, \sigma_B}(\x) \right) 
			+  \sum_{{j'}=1}^{J'} \partial \PP_{B,\sigma_B} \! \left( h_{j'}(\x,\z^*_{\mu, \sigma_B}(\x)) \right)  \\
			+ \frac{\mu}{2} \| \z^*_{\mu, \sigma_B}(\x) \|^2,
		\end{aligned}
	\end{equation*}
	\begin{equation*}\small
		\frac{ \partial f_{\mu, \sigma_B}^*(\x) }{\partial \x} \! \! = \! \! 
		\Bigg[
		\frac{\partial f \! \left(\x,\y \right) }{\partial \x} 
		+ \sum_{{j'}=1}^{J'} \frac{ \partial \PP_{B,\sigma_B} \! \left( h_{j'} \! \left(\x,\y \right)  \right)  }{\partial \x} \Bigg]  \big|_{\y=\z^*_{\mu, \sigma_B}(\x)}.
	\end{equation*}
\end{pro}

The proof is similar to Proposition~\ref{prop1}, obtained by applying~\cite[Theorem 4.13, Remark 4.14]{bonnans2013perturbation}.
The algorithm is then based on Proposition~\ref{prop constrain} and similar to Algorithm~\ref{alg:outer} and \ref{alg:ours}.
Note that in Section~\ref{sec:theoretical investigations},
our convergence analysis is carried out under this constrained setting,
because problems without constraints can be regarded as \textcolor{black}{its special case}.

\begin{remark}
	\color{black}{
		In terms of the uniqueness of $\z^*_{\mu,\sigma_B}(\x)$ and $\y_{\mu,\theta,\sigma}^*(\x)$, 
		if we select convex auxiliary functions $\PP_{B,\sigma_B}$, $\PP_{H,\sigma_H}$, $\PP_{h,\sigma_h}$,  $\PP_{f,\sigma_f}$,
		and suppose $h_{j'}(\x,\y), \forall \ j',$ and $H_{j}(\x,\y), \forall \ j,$ in the constraints are convex in the level set,
		then the uniqueness follows similarly as in Remark~\ref{remark 1}.		
%
%
%
	}	
\end{remark}


\subsection{Extension for Pessimistic BLO}\label{sec:pessimistic}

In this part, we consider the pessimistic BLO, 
which has been rarely discussed for handling learning tasks to our best knowledge.
For brevity, we focus on problems without constraints on UL and LL, and this can be extended to the case with constraints easily.
As what we have discussed about pessimistic BLO at the very beginning,
its form is
\begin{equation}
	\begin{aligned}
		\min\limits_{\x \in \X} & \max\limits_{\y \in \Y} \ F(\x,\y),  \
		\mathrm{ \ s.t.\  } \y \in {\rm \S}(\x) = \mathop{\arg\min}_{\y} f(\x,\y).
		\label{eq:PBO}
	\end{aligned}
\end{equation}

Similar to the optimistic case, this problem can be transformed into 
$
	\min_{\x} \ \varphi^p(\x) ,
$	
where the value-function $\varphi(\x)$ in Eq.~\eqref{eq:minvarphi} is redefined as
$
	\begin{aligned}
		\varphi^p(\x):= &\mathop{\max}_{\y} \Big\{  F(\x,\y) 
		: \y \in {\rm \S}(\x)
		\Big\}.
	\end{aligned} 
$
Considering the value-function $f^*(\x)$, we have the same regularized~$f_{\mu}^*(\x)$ to the optimistic case in Eq.~\eqref{eq:fmu}.
As for the approximation of $\varphi^p(\x)$, 
thanks to the value-function-based sequential minimization, different from Eq.~\eqref{eq:phimu}, we have 
\begin{equation*}\small
	\begin{aligned}	
		\varphi^p_{\mu,\theta,\sigma}(\x) =  \max\limits_{\y} &
		\Big\{ 
		F(\x,\y) - \PP_{\sigma}\! \left( f(\x,\y) - f_{\mu}^*(\x)  \right)
		- \frac{\theta}{2} \| \y \|^2
		\Big\},
	\end{aligned}
\end{equation*}
\textcolor{black}{
where $\PP_{\sigma}$ is defined in Eq.~\eqref{eq:\PP}}.
Same as before, our goal is to solve
$
	\min_{\x} \varphi^p_{\mu,\theta,\sigma}(\x).
$

Denote $\z^*_{\mu}(\x)$ to be the same as in Eq.~\eqref{eq:z*},
and Eq.~\eqref{eq:y*} is changed into
\begin{equation*}\small
	\begin{aligned}
		\y_{\mu,\theta,\sigma}^*(\x) 
		= \underset{\y}{\mathrm{argmax}}
		& \Big\{ 
		F(\x,\y) 
		-  \PP_{\sigma} \! \left( f(\x,\y) - f_{\mu}^*(\x)  \right)
		- \frac{\theta}{2} \| \y \|^2
		\Big\}.
	\end{aligned}
\end{equation*}
Then Proposition~\ref{prop1} in the optimistic case is changed into the following in the pessimistic case. 
\begin{pro}\label{propPBO}
	Suppose $ -F(\x, \y) $ and $f(\x, \y)$ are bounded below  and continuously differentiable. 
	Given $\x \in \X$ and $\mu, \theta, \sigma > 0$, when
	$\z^*_{\mu}(\x)$ and $\y_{\mu,\theta,\sigma}^*(\x)$
	are unique, 
	then $\varphi^p_{\mu,\theta,\sigma}(\x)$ is differentiable and
	$$
		\frac{ \partial \varphi^p_{\mu,\theta,\sigma} \left( \x \right) }{\partial \x} 
		= \frac{\partial  F \! \left(\x,\y_{\mu,\theta,\sigma}^*(\x) \right)} {\partial \x} + G(\x),
	$$
	with
	$
		\begin{aligned}
			G(\x) = \frac{ - \partial \PP_{\sigma} \! \left( f \left(\x,\y \right) - f_{\mu}^*(\x)   \right) }{ \partial \x } \big|_{ \y = \y_{\mu,\theta,\sigma}^*(\x)},		
		\end{aligned}
	$
	where
	$
	\begin{aligned}
		f_{\mu}^*(\x) = 
		f \! \left(\x,\z^*_{\mu}(\x) \right) 
		+ \frac{\mu}{2} \| \z^*_{\mu}(\x) \|^2,
	\end{aligned}
	$
	and
	$
	\frac{ \partial f_{\mu}^*(\x) }{\partial \x} = 
	\frac{\partial f \! \left(\x,\y \right) }{\partial \x} 
	\big|_{\y=\z^*_{\mu}(\x)}.
	$
\end{pro}

The proof is similar to Proposition~\ref{prop1}, obtained by applying~\cite[Theorem 4.13, Remark 4.14]{bonnans2013perturbation}.
The algorithm in the pessimistic case can then be derived similar to the optimistic case,
but when calculating $\y_{\mu,\theta,\sigma}^*(\x)$, 
gradient ascent is performed instead of gradient descent. 
In addition, the convergence of BVFSM for pessimistic BLO 
will be discussed in detail in Section~\ref{sec:theoretical investigations}.

\begin{remark}
\color{black}{
	The uniqueness of $\z^*_{\mu}(\x)$ can be guaranteed same to the analysis in Remark~\ref{remark 1}.
	Similarly, 
	suppose given $\x$, there exist constants $c_1 < \max_{\y \in \S(\x)} F(\x,\y)$ and $c_2 > \min_{\y} f(\x,\y),$ such that $F$ is concave and $f$ is convex in $\y$ on the set $\{\y: F(\x,\y) \ge c_1 \text{ and } f(\x,\y) \le c_2\}$,
	and we select a non-decreasing and convex auxiliary function $\PP_{\sigma}(\cdot)$.
	Or simply suppose there exists $c < \max_{\y} F(\x,\y) +  \PP_{\sigma} \! \left( f(\x,\y) - f_{\mu}^*(\x) \right)$ such that	
	$F(\x,\y) +  \PP_{\sigma} \! \left( f(\x,\y) - f_{\mu}^*(\x) \right)$ is concave on the set $\{\y: F(\x,\y) -  \PP_{\sigma} \! \left( f(\x,\y) - f_{\mu}^*(\x) \right) \ge c\}$.
	Then the uniqueness of $\y_{\mu,\theta,\sigma}^*(\x)$ follows.
	 
	
}
\end{remark}

\section{Theoretical Analysis}
This section brings out the convergence analysis and complexity analysis of the proposed BVFSM.

\subsection{Convergence Analysis}
\label{sec:theoretical investigations}

Here we show the convergence analysis of the proposed method.
As BLO without constraints can be seen as a special case of BLO with constraints
by regarding constraints as $H_{j}(\x,\y) \equiv 0, \forall j,$ and $h_{j'}(\x,\y) \equiv 0, \forall {j'}$,
we prove the more general constrained case.
Also, for brevity, we first prove in the optimistic BLO case, and the pessimistic case will be analyzed later in Corollary~\ref{corollary}.

Note that for sequential-minimization-type scheme,
including EGBMs and BVFSM,
\textcolor{black}{
	the convergence analysis can be classified into asymptotic and non-asymptotic convergence~\cite{liu2023averaged,ji2022lower}.
	This work considers asymptotic convergence and focuses on the approximation quality. 
	That is, whether the solutions to approximate problems
	converge to the original solution,
	which comes from the sequential approximated sub-problems
	converging to the original bi-level problem.
	We prove the asymptotic convergence from the aspect of global solution,
	and start by recalling the equivalent definition of epiconvergence given in~\cite[pp. 41]{bonnans2013perturbation}.
}
%
%

\begin{defi} \label{def_epic}
	$\varphi_k \stackrel{e}{\longrightarrow} \varphi$ if and only if for all $\x \in \mathbb{R}^m$, 
	the following two conditions hold:
	\begin{enumerate}
		\item for any sequence $\{\x_k\}$ converging to $\x$,
			\begin{equation}
				\label{eq:epi condition 1}
				\liminf \limits_{k \rightarrow \infty}\varphi_k(\x_k) \geq \varphi(\x);
			\end{equation}
		\item there is a sequence $\{\x_k\}$ converging to $\x$ such that
			\begin{equation}
				\label{eq:epi condition 2}
				\limsup \limits_{k \rightarrow \infty}\varphi_k(\x_k) \leq \varphi(\x).
			\end{equation}
	\end{enumerate}
	
%
	
\end{defi}

The convergence results are given under the following statements as our blanket assumption.

\begin{asu}[Assumptions for the problem] \label{assumption 1}
	$\quad$ 
	\begin{enumerate}
		\item $ \S(\x) $ is nonempty for $\x \in \X$.
		\item Both $F(\x,\y)$ and $f(\x,\y)$ are jointly continuous and continuously differentiable.
		Both $H_{j}(\x,\y)$, $\forall \ {j}$
		and $h_{j'}(\x,\y)$, $\forall \ {j'}$
		are continuously differentiable.
		\item $F(\x, \y)$ is level-bounded in $\y$ locally uniformly in $\x \in \X$ (see \cite[Definition 3]{liu2020generic}).
		\item For constrained BLO, $0$ is not a local optimal value of $h_{j'}(\x,\y)$ w.r.t. $\y$ for all ${j'}$.
	\end{enumerate}
%
%
%
\end{asu}

For the simplicity of symbols,
here we let $j=j'=1$, meaning that there is one constraint on the UL and LL problem respectively, 
and denote them as $H(\x,\y)$ and $h(\x,\y)$.
When $j \neq 1$ or $j' \neq 1$, the proofs parallel actually.
In addition, denote $f^*_{k}(\x) := f_{\mu_k, \sigma_{B,k}}^*(\x )$,
$\varphi_{k}(\x) := \varphi_{\mu_k,\theta_k,\sigma_k}(\x)$, and
$\PP_{k}(\omega) \! := \! \PP_{\sigma_k}(\omega)$ \textcolor{black}{defined in Eq.~\eqref{eq:\PP}}.
Also, $\PP_{B,k}(\omega), \PP_{H,k}(\omega), \PP_{h,k}(\omega), \PP_{f,k}(\omega)$ are defined similarly. 
Note that $\PP_{B,k}$ is the standard barrier function, 
while $\PP_{H,k}$, $\PP_{h,k}$, $\PP_{f,k}$ are penalty or modified barrier functions.~Then 
\begin{equation*} \label{eq:fk*}
	f^*_{k}(\x) = \min\limits_{\y} 
	\left\{ 
	f(\x,\y) + \PP_{B,k} \! \left( h(\x,\y) \right) + \frac{\mu_k}{2} \| \y \|^2
	\right\},
\end{equation*}
\begin{equation*} \label{eq:phik}
	\begin{aligned}	
		\varphi_k(\x) =  \min\limits_{\y} 
		\Big\{ 
		F(\x,\y) +  \PP_{H,k} \! \left( H(\x,\y) \right)
		+ \PP_{h,k} \! \left( h(\x,\y) \right) & \\
		+  \PP_{f,k} \! \left( f(\x,\y) - f_{k}^*(\x) \right)
		+ \frac{\theta_k}{2} \| \y \|^2
		\Big\} . &
	\end{aligned}
\end{equation*}

To begin with, we present the following lemma on the properties of penalty and modified barrier functions, as the preparation for further discussion and proofs\footnote{Proofs of the four lemmas are provided in Appendix~\ref{sec:appendix lemmas}, available at \url{https://arxiv.org/abs/2110.04974}.}. 

\begin{lemma}\label{lem1}
	\textcolor{black}{Let $\{\sigma_k\}$ in $\PP_{k}(\omega) \! = \! \PP_{\sigma_k}(\omega)$ be a positive sequence such that $\lim_{k \rightarrow \infty}\sigma_k = 0$.}
	Additionally assume that $\lim_{k \rightarrow \infty}\rho( - \eta_k \; ; \sigma_k^{(1)}) = 0$ when $\rho$ is a modified barrier function.
	Then we have 
	\begin{enumerate}
		\item $\PP_{k}(\omega)$ is continuous, differentiable and non-decreasing, and satisfies $\PP_{k}(\omega) \geq 0$. 
		\item For any $\omega \leq 0 $, $\lim_{k \rightarrow \infty} \PP_k(\omega) = 0$.
		\item For any sequence $\{\omega_k\}$, $\lim_{k \rightarrow \infty} \PP_k(\omega_k) < + \infty$ implies that $\limsup_{k \rightarrow \infty} \omega_k \leq 0 $.
	\end{enumerate}
	
%
%
\end{lemma}

We will use these properties in later proofs, 
so we hold these requirements on parameters in Lemma~\ref{lem1} from now on. 
To prove the convergence result, we verify the two conditions given in Definition~\ref{def_epic}, and show that $\varphi_k(\x)+\delta_{\X}(\x)\stackrel{e}{\longrightarrow}\varphi(\x)+\delta_{\X}(\x)$, 
where $\delta_\X(\x)$ denotes the indicator function of the set $\X$, 
i.e., $\delta_\X(\x) = 0 $ if $\x \in \X$ and $\delta_\X(\x) = + \infty$ if $\x \notin \X$. 
To begin with, we propose the following three lemmas to verify the two condition in Eq.~\eqref{eq:epi condition 1} and Eq.~\eqref{eq:epi condition 2} in Definition~\ref{def_epic}.

\begin{lemma}\label{lem2}
	Let $\{ (\mu_k, \sigma_{k}) \}$ be a positive sequence such that $ (\mu_k, \sigma_{k}) \rightarrow 0$, also satisfying the same setting as in Lemma~\ref{lem1}.
	Then for any sequence $\{\x_k\}$ converging to $\bar{\x}$,
	$
	\limsup\limits_{k \rightarrow \infty} f_{k}^*(\x_k) \le~f^*(\bar{\x}).
	$
\end{lemma}

\begin{lemma}\label{lem3}
\textcolor{black}{	
	Let $\{(\mu_k, \theta_k, \sigma_k)\}$ be a positive sequence such that $\lim_{k \rightarrow \infty}(\mu_k, \theta_k, \sigma_k) = 0$,
	and
	satisfy the same setting as in Lemma~\ref{lem1}.
	Given $\bar{\x}\in \X$, then for any sequence $\{ \x_k\}$ converging to $\bar{\x}$, we have
}
	$
		\liminf\limits_{k\rightarrow\infty} \varphi_{k}(\x_k)\geq \varphi(\bar{\x}).
	$
\end{lemma}

\begin{lemma}\label{lem4}
\textcolor{black}{
	Let $\{(\mu_k, \theta_k, \sigma_k)\}$ be a positive sequence such that $\lim_{k \rightarrow \infty}(\mu_k, \theta_k, \sigma_k) = 0$,
	and
	satisfy the same setting as in Lemma~\ref{lem1}.
	Then for any $\x \in \X$, }
	$
	\limsup\limits_{k \rightarrow \infty} \varphi_k(\x) \le \varphi(\x).
	$
\end{lemma}

\begin{table*}[hbtp]
	\centering
	\caption{Convergence of existing methods and BVFSM.
		We present the convergence results and conditions 
		whether it is available without the LLC condition, whether it can be extended to BLO with constraints and pessimistic BLO, respectively for each method.
		\textcolor{black}{Note that these convergence results are studied via two different types: the asymptotic and non-asymptotic analysis~\cite{liu2023averaged,ji2022lower}, and BVFSM achieves the asymptotic convergence without the LLC condition. 
		BVFSM can also be extended to BLO with constraints and pessimistic BLO which other methods cannot carry out.		
		}
	} 
	\label{table conditions}
	\vspace {-0.5cm}
	\begin{footnotesize}
		\begin{threeparttable}
\color{black}{
\resizebox{\linewidth}{!}{
			\begin{tabular}{|c|c|c|c|c|c|c|}
				\hline
				&&&&&&\\[-6pt]
				Category & Method & {Convergence Results} & Required Conditions
				& {w/o LLC} & Constraints & {Pessimistic} \\			
				&&&&&&\\[-6pt]
				\hline
				&&&&&&\\[-6pt]
				\multirow{8}{*}{EGBMs} & \multirow{2}{*}{FHG/RHG}
				& 
				Asymptotic 
				&    \multirow{1}{*}{$F(\x,\y)$ and $f(\x,\y)$ are $C^1$. }
				& \multirow{2}{*}{\XSolidBrush} & \multirow{2}{*}{\XSolidBrush} & \multirow{2}{*}{\XSolidBrush}\\
				
				& &   $\inf\limits_{\x \in \X}\varphi_k(\x) \rightarrow
				\inf\limits_{\x \in \X}\varphi(\x)$
				& $\S(\x)$ is a singleton. & && \\

				&&&&&&\\[-6pt]
				\cline{2-7}
				&&&&&&\\[-6pt]
				& \multirow{2}{*}{TRHG} 				
				& Non-asymptotic 
				&    $F(\x,\y)$ is $C^1$ and bounded below. 
				& \multirow{2}{*}{\XSolidBrush} & \multirow{2}{*}{\XSolidBrush} & \multirow{2}{*}{\XSolidBrush} \\
				
				& &    \multirow{1}{*}{$\x_k {\longrightarrow} \widehat{\x}^*$}
				& $f(\x,\y)$ is $C^1$, $L_f$-smooth and strongly convex.
				& && \\

				&&&&&&\\[-6pt]
				\cline{2-7}
				&&&&&&\\[-6pt]
				& \multirow{3}{*}{BDA} 				
				& 
				Asymptotic 
				&    
				$F(\x,\y)$ is $L_F$-smooth, convex, bounded below.
				& \multirow{3}{*}{\XSolidBrush} & \multirow{3}{*}{\XSolidBrush} & \multirow{3}{*}{\XSolidBrush}\\
				& & $\inf\limits_{\x \in \X}\varphi_k(\x) \rightarrow
				\inf\limits_{\x \in \X}\varphi(\x)$
				&\multirow{2}{*}{$f(\x,\y)$ is $L_f$-smooth. $\S(\x)$ is a singleton.}
				& && \\
				
				&&&&&&\\[-6pt]
				\hline
				&&&&&&\\[-6pt]
				\multirow{3}{*}{IGBMs} & \multirow{3}{*}{CG/Neumann}
				& Non-asymptotic 
				&    $F(\x,\y)$ and $f(\x,\y)$ are $C^1$.				
				& \multirow{3}{*}{\XSolidBrush} & \multirow{3}{*}{\XSolidBrush} & \multirow{3}{*}{\XSolidBrush}\\
				
				& &  \multirow{1}{*}{$\x_k {\longrightarrow} \widehat{\x}^*$}
				& $\frac{\partial^2 f(\x,\y)}{\partial \y \partial \y} $ is invertible. $\S(\x)$ is a singleton.
				& && \\

				&&&&&&\\[-6pt]
				\hline
				&&&&&&\\[-6pt]
				\multirow{3}{*}{Ours} & \multirow{3}{*}{BVFSM}				
				& Asymptotic 
				&  $F(\x,\y)$ and $f(\x,\y)$ are $C^1$  
				& \multirow{3}{*}{\CheckmarkBold} & \multirow{3}{*}{\CheckmarkBold} & \multirow{3}{*}{\CheckmarkBold}\\
				
				& &  $\inf\limits_{\x \in \X}\varphi_k(\x) \rightarrow
				\inf\limits_{\x \in \X}\varphi(\x)$
				&   \multirow{2}{*}{and level-bounded.}
				& && \\
				\hline
			\end{tabular}
}
}
			\begin{tablenotes}
				\footnotesize
				\item[1] $C^1$ denotes continuously differentiable.
				$L_f$ (or $L_F$)-smooth means the gradient of $f$ (or $F$) is Lipschitz continuous with Lipschitz constant $L_f$ (or $L_F$).
				``Level-bounded" is short for ``level-bounded in $\y$ locally uniformly in $\x\in\X$".
				\item[2] 
				$\widehat{\x}^*$ denotes the stationary point.
				
			\end{tablenotes}
		\end{threeparttable}
	\end{footnotesize}
\end{table*}

\begin{table*}[htbp]
	\caption{
		Complexity of existing gradient-based methods and BVFSM.
		We show the key update ideas for calculating $G(\x)$ in $\frac{\partial\varphi(\x)}{\partial \x}$. 
		Please see ~\cite{franceschi2017forward,pedregosa2016hyperparameter,lorraine2020optimizing} for more details of EGBMs and IGBMs.  
		Note that our method avoids solving an unrolled dynamic system or approximating the inverse of Hessian.
	} 
	\label{table complexity}
	\vspace {-0.7cm}
\color{black}{
	\begin{center}
\resizebox{\linewidth}{!}{
			\begin{tabular}{|c|c|cc|c|c|}
				\hline
				&&&&&\\[-6pt]
				Category&Method&\multicolumn{2}{c|}{Key point for calculating $G(\x)$}&Time&Space\\
				&&&&&\\[-6pt]
				\hline
				&&&&&\\[-6pt]
				 \multirow{8}{*}{EGBMs}&\multirow{2}{*}{FHG}&\multirow{2}{*}{$ G(\x) \approx \Z_{T}^\top \frac{\partial F(\x,\y_T)}{\partial \y}$}&\multirow{2}{*}{$\Z_{t}= \frac{\partial^2 f}{\partial \y^2}\Z_{t-1} +  \frac{\partial^2 f}{\partial \y\partial \x }$}
				&\multirow{2}{*}{$O(m^2nT)$}&\multirow{2}{*}{$O(mn)$}\\
				&&&&&\\
				&\multirow{2}{*}{RHG}&\multirow{2}{*}{$G(\x) \approx \q_{-1}$}&\multirow{4}{*}{$ \q_{t-1}=\q_{t}+ \left( \frac{\partial^2 f}{\partial \x \partial \y}\right)^\top\p_t, \ \p_{t-1}=\left( \frac{\partial^2 f}{\partial \y^2}\right)^\top  \p_{t}$}&\multirow{2}{*}{$O(n(m+n)T)$}&\multirow{2}{*}{$O(m+nT)$}\\
				&&&&&\\
				&\multirow{2}{*}{TRHG}&\multirow{2}{*}{$G(\x) \approx \q_{I-1}$}&&\multirow{2}{*}{$O(n(m+n)I)$}&\multirow{2}{*}{$O(m+nI)$}\\
				&&&&&\\
				&\multirow{2}{*}{BDA}&\multirow{2}{*}{$G(\x) \approx \q_{-1}$}&\multirow{2}{*}{Same as RHG, but replace $f$ with $(1-\alpha)f+\alpha F$
				}&\multirow{2}{*}{$O(n(m+n)T)$}&\multirow{2}{*}{$O(m+nT)$}\\
				&&&&&\\
				&&&&&\\[-6pt]
				\hline
				&&&&&\\[-6pt]
				\multirow{4}{*}{IGBMs}&\multirow{2}{*}{CG}&\multirow{4}{*}{$G(\x) \approx -\left( \frac{\partial^2  f(\x,\y_T)}{ \partial \y \partial \x}\right)^\top\q$}&\multirow{2}{*}{$ \frac{\partial^2 f}{\partial \y^2} \q=\frac{\partial F}{\partial \y} $}&\multirow{2}{*}{$O(m+nT+n^2Q)$}&\multirow{2}{*}{$O(m+n)$}\\
				&&&&&\\
				&\multirow{2}{*}{Neumann}&&\multirow{2}{*}{$\q=\sum_{i=0}^{Q}\left(\mathbf{I}-\frac{\partial^{2} f}{\partial \mathbf{y}^2}\right)^{i}\frac{\partial F}{\partial \y}$}&\multirow{2}{*}{$O(m+nT+n^2Q)$}&\multirow{2}{*}{$O(m+n)$} \\
				&&&&&\\
				&&&&&\\[-6pt]
				\hline
				&&&&&\\[-6pt]
				\multirow{3}{*}{Ours}&\multirow{3}{*}{BVFSM}&\multirow{3}{*}{
				$G(\x) \approx \frac{ \partial \PP_{\sigma_k} \! \left( f \left( \x_l,\y_{k,l}^{T_{\y}} \right) - f_{k,l}^{T_{\z}}   \right) }{ \partial \x }	$
				}& 
				\multirow{3}{*}{$f_{k,l}^{T_{\z}} = 
				f(\x_l, \z_{k,l}^{T_{\z}}) 
				+ \frac{\mu_{k}}{2} \| \z_{k,l}^{{T_{\z}}} \| ^2$}
				&\multirow{3}{*}{$O(m+n(T_{\z}+T_{\y}))$}&\multirow{3}{*}{$O(m+n)$}
				\\
				&&& 
				&&\\     
				&&&&&\\
				\hline
			\end{tabular}
}
	\end{center}
}
\end{table*}

Now, by combining the above results,
we can obtain the desired epiconvergence result,
which also indicates the convergence of our method.
\textcolor{black}{
	Note that this is another type of the convergence of algorithm iterates in asymptotic convergence different from non-asymptotic convergence.
}

\begin{theorem} [Convergence for Optimistic BLO] \label{thm}
	$\quad$
	
	Let $\{(\mu_k, \theta_k, \sigma_k)\}$ be a positive sequence such that $( \mu_k, \theta_k, \sigma_k) \rightarrow 0$, also satisfying the same setting as in Lemma~\ref{lem1}. 
	\begin{enumerate}
		\item The epiconvergence holds:
			$$
				\varphi_k(\x)+\delta_{\X}(\x)\stackrel{e}{\longrightarrow}\varphi(\x)+\delta_{\X}(\x).
			$$
		\item We have the following inequality:
			\begin{equation*}
				\limsup \limits_{k \rightarrow \infty}\left(\inf \limits_{\x \in \X}\varphi_k(\x) \right)\leq \inf \limits_{\x \in \X}\varphi(\x).
			\end{equation*}
			In addition, if $\x_\ell \in \mathrm{argmin}_{\x \in \X}\varphi_\ell(\x)$ for some subsequence $\{\ell\} \subset \mathbb{N}$, and $\x_\ell$ converges to $\tilde{\x}$, 
			then $\tilde{\x} \in \mathrm{argmin}_{\x\in \X} \varphi(\x)$
			and
			\begin{equation*}
				\lim \limits_{\ell \rightarrow \infty}\left (\inf \limits_{\x \in \X}\varphi_\ell(\x) \right)= \inf \limits_{\x \in \X}\varphi(\x).
			\end{equation*}
	\end{enumerate}
	
%
\end{theorem}

\begin{proof}
	To prove the epiconvergence of $\varphi_k$ to $\varphi$, we just need to verify that the sequence $\{\varphi_k\}$ satisfies the two conditions given in Definition \ref{def_epic}. 
	Considering any sequence~$\{\x_k\}$ converging to $\x$, if $\x \in \X$, from Lemma \ref{lem3} we have
	\[
	\begin{aligned}
		\varphi(\x) + \delta_\X(\x) = \varphi(\x) 
		&\le  \liminf_{k \rightarrow \infty} \varphi_k(\x_k) \\ 
		&\le \liminf_{k \rightarrow \infty} \varphi_k(\x_k) + \delta_\X(\x_k).
	\end{aligned}
	\]
	When $\x \notin \X$, we have $\liminf_{k \rightarrow \infty} \varphi_k(\x_k) + \delta_\X(\x_k) = + \infty$ because $\X$ is closed. 
	Thus the first condition Eq.~\eqref{eq:epi condition 1} in Definition \ref{def_epic} is satisfied. 
	
	Next, for any $\x \in \mathbb{R}^m$, if $\x \in \X$, then it follows from Lemma \ref{lem4} that
	\[
	\begin{aligned}
		\limsup_{k \rightarrow \infty} \varphi_k(\x) + \delta_\X(\x) 
		& = \limsup_{k \rightarrow \infty} \varphi_k(\x) \\
		& \le \varphi(\x) 
		= \varphi(\x) + \delta_\X(\x).
	\end{aligned}
	\]
	When $\x \notin \X$, we have $\varphi(\x) + \delta_\X(\x) = + \infty$. 
	Thus, the second condition Eq.~\eqref{eq:epi condition 2} in Definition \ref{def_epic} is satisfied. Therefore, we get the conclusion (1) immediately from Definition \ref{def_epic},
	%
	and the conclusion (2) follows from \cite[Proposition 4.6]{bonnans2013perturbation}.
\end{proof}

Next, we consider the convergence for pessimistic BLO.
To begin with, for pessimistic BLO without functional constraints, we denote $\varphi^p_k(\x)$ similarly to the optimistic case: 
\[
\begin{aligned}	
	\varphi^p_k(\x) := &\varphi^p_{\mu_k,\theta_k,\sigma_k}(\x)  \\ 
	= & \max\limits_{\y} 
	\Big\{ 
	F(\x,\y) 
	-  \PP_{k} \! \left( f(\x,\y) - f_{k}^*(\x) \right)
	- \frac{\theta_k}{2} \| \y \|^2
	\Big\}, 
\end{aligned}
\]
where $
	f^*_{k}(\x) = \min_{\y} 
	\left\{ 
	f(\x,\y) + \frac{\mu_k}{2} \| \y \|^2
	\right\}.
$
Then we have the following corollary. 
Note that this convergence result can also be extended to pessimistic BLO with constraints easily.

\begin{corollary}[Convergence for Pessimistic BLO] \label{corollary}
	$\quad$
	
	Let $\{(\mu_k, \theta_k, \sigma_k)\}$ be a positive sequence such that $( \mu_k, \theta_k, \sigma_k) \rightarrow 0$, also satisfying the same setting as in Lemma~\ref{lem1}. 
	Then
		we have the following inequality:
		\begin{equation*}
			\limsup \limits_{k \rightarrow \infty}\left(\inf \limits_{\x \in \X}\varphi^p_k(\x) \right)\leq \inf \limits_{\x \in \X}\varphi^p(\x).
		\end{equation*}
		In addition, if $\x_\ell \in \mathrm{argmin}_{\x \in \X}\varphi^p_\ell(\x)$ for some subsequence $\{\ell\} \subset \mathbb{N}$, and $\x_\ell$ converges to $\tilde{\x}$, 
		then we have $\tilde{\x} \in \mathrm{argmin}_{\x\in \X} \varphi^p(\x)$
		and
		\begin{equation*}
			\lim \limits_{\ell \rightarrow \infty}\left (\inf \limits_{\x \in \X}\varphi^p_\ell(\x) \right)= \inf \limits_{\x \in \X}\varphi^p(\x).
		\end{equation*}
	
\end{corollary}

\begin{proof}
	Based on the proof of Theorem~\ref{thm}, we first need to prove 
	$
		\varphi^p_k(\x)+\delta_{\X}(\x)\stackrel{e}{\longrightarrow}\varphi^p(\x)+\delta_{\X}(\x)
	$
	by Lemma~\ref{lem1},~\ref{lem2},~\ref{lem3}, and~\ref{lem4}.
	Lemma~\ref{lem1} and~\ref{lem2} are unrelated to whether it is the optimistic or pessimistic case, and thus holds naturally.
	The corresponding results for Lemma~\ref{lem3} and~\ref{lem4} can be derived simply by replacing $F$ in their proof with $-F$.
	Then the conclusion can be obtained by the process same to~Theorem~\ref{thm}.
\end{proof}

In Table~\ref{table conditions}, we present the comparison among existing methods and our BVFSM. 
\textcolor{black}{
It can be seen that under mild assumptions, BVFSM is able to achieve {asymptotic} convergence without the LLC restriction, and be applied in BLO with constraints and pessimistic BLO,
which is not available by other methods.
In addition, as shown in Theorem~\ref{thm}, our asymptotic convergence is obtained from the epiconvergence property, which is a stronger result than solely asymptotic~convergence. 
}

\subsection{Complexity Analysis}\label{sec:complexity}
%
In this part, we compare the time and space complexity of Algorithms~\ref{alg:ours} 
with EGBMs (i.e., FHG, RHG, TRHG and BDA) and IGBMs (i.e., CG and Neumann) 
for computing $G(\x)$ or $\frac{\partial\varphi(\x)}{\partial \x}$, 
i.e., the direction for updating variable $\x$.
Table~\ref{table complexity} summarizes the complexity results.
Our complexity analysis follows the assumptions in \cite{liu2018darts}. 
\textcolor{black}{
Note that BVFSM has an order of magnitude lower time complexity with respect to the LL dimension $n$ compared to existing methods.}
For all existing methods, we assume 
solving the optimal solution of the LL problem,
also the transition function $\Phi$ in EGBMs for obtaining $\y_T(\x)$,
is the process of a $T$-step gradient descent.

%

\textbf{EGBMs.} 
As discussed in~\cite{franceschi2017forward,shaban2019truncated}, 
after implementing $T$ steps of gradient descent with time and space complexity of~$O(n)$ to solve the LL problem,
FHG for forward calculation of Hessian-matrix product can be evaluated in time $O(m^2nT)$ and  space $O(mn)$, and RHG for reverse calculation of Hessian- and Jacobian-vector products can be evaluated in time $O(n(m+n)T)$ and space $O (m+nT)$. 
TRHG truncates the length of back-propagation trajectory to~$I$ after a $T$-step gradient descent, 
and thus reduces the time and space complexity to $O(n(m+n)I)$ and  space $O(m+nI)$.
BDA uses the same idea to RHG, except that it combines UL and LL objectives during back propagation,
so the order of complexity of time and space is the same to RHG.
The time complexity for EGBMs to calculate the UL gradient is proportional to $T$, the number of iterations of the LL problem,
and thus EGBMs take a large amount of time to ensure convergence.

\textbf{IGBMs.} After implementing a $T$-step gradient descent for the LL problem, IGBMs approximate the inverse of Hessian matrix
 by conjugate gradient (CG), which solves a linear system of $Q$ steps,
or by Neumann series.
Note that each step of CG and Neumann method includes Hessian-vector products, 
requiring $O(m+n^2Q) $ time and $O(m+n)$ space,
so IGBMs run in time $O(m+nT+n^2Q)$ and space~$O(m+~n)$. 
IGBMs decouple the complexity of calculating the UL gradient from being proportional to $T$, 
but the iteration number $Q$ always relies on the properties of the Hessian matrix, 
and in some cases, $Q$ can be much larger than $T$.

\textbf{BVFSM.} 
In our algorithm, it takes time $O(nT_{\z})$ and space~$O(n)$
to calculate $T_{\z}$ steps of gradient descent on Eq.~\eqref{eq:z*} for the solution of LL problem $\z^{T_{\z}}$. 
Then $T_{\y}$ steps of gradient descent on Eq.~\eqref{eq:y*} are used to calculate $\y^{T_{\y}}$, which requires time $O(nT_{\y})$ and space $O(n)$.
After that, the direction can be obtained according to the formula given in Eq.~\eqref{eq:a_g_vphi} by several computations of the gradient $\frac{\partial f}{\partial \x}$ and $\frac{\partial F}{\partial \x}$ without any intermediate update, which requires time $O(m)$ and space $O(m+n)$.
Therefore, BVFSM runs in time $O(m+n(T_{\z}+T_{\y}))$ and space $O(m+n)$ for each iteration.

It can be observed from Table~\ref{table complexity} that BVFSM needs less space than EGBMs,
and 
it takes much less time than EGBMs and IGBMs, 
especially when $n$ is large, meaning the LL problem is high-dimensional, 
such as in application tasks with a large-scale network. 
Overall, this is because BVFSM does not need any computation of Hessian- or Jacobian-vector products for solving the unrolled dynamic system by recurrent iteration or approximating the inverse of Hessian.
Its complexity only comes from calculating the gradients of $F$ and $f$,
which is much easier than calculating Hessian- and Jacobian-vector products (even by AD).
Besides, although BVFSM has the same order of space complexity to IGBMs,
it is indeed smaller,
{because the memory is saved by eliminating the need to save the computational graph used for calculating Hessian.}
We will further verify these advantages through numerical results in Section~\ref{sec:experiments}.


\begin{figure*}[!tbp]
	\centering  
	{	
		\subfigure[Convergence with the initial point $(\x,\y)=(8,8,8)$]{
			\label{subfig:initial points 888}
			\includegraphics[height=3cm,width=4cm]{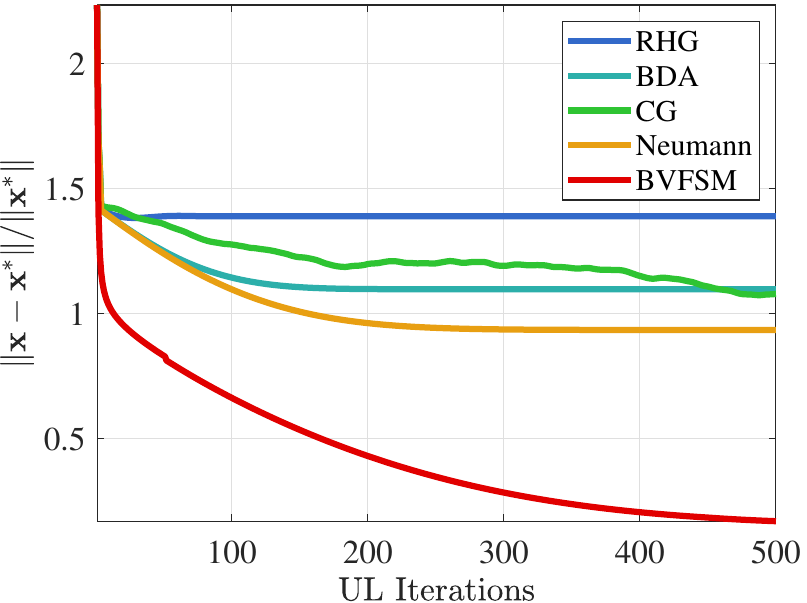}
			\includegraphics[height=3cm,width=4cm]{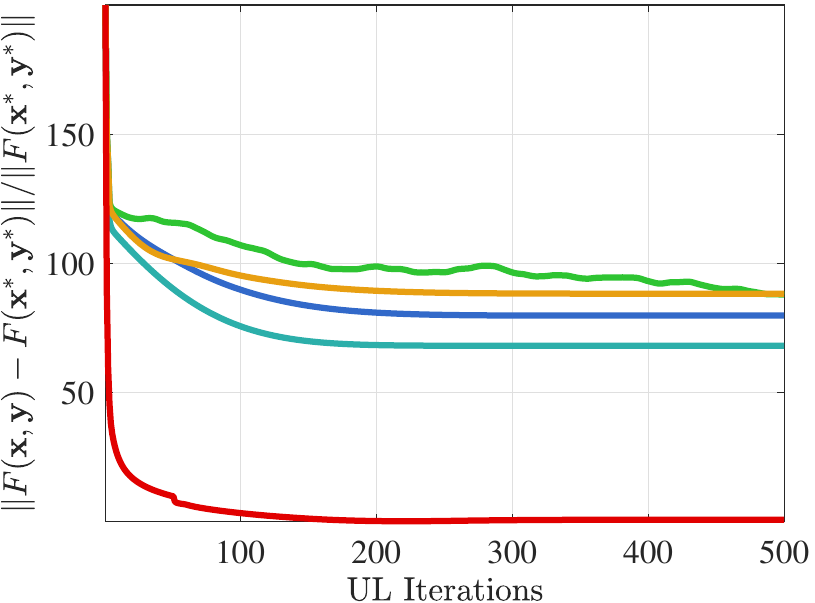}  
		}
		\subfigure[Convergence with the initial point $(\x,\y)=(0,0,0)$]{
			\label{subfig:initial points 000}
			\includegraphics[height=3cm,width=4cm]{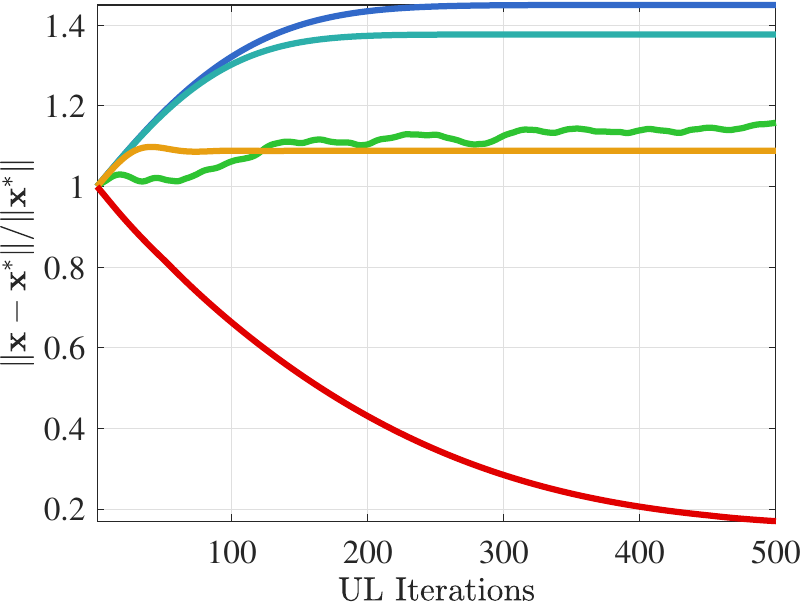}  
			\includegraphics[height=3cm,width=4cm]{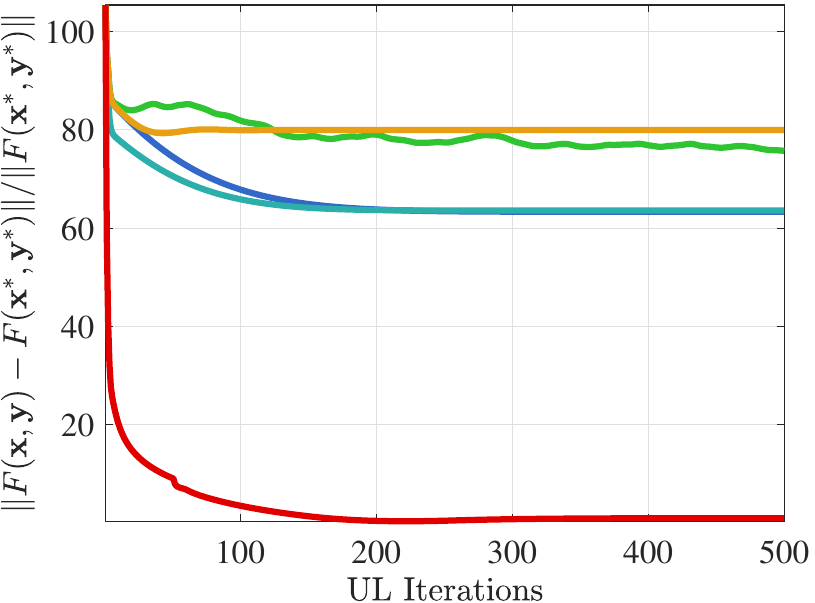} 
		}
	}	
	\caption{Convergence behavior of gradient-based BLO algorithms with different initial points. The (1st,3rd) and (2nd,4th) subfigures respectively show the errors of UL variable (i.e., $\|{\x}-{\x}^*\|/\|{\x}^*\|$) and UL objective (i.e., $\| F({\x},{\y})-F({\x^*},{\y^*})\|/ \| F({\x^*},{\y^*})\|$).
		The legend is only shown in the first~plot.
	}
	\label{fig:convergence}
\end{figure*}

\section{Experimental Results}\label{sec:experiments}
In this section, we quantitatively demonstrate the performance of our BVFSM\footnote{\textcolor{black}{Code is available at \url{https://github.com/vis-opt-group/BVFSM}.}}, especially when dealing with complicated and high-dimensional problems.
We start with investigating the convergence performance, computational efficiency, and effect of hyper-parameters on numerical examples in Section~\ref{sec:8.1}. 
In Section~\ref{sec:ho}, we apply BVFSM in the hyper-parameter optimization for the data hyper-cleaning task 
under different settings including the type of auxiliary functions, contamination rates, and  various network structures.
To further validate the generality of our method, we conduct experiments on other tasks such as few-shot learning in Section~\ref{sec:meta} and GAN in Section~\ref{sec:gan}.
The experiments were conducted on a PC with Intel Core i7-9700K CPU (4.2 GHz), 
32GB RAM and an NVIDIA GeForce RTX 2060S GPU with 8GB VRAM, and the algorithm was implemented using PyTorch 1.6. 
We use the implementation in~\cite{grefenstette2019generalized,grazzi2020iteration} for the existing methods, 
and use MB (MegaByte) and S (Second) as the evaluation units of space and time complexity, respectively. 
\textcolor{black}{
	Regarding the selection of coefficients and hyper-parameters, we evaluate them in numerical experiments and use the same method to select them in later tasks. 
	Furthermore, 
	we set $\y_{k,l}^{0}=\y_{k,l-1}^{T_{\y}}$, $\z_{k,l}^{0}=\z_{k,l-1}^{T_{\z}}$ to initialize each step of the sub-problems. 
	In view of the optimizer, we use SGD for solving LL and UL sub-problems in numerical experiments. 
	In some applications, we change the UL optimizer to Adam to speed up the convergence.
}

\subsection{Numerical Evaluations}\label{sec:8.1}
\subsubsection{\textcolor{black}{ Optimistic BLO }} \label{sec:toy optimistic}
\textcolor{black}{We start with the optimistic BLO, and}
use the numerical example with a non-convex LL which can adjust various dimensions to validate 
\textcolor{black}{the effectiveness of BVFSM} over existing methods.
In particular, 
consider
\begin{equation}
    \label{eq:non-convexExperiment}
    \begin{aligned}
        & \min_{ \x \in \mathbb{R}, \y \in \mathbb{R}^n} \| \x-a\|^2+\| \y-a-\c\|^2 \\
	    & \text{\ s.t.\ }\;  [\y]_i \in \underset{ [\y]_{i} \in \mathbb{R}}{\mathrm{argmin}}\; \sin( \x+ [\y]_i-[\c]_i), \forall \ i,
    \end{aligned}
\end{equation}
where $[\y]_i$ denotes the $i$-th component of $\y$, 
while $a \in \mathbb{R}$ and $\c\in \mathbb{R}^n$ are adjustable parameters. 
Note that here $\x \in \mathbb{R}$ is a one-dimensional real number,
but we still use the bold letter to represent this scalar to maintain the context consistency.
The solution of such problem is
$
\x^*=\frac{(1-n) a+n C}{1+n}, \ \text { and } \ [\y^*]_{i}=C+[\c]_i-\x^* , \forall \ i,
$
where
$
C=\underset{{k}}{\operatorname{argmin}}\left\{\|C_k-2 a\| : C_k=-\frac{\pi}{2}+2 k \pi, k \in \mathbb{Z}\right\} ,
$
and the optimal value is $F^{*}=\frac{n(C-2 a)^{2}}{1+n}$\footnote{	
\textcolor{black}{Derivation of the closed-form solution is provided in Appendix~\ref{sec:appendix closed-form solution}, available at \url{https://arxiv.org/abs/2110.04974}.}}. 
This example satisfies all the assumptions of BVFSM, but does not meet the LLC assumption in ~\cite{pedregosa2016hyperparameter,rajeswaran2019meta,lorraine2020optimizing}, which makes it a good example to validate the advantages of BVFSM. 
\textcolor{black}{In the following experiments we set $a=2$ and $[\c]_i=2, \text{ for any }i = 1,2,\cdots,n$.}

We compare BVFSM with several gradient-based optimization methods, including RHG, BDA, CG and Neumann. 
Note that they all assume the solution of the LL problem is unique except BDA, 
so for these methods we directly regard the obtained local optimal solutions of LL problems
as the unique solutions.
We set $T=100$ for RHG and BDA, $T=100$, $Q=20$ for CG and Neumann, 
the aggregation parameters equal to $0.5$ in BDA,
and $(\mu_{k}, \theta_k, \sigma_k^{(1)})=(1.0, 1.0, 1.0)/1.01^k$, $\sigma_k^{(2)}=f(\x_k,\y_k)+1$, 
step size $\alpha=0.01$,  $T_{\z}=50$, $T_{\y}=25$, and $L=1$ in BVFSM.


\begin{table}[tbp]
	\centering
	
	\caption{Errors of UL variable $\|{\x}-{\x}^*\|/\|{\x}^*\|$ 
		with large-scale LL of $n$ dimension.}
\vspace{-0.2cm}
	\begin{tabular}{|c|c|c|c|c|c|}
		
		\hline
		$n$ & RHG   & BDA   & CG    & Neumann & BVFSM \\
		\hline
		50    & 2.296 & 2.336 & 2.058 & 2.260 &\textbf{0.117} \\
		100   & 2.253 & 2.294 & 2.073 & 2.236 & \textbf{0.159} \\
		150   & 2.213 & 2.253 & 2.032 & 2.202 & \textbf{0.190} \\
		200   & 2.187 & 2.227 & 1.972 & 2.178 & \textbf{0.209} \\
		\hline	
	\end{tabular}%
	\label{tab:highcon}%
\end{table}%

\textcolor{black}{\textbf{Convergence performance.} }
Figure~\ref{fig:convergence} compares the convergence curves of UL variable $\x$ and objective $F(\x,\y)$ in the 2-dimensional case ($n=2$).
Here the optimal solution is $(\x^*,\y^*)=(-2/3+\pi,8/3+\pi/2,8/3+\pi/2)$. 
In order to show the impact of initial points, we also set different initial points.
From Figure~\ref{subfig:initial points 888}, when the initial point is $(\x,\y)=(8,8,8)$, 
existing methods show the trend of convergence at the beginning of iteration, but they soon stop further converging
due to falling into a local optimal solution. 
\textcolor{black}{
Furthermore, when the initial point is $(\x,\y)=(0,0,0)$ in Figure~\ref{subfig:initial points 000}, 
existing methods show a trend that the distance to the optimal solution even increases
during the whole iterative process 
because they incorrectly converge to the local solution away from the global solution.
}
On the contrary, our method can converge to the optimal solution under different initial points.
Table~\ref{tab:highcon} further verifies the convergence performance for larger-scale problems of various LL dimension $n$.
It shows that our method can still maintain good convergence performance with high-dimensional LL, 
while existing methods fail because they cannot solve the non-convex LL with convergence guarantee.

\begin{figure} [tbp]
	\centering  
\color{black}{
	\includegraphics[height=3cm,width=7.5cm,trim=0 50 0 50]{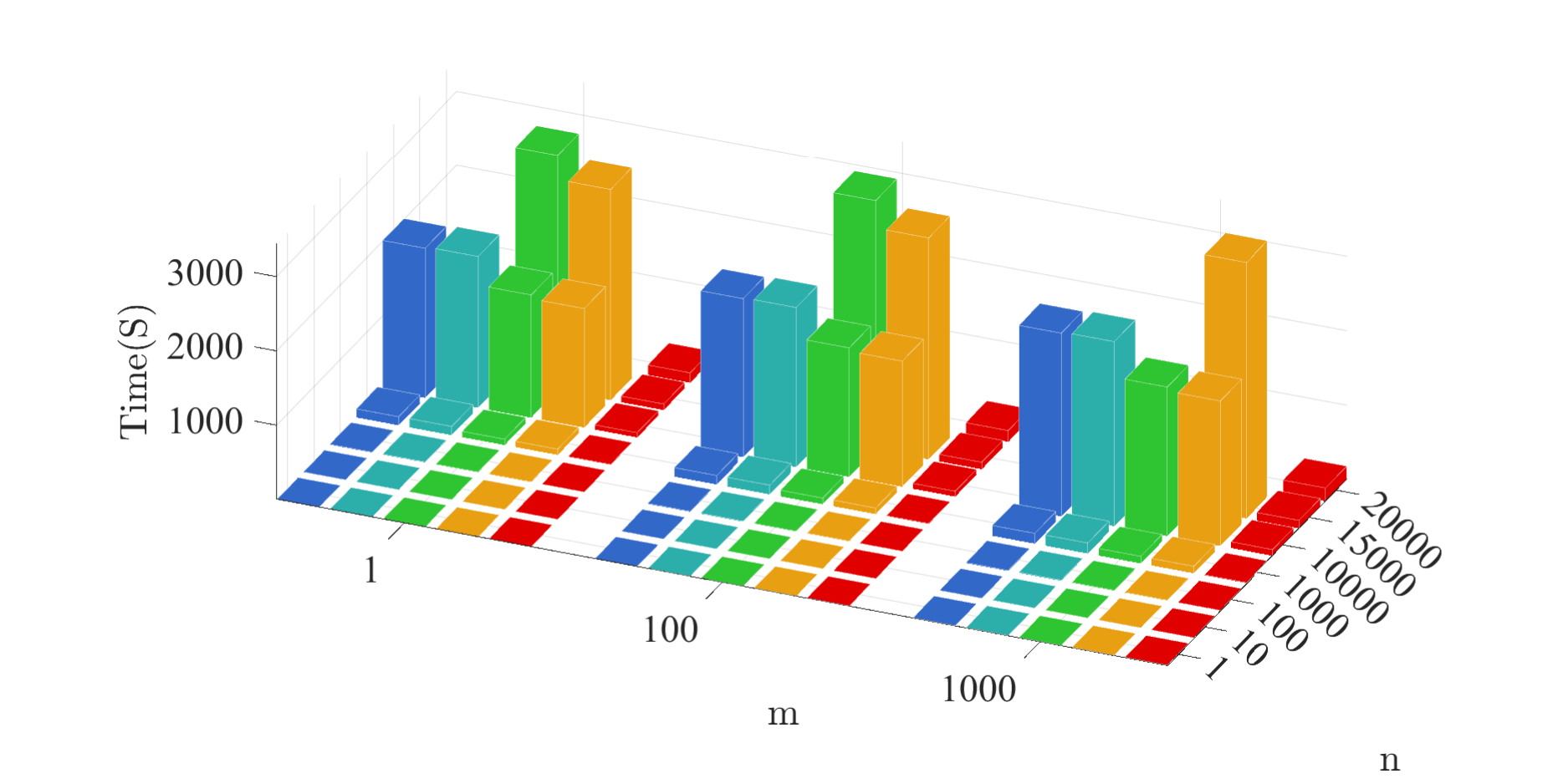}  
	\includegraphics[height=2cm,width=1cm,trim=0 -50 0 -50]{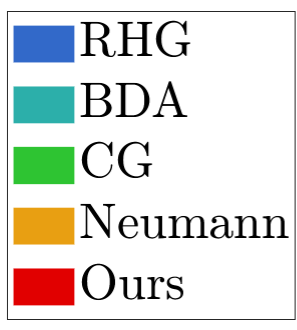}
	\caption{\textcolor{black}{Computation time (seconds, S) in each step for calculating the gradient of $\x$ under various scales $n$ and $m$ (LL and UL dimensions). Blank areas not drawn indicate that the 3600-second time limit is~exceeded}.
	}
	\label{fig:Timesimu}  
}
\end{figure}

\begin{table}[!tbp]
	\centering
\color{black}{
	\caption{
	The largest LL dimension $n$ that can be achieved by different methods 
	for a single-step computation within 3600 seconds.}
\vspace{-0.2cm}
	\begin{tabular}{|c|c|c|c|c|c|}

		\hline
		 & RHG   & BDA   & CG    & Neumann & BVFSM \\
		\hline
		$n$& 13089 & 12871 & 15093 & 18118 &\textbf{283200} \\
		\hline	
	\end{tabular}%
	\label{tab:time}%
}
\end{table}%



\textcolor{black}{\textbf{Computational efficiency for large-scale problems.} }
{Figure~\ref{fig:Timesimu} } compares the computation time for problems under various scales $n$ and $m$. 
Note that the scale-up of UL dimension $m$ can be achieved by converting the one-dimensional~$\x$ to the mean of multi-dimensional $\x$. 
As we can see, our method costs the least computation time for problems of all scales, and the LL dimension $n$ has much more influence than the UL dimension $m$. 
\textcolor{black}{
Table~\ref{tab:time} shows the largest LL dimension within the 3600-second time limit.}
This allows us to apply BVFSM to more complex LL problems, which existing methods cannot deal with. 
We attribute these superior results to our novel way of the re-characterization via value-function.
We further explore our performance on problems with complex network structures in Section~\ref{sec:ho}. 

\begin{figure} [tbp]
	\centering  
	\subfigure[Errors when $\mu_0,\theta_0=0.01$]{
		\label{subfig:sigma0.01}	
		\includegraphics[height=3cm,width=4cm]{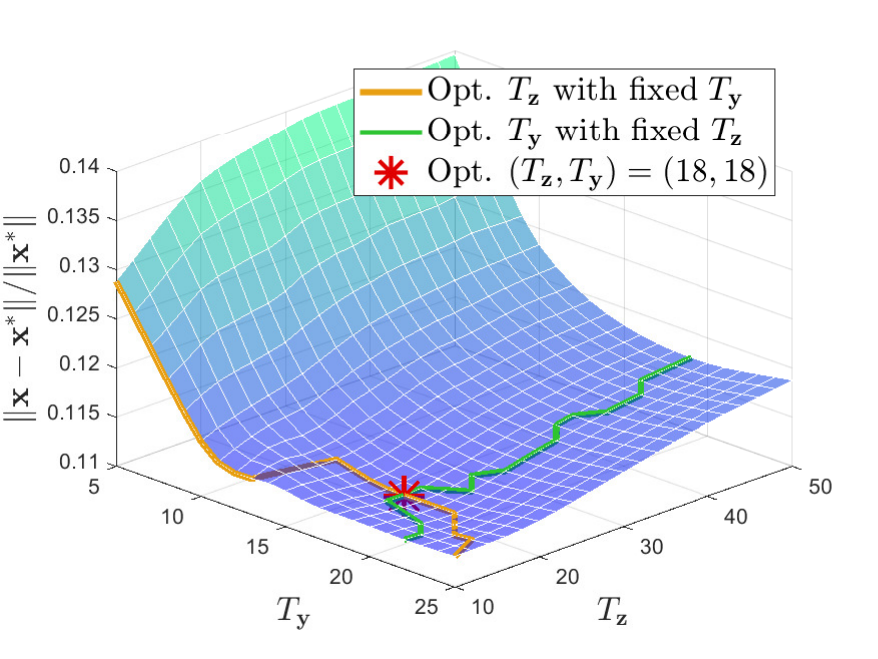} 
		\includegraphics[height=3cm,width=4cm]{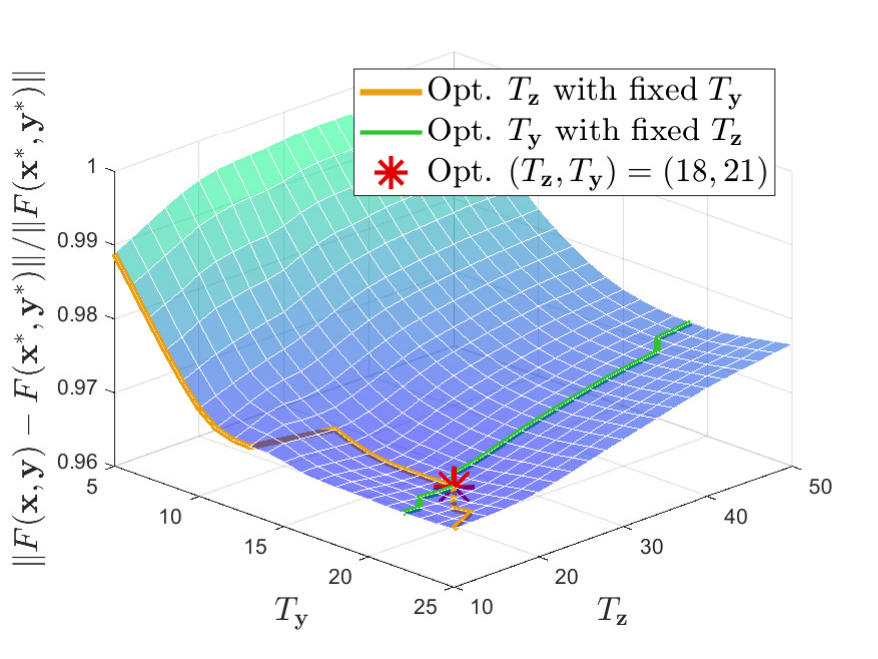} 
	}
	\subfigure[Errors when $\mu_0,\theta_0=0.001$]{
		\label{subfig:sigma0.001}
		\includegraphics[height=3cm,width=4cm]{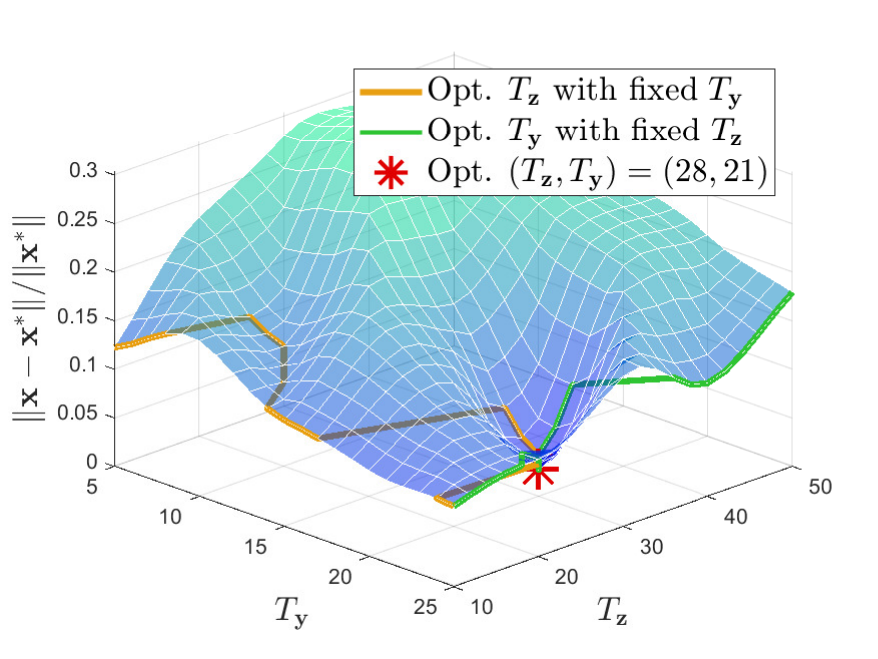} 
		\includegraphics[height=3cm,width=4cm]{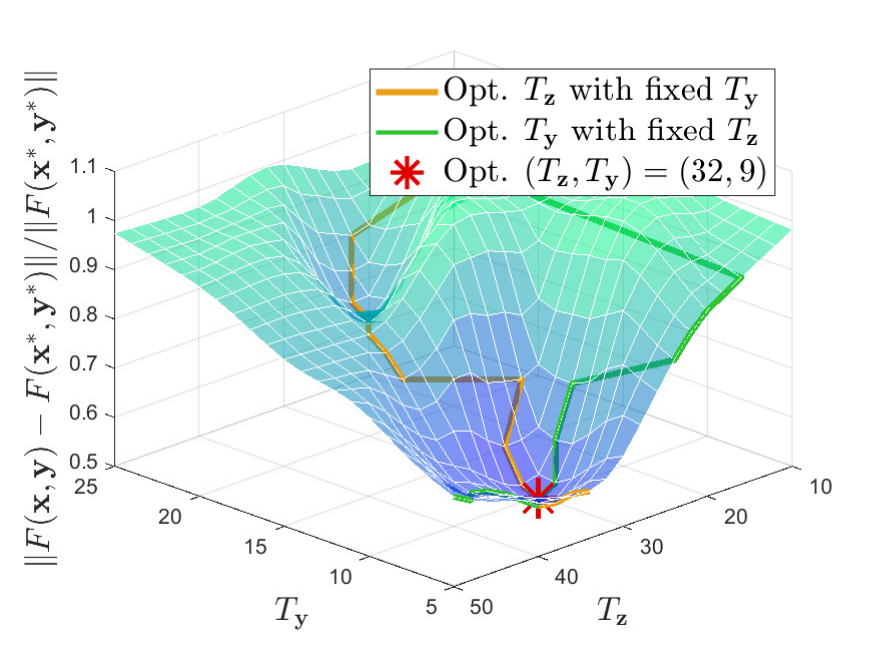} 
	}
	\caption{ 
		Effect of specified $T_{\z}$, $T_{\y}$, $\mu$ and~$\theta$ on the solution errors. 
		Here $T_{\z}$ and $T_{\y}$ are the number of iterations for gradient descent on LL sub-problems in Eq.~\eqref{eq:z*} and simple BLO problems in Eq.~\eqref{eq:y*}, 
		which affect the accuracy of obtained solution.
		The curves on the surfaces denote the optimal $T_{\z}$ (or $T_{\y}$) to minimize the error with fixed $T_{\y}$ (or $T_{\z}$).
		Subfigure~\subref{subfig:sigma0.01} shows that the increasing of $T_{\y}$ and $T_{\z}$ abnormally leads to the deviation of the solution, 
		because	with larger regularization coefficients $\mu$ and $\theta$, the regularized problems are away from the original problems.
		Subfigure~\subref{subfig:sigma0.001} shows that smaller $\mu$ and $\theta$ may not completely overcome the ill condition of the LL problem, resulting in the instability of the surfaces
		\textcolor{black}{(in the second plot we reverse the axis to make it easier to observe the trend)}.
    } 
	\label{fig:hyperparameters}
\end{figure}

\textcolor{black}{\textbf{Effect of hyper-parameters.}}
\textcolor{black}{We next evaluate the effect of various hyper-parameters in the 2-dimensional case ($n=2$).}
In Figure~\ref{fig:hyperparameters}, we compare the errors under different settings of~$T_{\z}$, $T_{\y}$, $\mu$ and~$\theta$. 
In Figure~\ref{subfig:sigma0.01},
with larger regularization coefficients $\mu$ and $\theta$, the regularized problems are away from the original problems.
Figure~\ref{subfig:sigma0.001} shows that smaller~$\mu$ and $\theta$ may not completely overcome the ill condition of the LL problem, 
which means the small coefficients cause the approximate problem to remain a little ill-conditioned,
resulting in the instability of the surfaces.
Hence, it is not an easy task to determine the regularization coefficients.
Since the selection of such parameters is often highly related to the specific problem, in order to maintain the fairness for comparing the computational burden with other methods, we set $T_{\z}+2T_{\y}=T$ in all experiments (because we need to calculate the gradient of two functions $F$ and $f$ for the $T_{\y}$-step gradient descent).
\textcolor{black}{
	Figure~\ref{fig:hyperparameter} shows the effect of $\mu$,~$\theta$, and~$\sigma$ on the convergence results.
	Figure~\ref{subfig:theta and mu} reveals the effect of regularization coefficients $\mu$ and~$\theta$. 
	We find that when $\mu=0$, the collapse may occur (with the collapse rate at around $44\%$), which indicates the necessity of adding the regularization term to avoid collapse and improve the computational stability.
	Figure~\ref{subfig:theta mu sigma} shows that it is a good choice to use smaller $\mu, \theta$ and larger $\sigma$ with a suitable decay factor to avoid the offset of solution and achieve better convergence.
	Figure~\ref{fig:L} further analyzes the effect of $L$, the number of inner-loop iterations, on the convergence speed. 
	It can be seen that the smaller $L$ is, the higher convergence speed can be obtained, so we set~$L=1$ in all experiments.
}

\begin{figure} [tbp]
	\centering  
\color{black}{
	\subfigure[Effect of $\mu_0$ and $\theta_0$]{ 
		\label{subfig:theta and mu}
		\includegraphics[height=3cm,width=4cm]{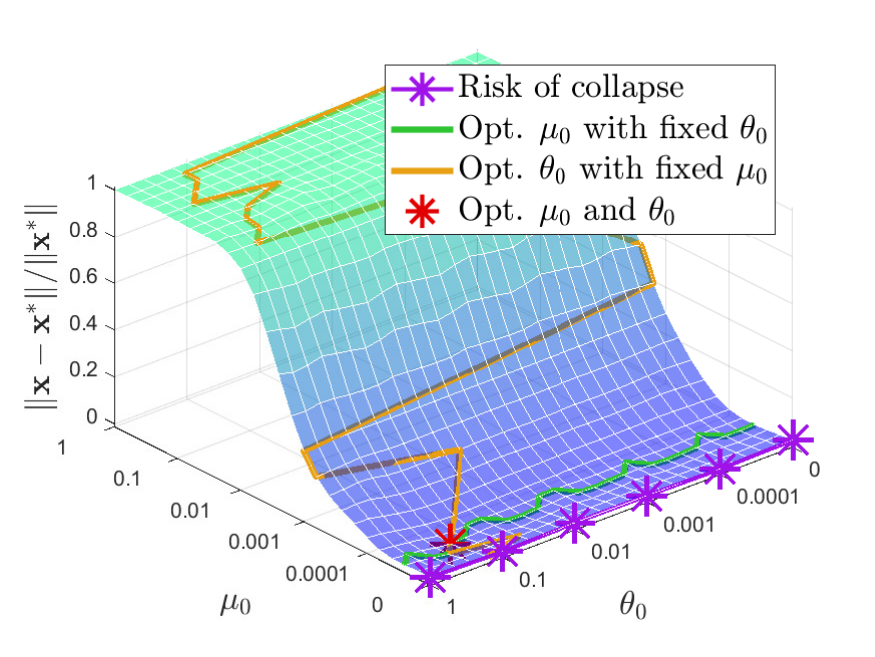} }
	\subfigure[Effect of ($\mu_0,\theta_0$) and $\sigma_0$]{ 		
		\label{subfig:theta mu sigma}
		\includegraphics[height=3cm,width=4cm]{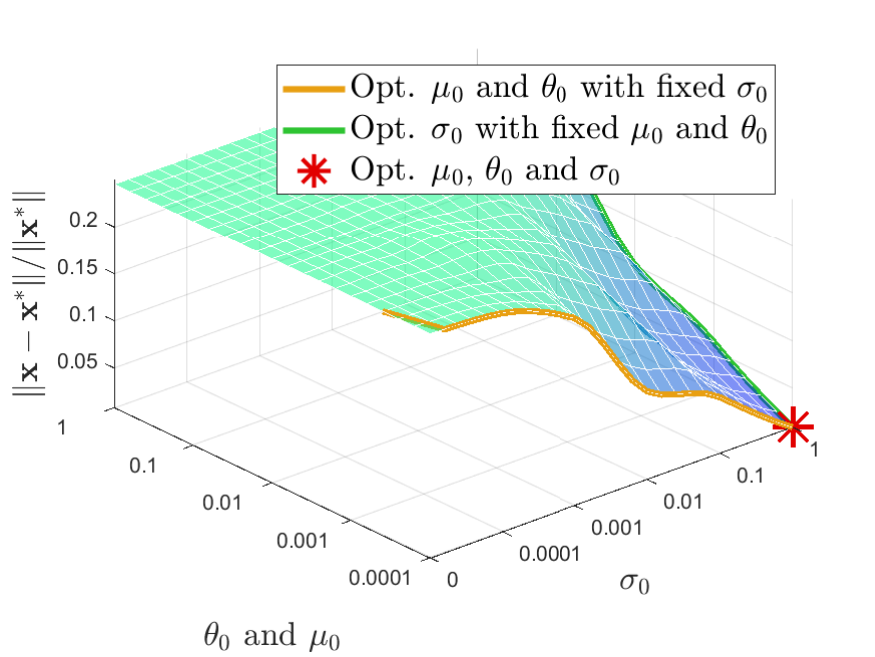} } 
	\caption{
		Effect of $\mu$, $\theta$ (in the regularization term) and $\sigma$ (in the auxiliary function) on the solution errors. 
		We also plot the optimal and easy-to-collapse parameters.
		Subfigure~\subref{subfig:theta and mu} reveals that the effect of $\mu$ is greater than that of $\theta$, and collapses occur when $\mu=0$.		
		Subfigure~\subref{subfig:theta mu sigma} shows that smaller $\mu$ and $\theta$ (but greater than 0) and larger~$\sigma$ would help to converge.
	}
	\label{fig:hyperparameter}
}
\end{figure}

\begin{figure}[tbp]
	\centering  
\color{black}{
	\includegraphics[height=3cm,width=4cm]{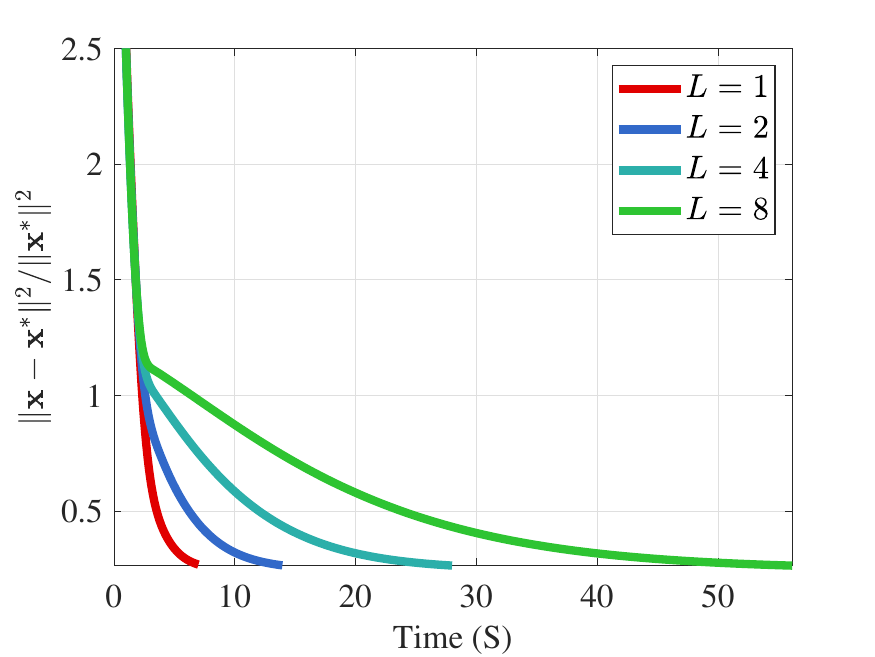} 
	\includegraphics[height=3cm,width=4cm]{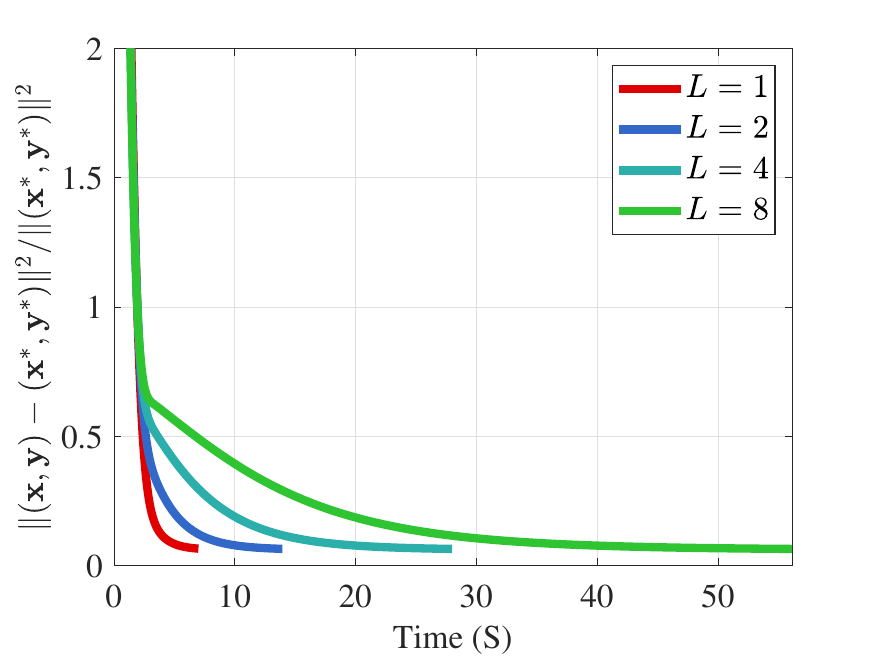} 
	\caption{Effect of $L$ (inner-loop iterations) on the convergence speed.
	}
	\label{fig:L}
}
\end{figure}


\subsubsection{\textcolor{black}{ BLO with Constraints}} \label{sec:toy constraint}

To show the performance of BVFSM for problems with constraints discussed in Section \ref{sec:constraint},
we use the following constrained example with non-convex LL:
$\small
	\begin{aligned}
		& \min _{\x \in \mathbb{R}, \y \in \mathbb{R}^{n}}\|\x-a\|^{2}+\left\|\y-a\right\|^{2} \\ 
		& \text { s.t. } [\y]_{i} \in \mathop{\arg\min}_{[\y]_i \in \mathbb{R}} \Big\{ \sin \left(\x+[\y]_i- [\c]_{i}\right) : \x+[\y]_i \in[0,1] \Big\},  \forall \ i,\\ 
	\end{aligned}
$
where $a \in \mathbb{R}$ and $\c \in \mathbb{R}^{n}$ are any fixed given constant and vector satisfying $ [\c ]_{i} \in[0,1] \text{ for any } i=1, \cdots, n$. The optimal solution is
$
\x^*=\frac{1-n}{1+n} a, 
[\y]_{i}=-\x^*, \forall \ i, 
$
and the optimal value is $F^{*}=\frac{4 n}{1+n} a^{2}$.
\textcolor{black}{Derivation of the closed-form solution is provided in Appendix~\ref{sec:appendix closed-form solution}.}
We conduct experiments under the 2-dimensional case~($n=2$) and set $a=2$ and $[\c]_i=1, \text{ for any }i = 1,2,\cdots,n$. The constraint is carried out via $(\x+[\y]_i-0.5)^2-0.25 \leq 0$, for each component of~$\y$, which is equivalent to  $\x+[\y]_i \in \left[ 0,1\right]$.

\begin{figure}[tbp]
	\centering  
	\includegraphics[height=3cm,width=4cm]{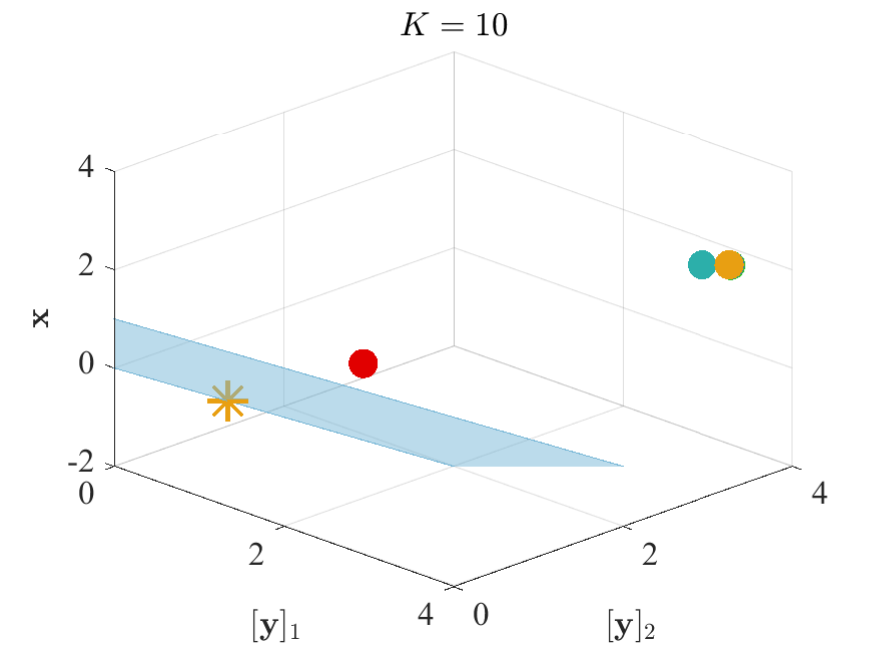} 
	\includegraphics[height=3cm,width=4cm]{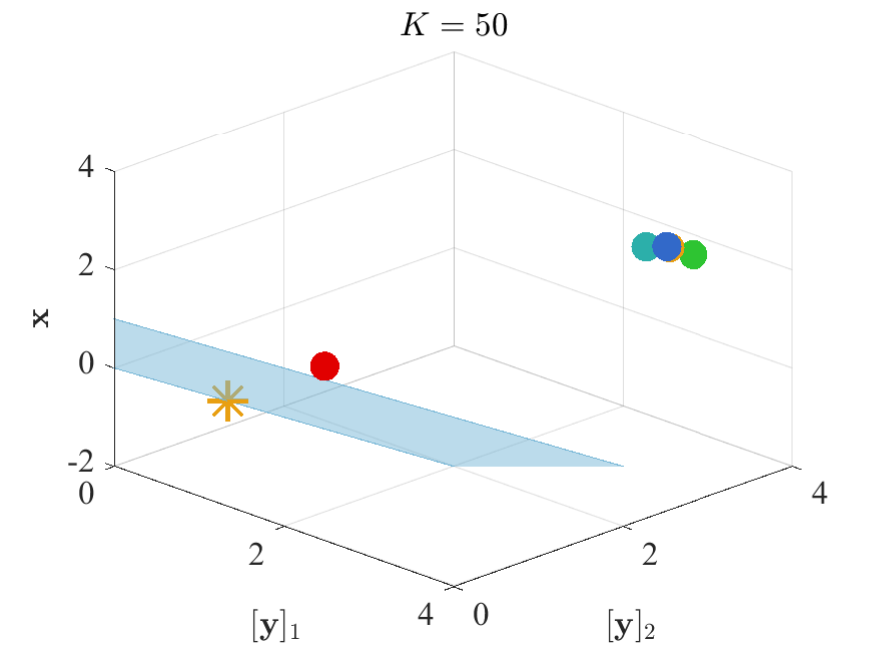}  
	\includegraphics[height=3cm,width=4cm]{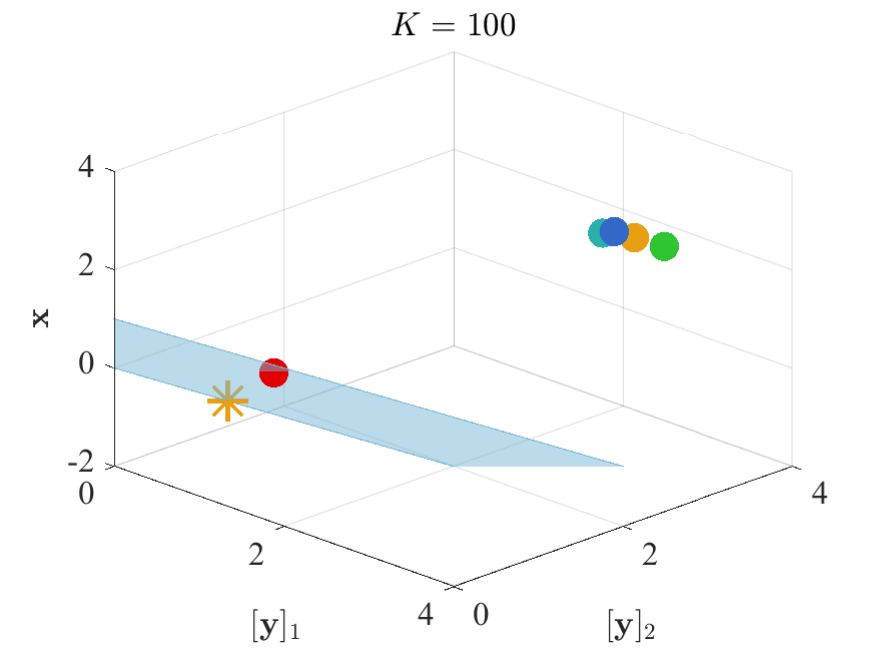}   
	\includegraphics[height=3cm,width=4cm]{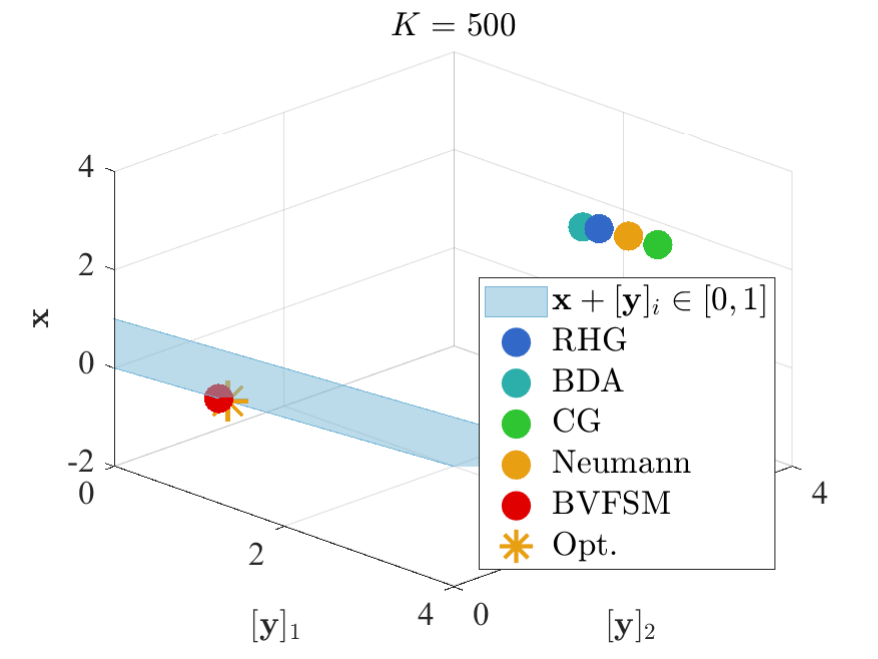}  
	\caption{The relationship between the true optimal solution (“Opt.” for short) and the solutions obtained by different methods during the UL iteration for constrained BLO. 
		It can be seen that BVFSM can gradually converge to the true solution inside the feasible region, 
		while other methods cannot deal with the constraint at all. 
		The legend is only shown in the last plot.
	}
	\label{fig:P}
\end{figure}

\begin{figure} [tbp]
	\centering  
	\subfigure[Barrier]{ 
		\label{subfig:constrain barrier}
		\includegraphics[height=3cm,width=4cm]{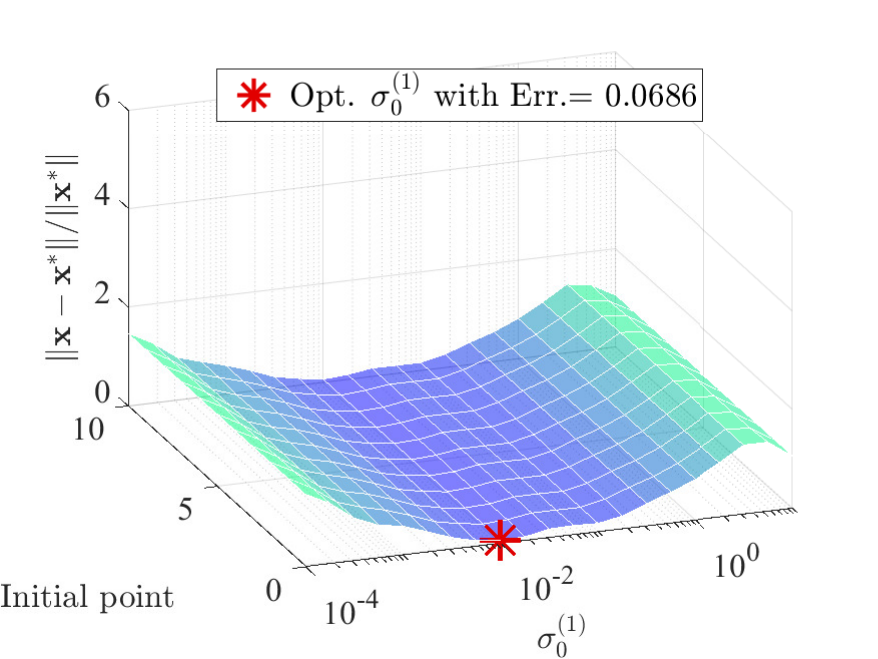} }  
	\subfigure[Penalty]{ 
		\label{subfig:constrain penalty}
		\includegraphics[height=3cm,width=4cm]{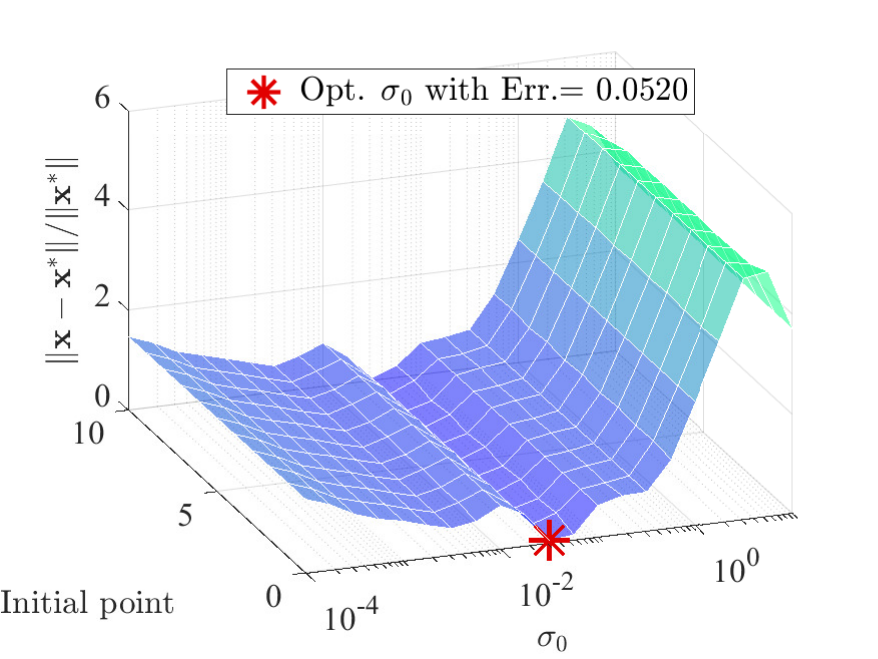} }
	\caption{Convergence results for constrained BLO with different initial points and parameters using~\subref{subfig:constrain barrier} barrier and~\subref{subfig:constrain penalty} penalty functions. 
		Err. in the legend denotes $\| {\x}-{\x^*}\|/ \| \x^*\|$. 
		We select the truncated log barrier and quadratic penalty as the representatives of barrier and penalty functions respectively. 
		Choosing the barrier function leads to higher stability and less sensitivity to parameters, 
		while using the penalty function is more sensitive to parameters
		but can obtain smaller errors. 
		Both methods are insensitive to initial values.
	}
	\label{fig:con}
\end{figure}

Figure~\ref{fig:P} displays the solutions after $K$ iterations. 
It can be seen that when dealing with constrained LL problems, 
only BVFSM can effectively deal with the constraint. 
Hence, our method has broader application space, 
and we will show the experiment in real learning tasks in Section~\ref{sec:ho}, which solves problems with UL constraints.

\begin{figure*} 
	\centering  	
	\begin{minipage}[c]{0.35\textwidth}
		\centering
		\subfigure[Errors]{ \label{subfig:pbo a}
			\includegraphics[height=3.9cm,width=5.2cm]{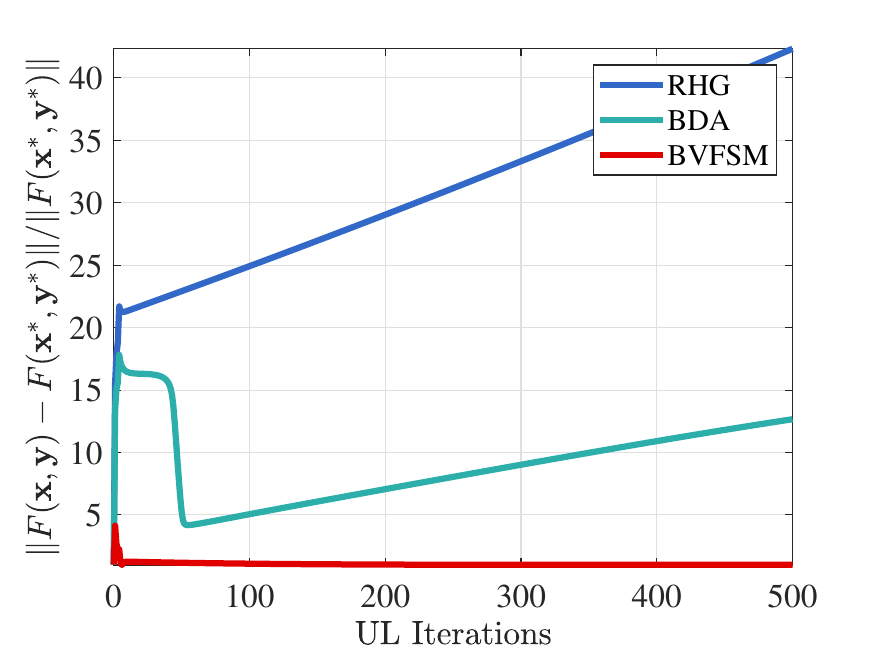} }\\
		\vskip 0.15in
		\qquad\qquad\includegraphics[width=3.5cm]{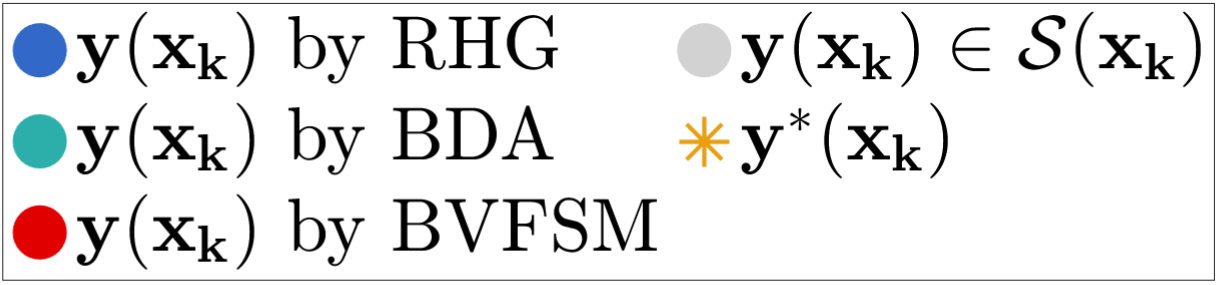}
	\end{minipage}
	\begin{minipage}[c]{0.6\textwidth}
		\subfigure[$k=100$]{\label{subfig:pbo b}
			\includegraphics[height=3cm]{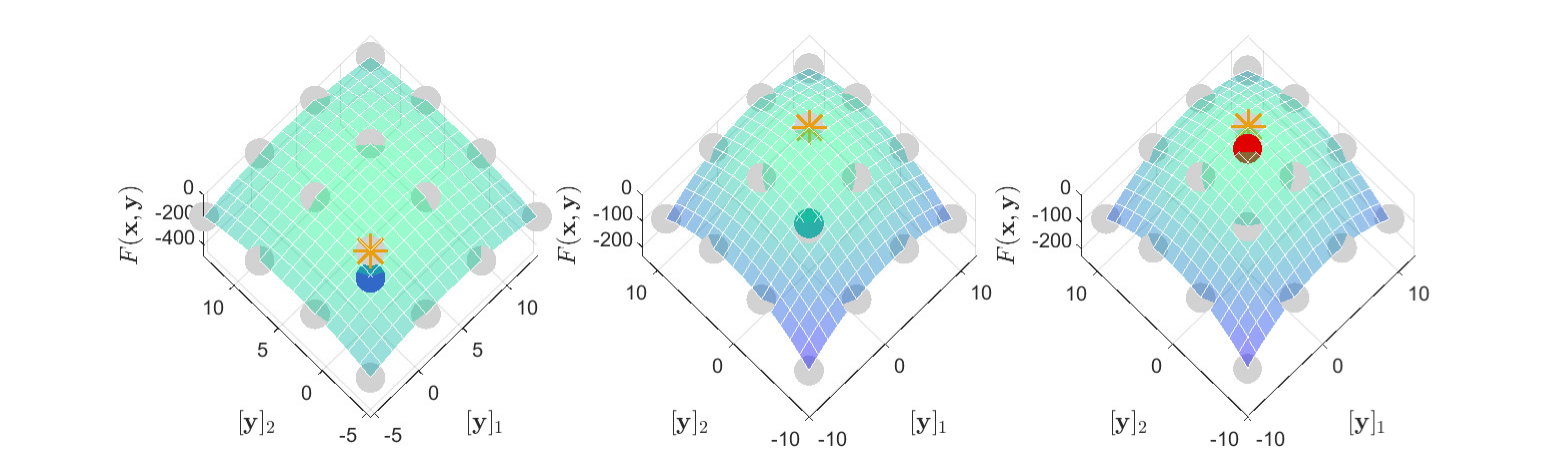} }
		\vskip -0.05in
		\subfigure[$k=500$]{\label{subfig:pbo c}
			\includegraphics[height=3cm]{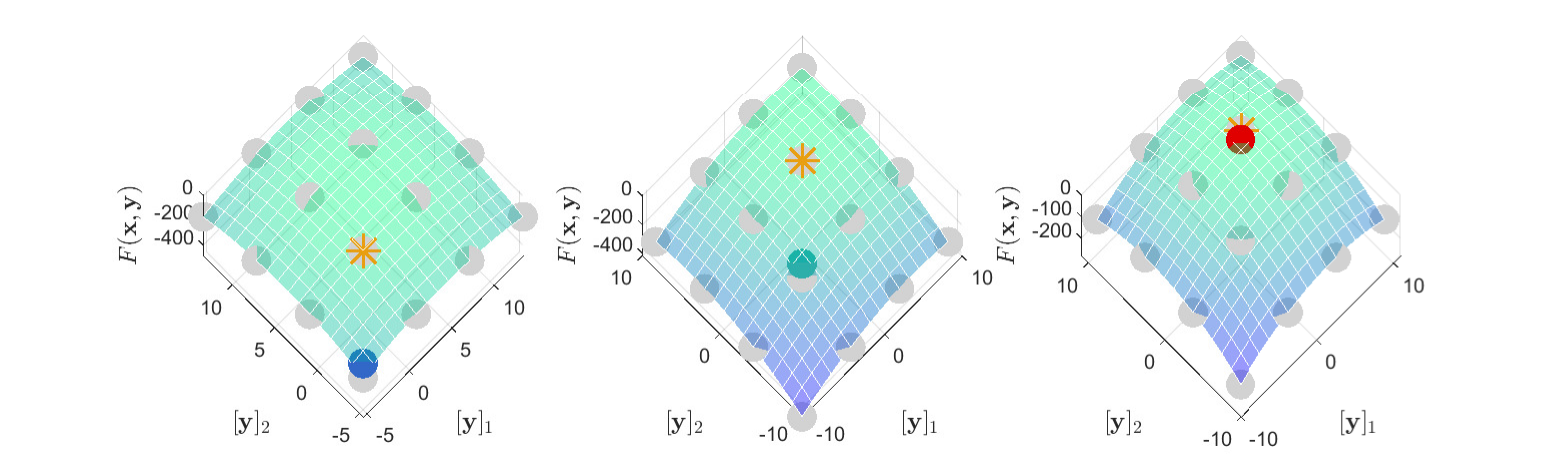} }
	\end{minipage}
	
	\vskip -0.15in
	
	\caption{Convergence behavior for pessimistic BLO. 
		On the right side, we show the LL solution $\y(\x)$ obtained by different methods. 
		Subfigures~\subref{subfig:pbo b} and~\subref{subfig:pbo c} illustrate the results	when $k = 100$ and $k = 500$, respectively. 
		Here the surfaces denote $F(\x_k,\y)$,
		and $\y^*(\x_k)$ means the optimal LL solution among multiple LL solutions $\y(\x_k) \in \S(\x_k)$. 
		Both RHG and BDA choose the incorrect LL solutions, 
		while BVFSM chooses the correct solution.
	}
	\label{fig:pbo}
\end{figure*}

\begin{table*}[htbp]
	\caption{Comparison among existing methods, BVFSM, and BVFSM with constraints (BVFSM-C) for data hyper-cleaning tasks on three datasets: MNIST, FashionMNIST and CIFAR10. 
		F1 score is the harmonic mean of precision and recall.}
	\label{table:hoT1}
	\vspace {-0.5cm}
	\color{black}{
		\begin{center}  \footnotesize
			\begin{tabular}{|c|c|c|c|c|c|c|c|c|c|}
				\hline		
				\multirow{2}{*}{Method}&\multicolumn{3}{c|}{MNIST}&\multicolumn{3}{c|}{FashionMNIST}&\multicolumn{3}{c|}{CIFAR10}\\
				\cline{2-10}
				&Accuracy&F1 score&Time (S)\ &\ Accuracy&F1 score&Time (S)\ &\ Accuracy&F1 score&Time (S)\\
				\hline
				RHG\ &\ 87.90$\pm$0.27&89.36$\pm$0.11&0.4131\ &\ 81.91$\pm$0.18&87.12$\pm$0.19&0.4589\ &\ 34.95$\pm$0.47&68.27$\pm$0.72&1.3374\\
				TRHG\ &\ 88.57$\pm$0.18&89.77$\pm$0.29&0.2623\ &\ 81.85$\pm$0.17&86.76$\pm$0.14&0.2840\ &\ 35.42$\pm$0.49&68.06$\pm$0.55&0.8409\\
				BDA\ &\ 87.15$\pm$0.82&90.38$\pm$0.76&0.6694\ &\ 79.97$\pm$0.71&88.24$\pm$0.58&0.8571\ &\ 36.41$\pm$0.23&67.33$\pm$0.31&1.4869\\
				\hline
				CG\ &\ 89.19$\pm$0.35&85.96$\pm$0.48&0.1799\ &\ 83.15$\pm$0.24&85.13$\pm$0.27&0.2041\ &\ 34.16$\pm$0.75&69.10$\pm$0.93&0.4796\\
				Neumann\ &\ 87.54$\pm$0.13&89.58$\pm$0.34&0.1723\ &\ 81.37$\pm$0.18&87.28$\pm$0.19&0.1958\ &\ 33.45$\pm$0.16&68.87$\pm$0.11&0.4694\\
				\hline
				BVFSM\ &\ {90.41}$\pm$0.32&{91.19}$\pm$0.25&\textbf{{0.1480}}\ &\ \textbf{{84.31}$\pm$0.27}&{88.35$\pm$0.13}&{0.1612}\ &\ \textbf{{38.19}$\pm$0.62}&69.55$\pm$0.42&\textbf{{0.4092}}\\
				BVFSM-C\ &\ \textbf{90.94$\pm$0.32}&\textbf{91.83$\pm$0.30}&0.1566\ &\ 83.23$\pm$0.34&\textbf{89.74$\pm$0.24}&\textbf{0.1514}\ &\ 37.33$\pm$0.33&\textbf{69.73$\pm$0.51}&0.4374\\
				\hline
			\end{tabular}
		\end{center}
		
	}
\end{table*}

To compare the performance of different auxiliary functions,
we try barrier and penalty functions for $\PP_{H,\sigma_H}, \PP_{h,\sigma_h}$ and $\PP_{f,\sigma_f}$,
which can be selected arbitrarily and separately indeed,
but here are chosen the same to be compared more directly. 
Since all of these auxiliary functions can guarantee the convergence theoretically, we mainly focus on the robustness of them under different settings. 
From Figure~\ref{fig:con}, it can be seen that using a penalty function can converge only under certain settings within a small region, while using a barrier function has greater robustness, 
so we use barrier functions in other experiments. 
In Section~\ref{sec:ho}, we further show investigations on penalty and barrier functions on complex networks.


\subsubsection{\textcolor{black}{Pessimistic BLO}} \label{sec:toy pessimistic}

\textcolor{black}{
To study the performance of pessimistic BLO, we use the example similar to optimistic BLO by changing Eq.~\eqref{eq:non-convexExperiment} from $\min_{ \x \in \mathbb{R}, \y \in \mathbb{R}^n}$  to $\min_{ \x \in \mathbb{R}}\max_{\y \in \mathbb{R}^n}$ and from $\| \x-a\|^2+\| \y-a-\c\|^2$ to $\| \x-a\|^2-\| \y-a-\c\|^2$. 
Here we consider the 2-dimensional case (LL dimension $n=2$), and set $a=2$ and $[\c]_i=2 \text{ for }i = 1,2$.
In this case, the optimal solution is $(\x^*,\y^*)=(-2+ \pi/2, 4 \pm \pi, 4 \pm \pi),$ and the optimal value is $F^* = -7/4 \pi^2 - 4 \pi +16.$
}
\textcolor{black}{
Derivation of the exact solution is provided in Appendix~\ref{sec:appendix closed-form solution}.
}
We select RHG and BDA respectively as the representatives of gradient-based methods with or without unique LL solution. 
We make no adaptive modifications to these methods which do not consider the pessimistic BLO situation. 

Figure~\ref{fig:pbo} shows the convergence curves of UL objective and how various methods choose $\y \in \S(\x)$ when $\S(\x)$ is not a singleton. 
From Figure~\ref{subfig:pbo a}, our method has significantly better convergence in the pessimistic case,
while RHG and BDA cannot converge at all. 
\textcolor{black}{
	Their distances to the optimal solution even increase because they fail to select the optimal LL solution $\y$ from multiple LL solutions $\y \in \S(\x)$,
	which is intuitively demonstrated in Figure~\ref{subfig:pbo b} and~\ref{subfig:pbo c}.}


\begin{figure} [tbp]
	\centering  
	\subfigure[UL objective $F({\x},{\y})$]{ \label{subfig6a}
		\includegraphics[height=3cm,width=4cm]{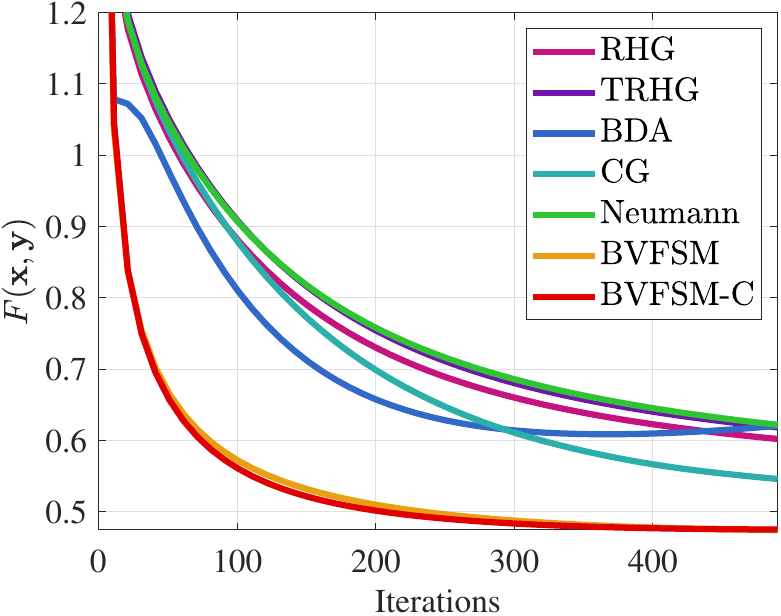}  }
	\subfigure[F1 score]{
		\includegraphics[height=3cm,width=4cm]{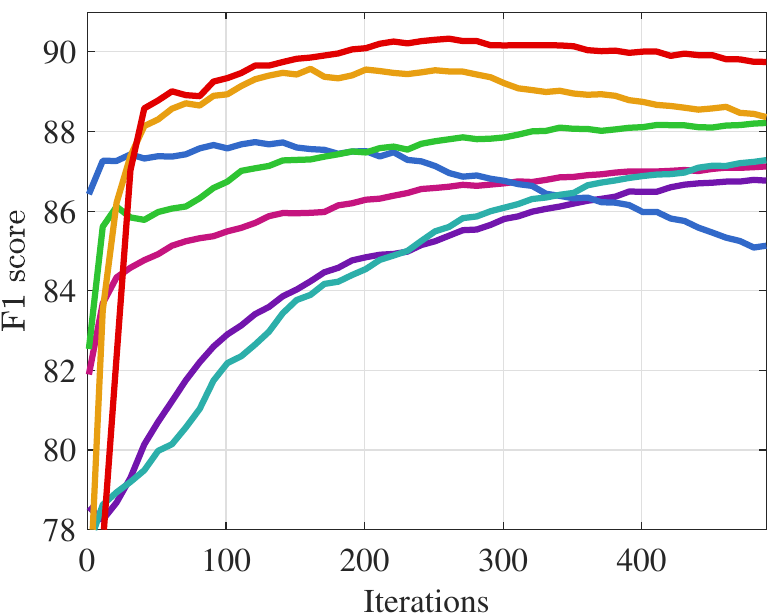} } 
	\caption{ 
		Performance for data hyper-cleaning 
		based on the FashionMNIST experiment in Table~\ref{table:hoT1}. 
		The legend is only shown in the first plot.}  
	\label{fig:ho1} 
\end{figure}

\begin{figure} [tbp]
	\centering  
	\subfigure[Accuracy]{
		\includegraphics[height=3cm,width=4cm]{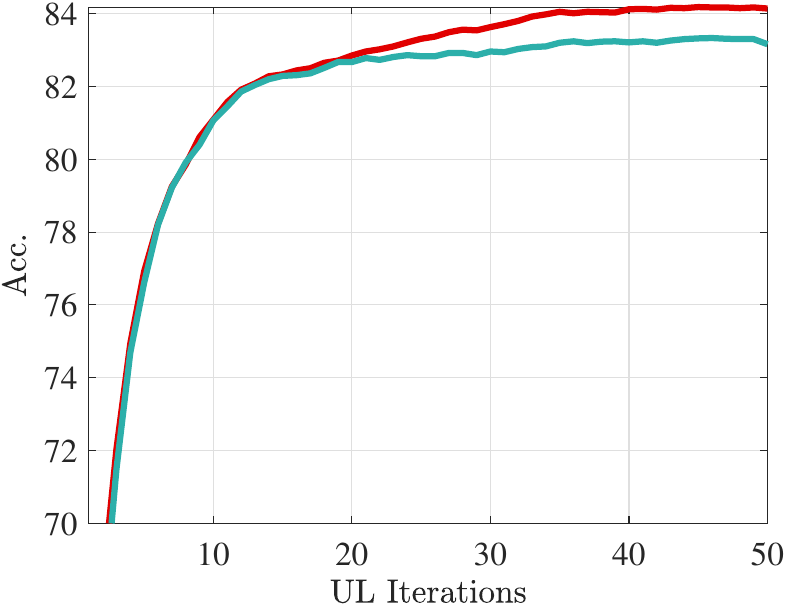}  }
	\subfigure[F1 score]{ \label{subfig7b}
		\includegraphics[height=3cm,width=4cm]{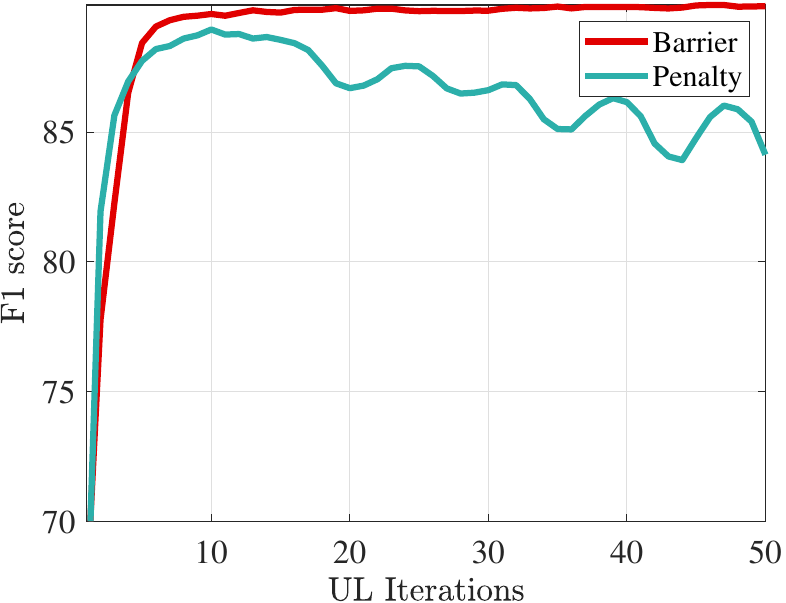}  }
	\caption{
		Performance for data hyper-cleaning with barrier and penalty functions on FashionMNIST.
		We choose the truncated log barrier and quadratic penalty as the representatives of barrier and penalty functions respectively.	
		Using a barrier function leads to
		higher accuracy and greater stability of F1 score. 
		The legend is only shown in the second plot.
	}
	\label{fig:ho2} 
\end{figure}

\begin{table*}[htbp]
	\centering
	\color{black}{
		\caption{The effect of contamination rates for data hyper-cleaning. 
			Accuracy and F1 scores of existing methods drop sharply with the increasing of contamination rate, while BVFSM maintains a slightly decreasing trend, verifying the robustness of BVFSM in the face of harsh data.
		}
\vspace {-0.3cm}
		\begin{tabular}{|c|c|c|c|c|c|c|c|c|}
			\hline
			Contamination rate & \multicolumn{2}{c|}{0.6}     & \multicolumn{2}{c|}{0.7}&\multicolumn{2}{c|}{0.8}     & \multicolumn{2}{c|}{0.9}\\
			\cline{1-9}
			
			Method& Accuracy  & F1 score & Accuracy  & F1 score & Accuracy  & F1 score & Accuracy  & F1 score \\
			\hline
			RHG   & 77.39±0.61 & 68.18±0.94 & 75.62±0.94 & 56.72±0.72 & 68.91±0.71 & 46.81±0.78 & 59.83±0.91 & 29.39±0.38 \\
			TRHG  & 77.37±0.52 & 76.76±0.13 & 75.60±0.84 & 65.30±0.10 & 68.89±0.30 & 55.39±0.97 & 59.81±0.38 & 37.97±0.48 \\
			BDA   & 75.44±0.44 & 78.24±0.34 & 73.67±0.59 & 66.78±0.82 & 66.96±0.69 & 56.87±0.79 & 57.88±0.87 & 39.45±0.08 \\
			\hline
			CG    & 78.64±0.52 & 75.17±0.79 & 76.87±0.06 & 63.71±0.41 & 70.16±0.80 & 53.80±0.09 & 61.08±0.63 & 36.38±0.03 \\
			Neumann & 76.85±0.95 & 77.29±0.29 & 75.08±0.15 & 65.83±0.22 & 68.37±0.40 & 55.92±0.46 & 59.29±0.26 & 38.50±0.63 \\
			\hline
			BVFSM & \textbf{81.49±0.22} & \textbf{85.51±0.70} & \textbf{81.34±0.42} & \textbf{82.55±0.33} & \textbf{80.06±0.97} & \textbf{73.51±0.83} & \textbf{79.73±0.20} & \textbf{55.97±0.73} \\
			\hline
		\end{tabular}%
		\label{tab:pollute rate}%
	}
\vspace{-0.1cm}
\end{table*}%

\begin{figure*} [tbp]
	\centering  
	\subfigure[Time]{\label{subfig:TimeLayer}
		\includegraphics[height=3cm,width=7.5cm,trim=0 50 0 50]{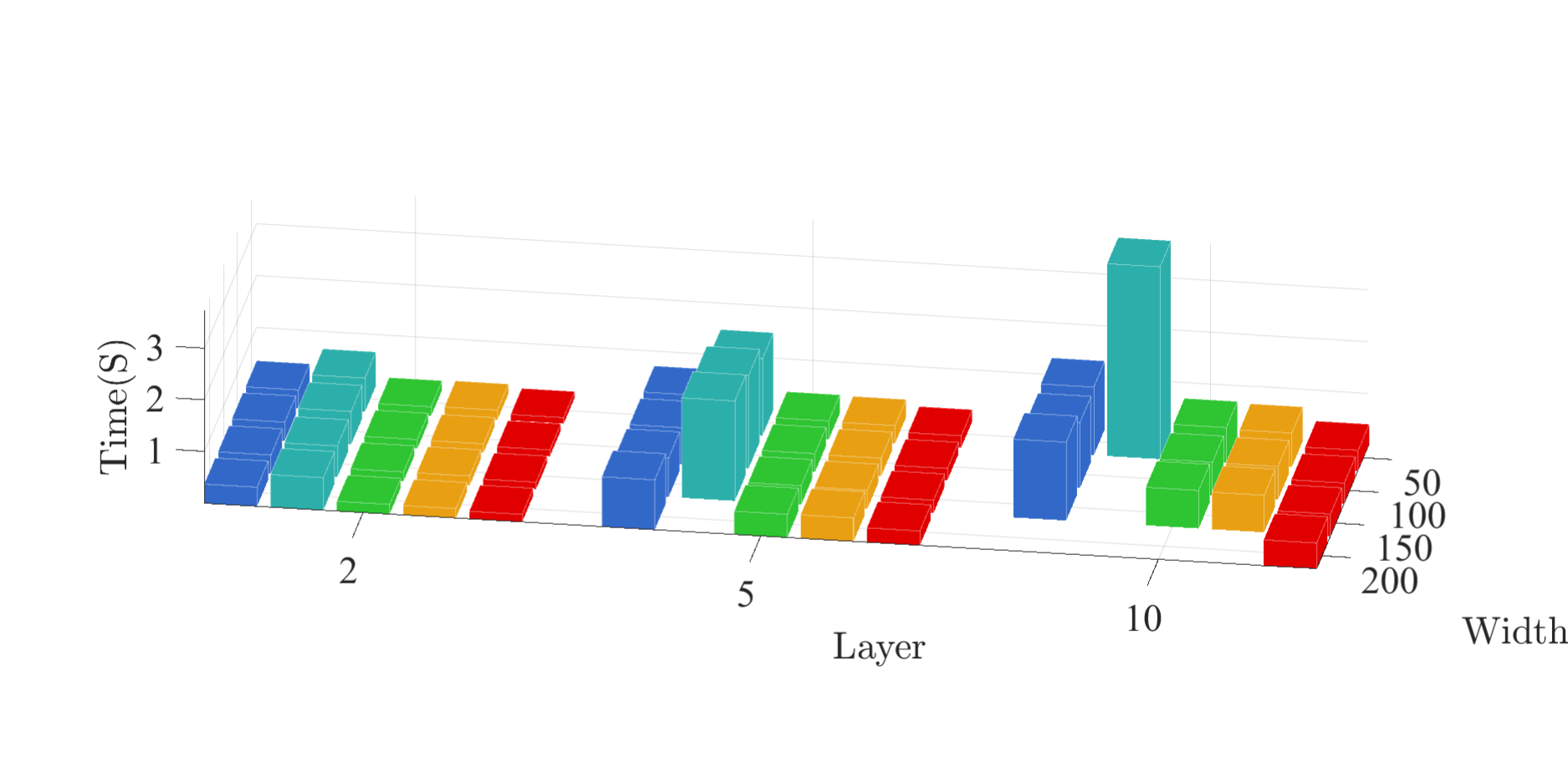}  }
	\subfigure[VRAM]{ \label{subfig:VRAMLayer}
		\includegraphics[height=3cm,width=7.5cm,trim=0 50 0 50]{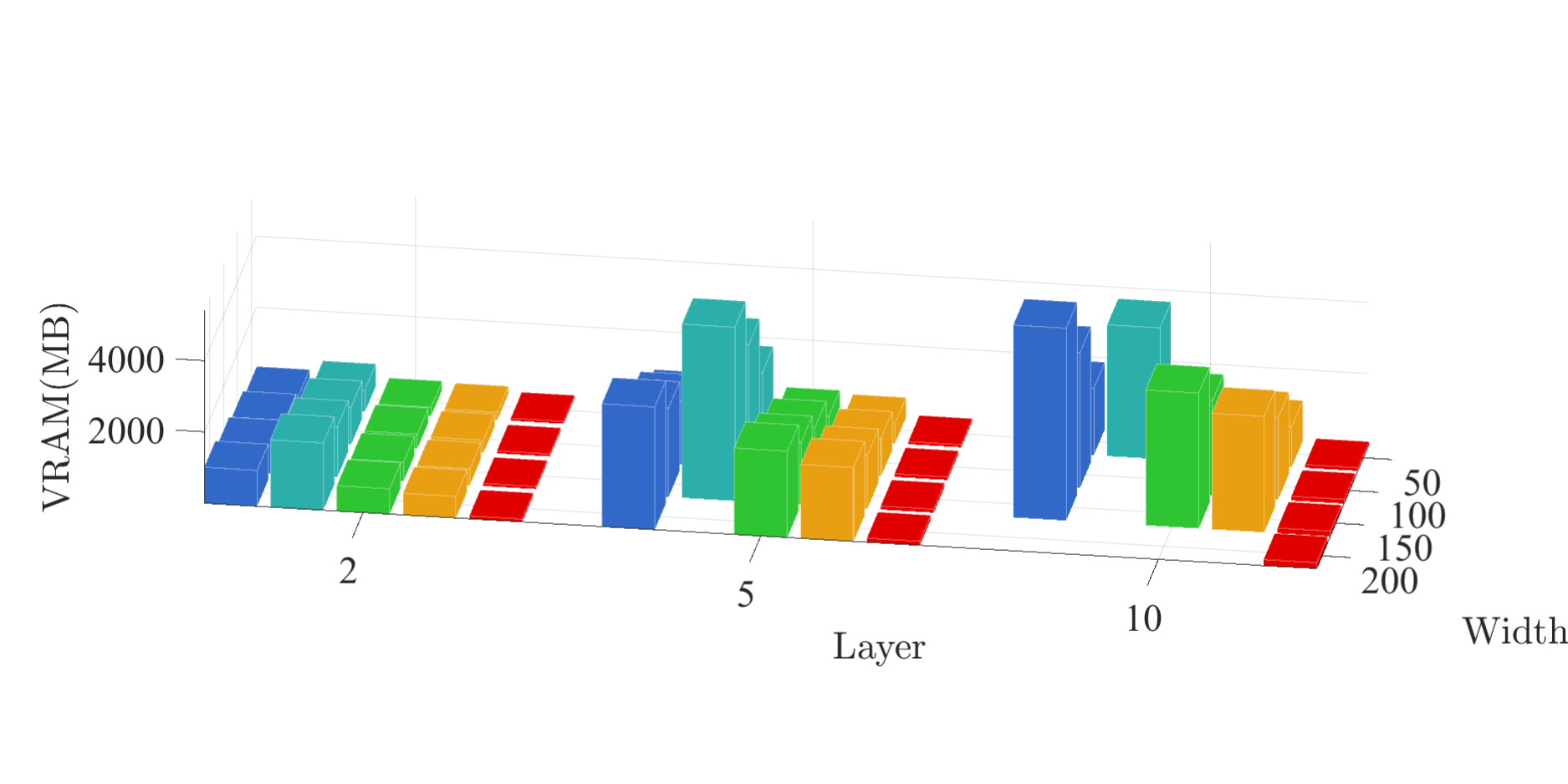}  }
	\includegraphics[width=1cm,trim=0 -50 0 -50]{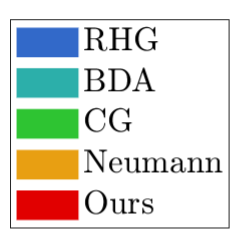}
\vspace{-0.4cm}
	\caption{\textcolor{black}{Computation time (S) and memory (VRAM, MB) with various network structures for data hyper-cleaning. 
			Blank areas not drawn indicate that the 8G VRAM limit is exceeded.
			We use fully connected networks of 2, 5, 10 layers and with widths of 50, 100, 150 and 200.
		}
	}
	\label{fig:TimeLayers}  
\end{figure*}

\subsection{Hyper-parameter Optimization}\label{sec:ho}
In this subsection, we use a specific task of hyper-parameter optimization, called data hyper-cleaning, to evaluate the performance of BVFSM when the LL problem is non-convex.
Assuming that some of the labels in our dataset are contaminated, 
the goal of data hyper-cleaning is to reduce the impact of incorrect samples by adding hyper-parameters to them. 
In this experiment, we set $\y \in \mathbb{R}^{10\times 301}\times\mathbb{R}^{300\times d}$ as the parameter of a non-convex 2-layer linear network classifier where $d$ is the dimension of data, 
and $\x\in \mathbb{R}^{|\D_{tr}|}$ as the weight of each sample in the training set. 
Therefore, the LL problem is to learn a classifier $\y$ by cross-entropy
loss $g$ weighted with given $\x$:
\begin{equation*}
	f(\x,\y)=\sum_{(\u_{i},\v_{i})\in \D_{\mathtt{tr}}}[\mathtt{sigmoid}(\x)]_i  \  g(\y,\u_{i},\v_{i}),
\end{equation*}
where $(\u_{i},\v_{i})$ are the training samples, and $\mathtt{sigmoid}(\x)$ is the sigmoid function to constrain the weights $\x$ into the range of $[0, 1]$.
The UL problem is to find a weight $\x$ to reduce the cross-entropy loss $g$ of $\y$ on a cleanly labeled validation set: 
\begin{equation*}
	F(\x,\y)=\sum_{(\u_{i},\v_{i})\in \D_{\mathtt{val}}}g(\y,\u_{i},\v_{i}).
\end{equation*}
In addition, we also consider adding explicit constraints directly on $\x$ (as discussed in Section \ref{sec:constraint})
instead of using the sigmoid function as indirect constraints.
The constraint is carried out via $([\x]_i-0.5)^2-0.25 \leq 0$, for each component of $\x$, such that $[\x]_i \in \left[ 0,1\right]$.



\textcolor{black}{\textbf{Overall performance.}}
Table~\ref{table:hoT1} shows the accuracy, F1 score and computation time 
on three different datasets.
For each dataset, we randomly select 5000 samples as the training set $\D_{\mathtt{tr}}$, 
5000 samples as the validation set  $\D_{\mathtt{val}}$, 
and 10000 samples as the test set $\D_{\mathtt{test}}$. 
After that, we contaminate half of the labels in $\D_{\mathtt{tr}}$. 
From the result, BVFSM achieves the most competitive performance on all datasets.
Furthermore, BVFSM is faster than EGBMs and IGBMs, 
and this advantage is more evident on CIFAR10 with larger LL dimension, consistent with the complexity analysis in Section~\ref{sec:complexity}. 
The UL objective value and F1 score during iterations 
on FashionMNIST are also plotted in Figure~\ref{fig:ho1}.

As for the performance of BLO with constraints, we can find from Table~\ref{table:hoT1} that BVFSM with constraints (denoted as BVFSM-C in the table) 
has slightly lower accuracy but higher F1 score than BVFSM using sigmoid function without explicit constraints. 
\textcolor{black}{
This is because for BVFSM without constraints, the compound of sigmoid function in the LL objective decreases the gradient of $\x$, 
and thus the UL variable $\x$ with small change rate contributes to its slower convergence.
Accuracy more reflects the convergence of LL variable $\y$, while F1 score more reflects the convergence of UL variable $\x$.
Therefore, BVFSM with constraints performs slightly worse in accuracy but better in F1 score than BVFSM without constraint but with the sigmoid function.
}

\textcolor{black}{\textbf{Evaluations on the auxiliary functions, robustness, and network structures.}}
Figure~\ref{fig:ho2} compares the performance of different auxiliary functions.
Consistent with the numerical experiment in Figure~\ref{fig:con}, the barrier function works better with higher stability
without the need for too much fine tuning of parameters.
\textcolor{black}{
	Table~\ref{tab:pollute rate} compares the robustness under various data contamination rates. 
}
Figure~\ref{fig:TimeLayers} further shows 
the impact of network structures in depth and width. 
For the LL variable $\y$ 
we use fully connected networks of various layers and widths.
It is worth noting that the computational burden is overall not quite sensitive to the network width, but very sensitive to the network depth. 
With the deepening of networks, other methods experience varying degrees of collapse due to occupying too much memory, while BVFSM can always keep the computation stable.
Since there is no need to retain the LL iteration trajectory, 
our storage burden is much less than that of EGBMs (RHG and BDA). 
Thanks to the fact that BVFSM does not need to calculate the Jacobian- and Hessian-vector products (realized by saving an additional calculation graph in AD), 
our burden is also significantly lower than that of IGBMs (CG and Neumann).


\begin{table}[tbp]
	\begin{center}  
		\footnotesize
		\caption{Computation time (S) in each epoch for data hyper-cleaning in VGG series networks with different convolution layers (Conv.), batch sizes (B) and iteration number (K). 
			N\,/\,A means exceeding the memory limit.
			Note that a smaller batch size may take more time because the batch switching time increases.
			BVFSM maintains the least burden and highest speed, especially with 
			large-scale LL in real-world networks.} 
		\label{table:VGG}
		\vspace {-0.3cm}
		\begin{tabular}{|c|c|c|c|c|c|}
			\hline  
			Conv.&\multirow{1}{*}{(B,K)}&\multirow{1}{*}{RHG}&\multirow{1}{*}{CG}&\multirow{1}{*}{Neumann}&\multirow{1}{*}{BVFSM}\\
			\hline
			2&(1,7)  & 7515 &  4730 & 3225 &  \textbf{2252} \\
			2&(128,20) &N/A & N/A  & 415.4 & \textbf{60.81}  \\
			13&(128,100)   &  N/A   & N/A  & 472.9 & \textbf{171.9} \\
			13&(512,100)   &  N/A   & N/A   & N/A   &\textbf{121.8}  \\
			\hline
		\end{tabular}
	\end{center}
\end{table}

\begin{figure*} [tbp]
	\centering  
	\subfigure{ 
		\includegraphics[height=3cm,width=7.5cm]{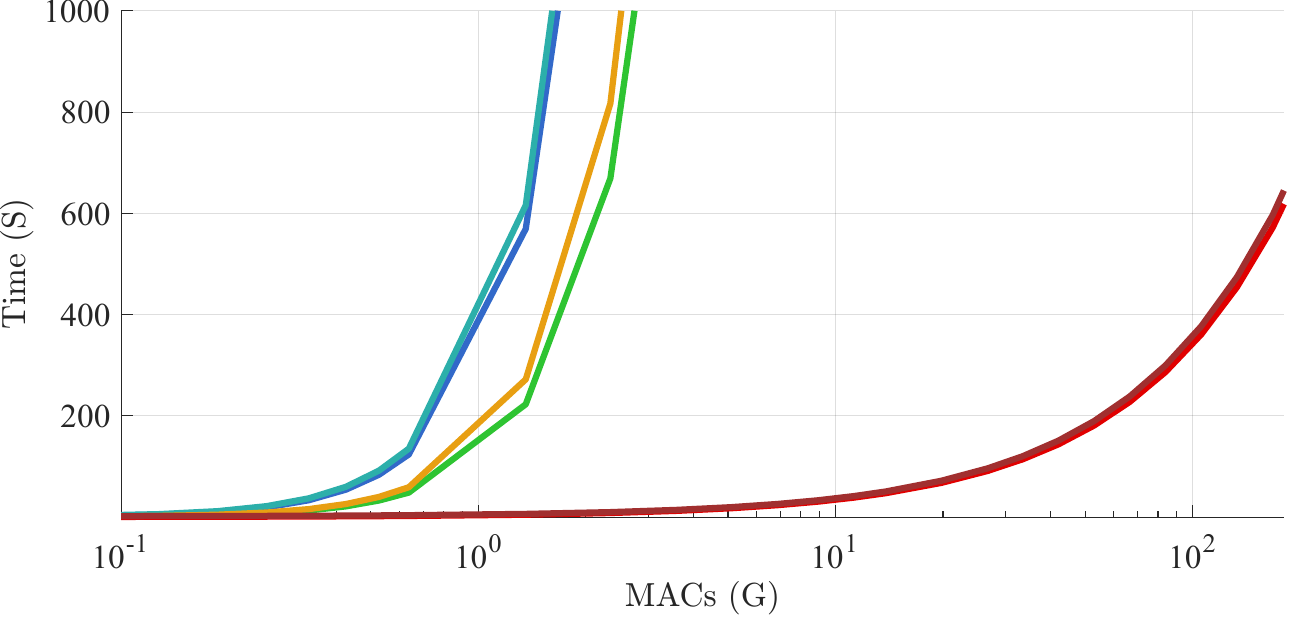}  }
	\subfigure{ 
		\includegraphics[height=3cm,width=7.5cm]{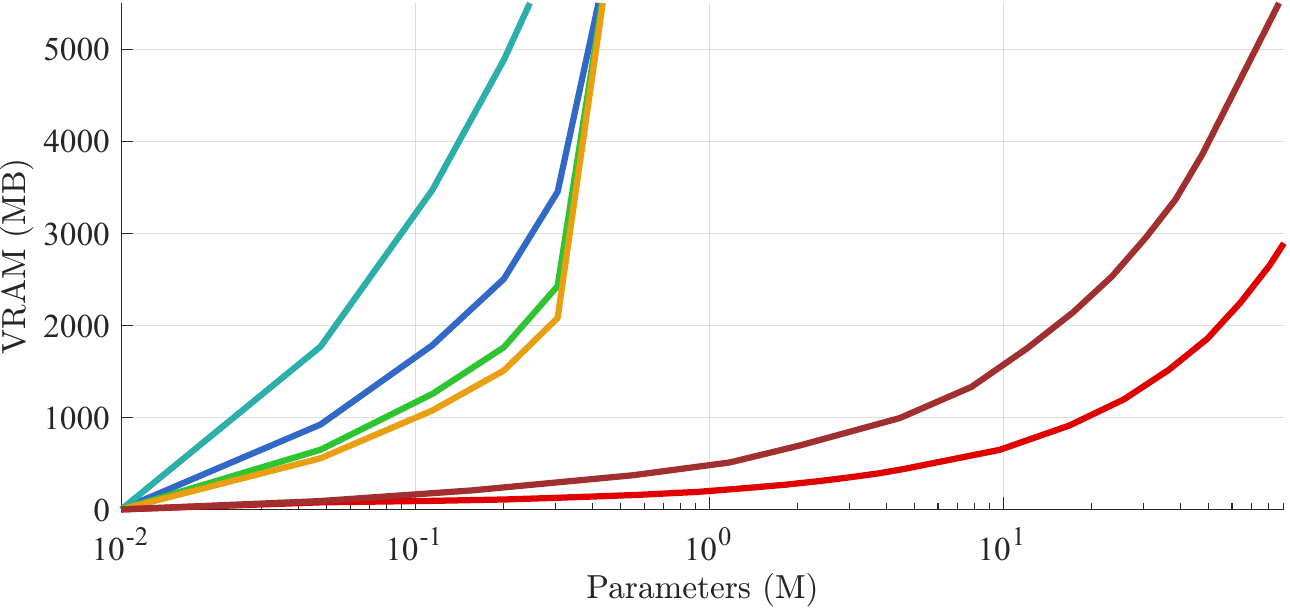}  }	
	\includegraphics[width=1cm,trim=0 -50 0 -50]{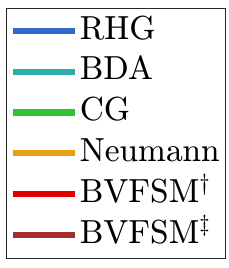}
\vspace {-0.3cm}
	\caption{\textcolor{black}{	
			Effects of the quantity of Multiply–Accumulate Operations (MACs) and Parameters on the computational efficiency for data hyper-cleaning in large-scale networks.					
			BVFSM$^\dag$ and BVFSM$^\ddag$ denote BVFSM in the 5-layer network same to other methods and the more challenging 50-layer network, respectively. 
			For a clearer comparison, we use logarithmic coordinates.
		} 
	}
	\label{fig:MACs}  
\end{figure*}

\textcolor{black}{\textbf{Computational efficiency for large-scale networks.} }
Next, we verify our computational burden on large-scale networks closer to real applications such as VGG16 on CIFAR10 dataset. 
Because VGG16 has too much computational burden on existing methods, in order to make the comparison available, we change the experimental settings as follows.
For each dataset, we randomly select 4096 samples as the training set $\D_{\mathtt{tr}}$, 
4096 samples as the validation set  $\D_{\mathtt{val}}$, 
and 512 samples as the test set $\D_{\mathtt{test}}$. 
Because the original network is too computationally intensive for EGBMs, we perform an additional experiment on some sufficiently small batch size and iteration number $K$.
We also simplify the convolution layers from 13 layers in VGG16 to only the first two layers, and retain the last 3 linear layers.
As shown in Table~\ref{table:VGG}, BVFSM always has the highest speed under various settings, 
and still works well with a large $K$ and batch size.

\textcolor{black}{
	Additionally, we visualize how BVFSM can be applied to large-scale networks by expanding the width of 5-layer network,
	and compare the computational efficiency when the Multiply–Accumulate Operations (MACs) and parameters are increased. 
	For the same-size network, the fully-connected layer typically has more parameters, while the convolutional layer has more MACs,
	so we use the fully-connected and convolutional layer respectively to simulate the scale-up of parameters and MACs.
	From Figure~\ref{fig:MACs},
	other methods are computationally inefficient and can only handle small-scale networks,
	while BVFSM with much higher efficiency is applicable to larger-scale networks in frontier tasks.
	Moreover, considering the effect of number of layers on efficiency as shown in Figure~\ref{fig:TimeLayers},
	we also use a more challenging 50-layer network for BVFSM to further demonstrate its high efficiency.
	Specifically, existing methods usually cannot work under MobileNet with around 1 GMACs, while BVFSM is available under StyleGAN with around 100 GMACs.
%
%
}


\begin{table*}[htbp]
	\begin{center} 
		\footnotesize
		\caption{Averaged accuracy using various methods (including model-based methods and gradient-based BLO methods) for the few-shot classification task
		(1- and 5-shot, i.e., $M = 1, 5,$ and $N = 5, 20, 30, 40$) on Omniglot.} 
		\label{table:meta}
		\vspace {-0.3cm}
	\color{black}{
		\begin{tabular}{|c|c|c|c|c|c|c|c|c|}
			\hline
			\multirow{2}{*} { Method } & \multicolumn{2}{c|} { 5-way } & \multicolumn{2}{c|} {20-way } & \multicolumn{2}{c|} {30-way } & \multicolumn{2}{c|} {40-way } \\
			\cline{2-9}
			\ &\  1-shot & 5-shot \ &\  1-shot & 5-shot \ &\  1-shot & 5-shot \ &\  1-shot & 5-shot \\
			\hline
			MAML\ &\ 98.70$\pm$0.40&\textbf{99.91$\pm$0.10}\ &\ 95.80$\pm$0.30&98.90$\pm$0.20\ &\ 86.86$\pm$0.49&96.86$\pm$0.19\ &\ 85.98$\pm$0.45&94.46$\pm$0.43\\
			Meta-SGD\ &\ 97.97$\pm$0.70&98.96$\pm$0.20\ &\ 93.98$\pm$0.43&98.42$\pm$0.11\ &\ 89.91$\pm$0.04&96.21$\pm$0.15\ &\ 87.39$\pm$0.43&95.10$\pm$0.15\\
			Reptile\ &\ 97.68$\pm$0.04&99.48$\pm$0.06\ &\ 89.43$\pm$0.14&97.12$\pm$0.32\ &\ 85.40$\pm$0.30&95.28$\pm$0.30\ &\ 82.50$\pm$0.30&92.79$\pm$0.33\\
			iMAML\ &\ \textbf{99.16$\pm$0.35}&99.67$\pm$0.12\ &\ 94.46$\pm$0.42&98.69$\pm$0.10\ &\ 89.52$\pm$0.20&96.51$\pm$0.08\ &\ 87.28$\pm$0.21&95.27$\pm$0.08\\
			\hline
			RHG\ &\ 98.64$\pm$0.21&99.58$\pm$0.12\ &\ 96.13$\pm$0.20&99.09$\pm$0.08\ &\ 93.92$\pm$0.18&98.43$\pm$0.08\ &\ 90.78$\pm$0.20&96.79$\pm$0.10\\
			TRHG\ &\ 98.74$\pm$0.21&99.71$\pm$0.07\ &\ 95.82$\pm$0.20&98.95$\pm$0.07\ &\ 94.02$\pm$0.18&98.39$\pm$0.07\ &\ 90.73$\pm$0.20&96.79$\pm$0.10\\
			BDA\ &\ 99.04$\pm$0.18&99.74$\pm$0.05\ &\ 96.50$\pm$0.16&\textbf{99.19$\pm$0.07}\ &\ 94.37$\pm$0.18&98.53$\pm$0.07\ &\ 92.49$\pm$0.18&97.12$\pm$0.09\\
			BVFSM\ &\ 98.85$\pm$0.12&99.21$\pm$0.18\ &\ \textbf{96.73$\pm$0.30}&98.95$\pm$0.20\ &\ \textbf{94.65$\pm$0.20}&\textbf{98.56$\pm$0.17}\ &\ \textbf{92.73$\pm$0.12}&\textbf{97.61$\pm$0.47}\\
			\hline
		\end{tabular}
	}
	\end{center}
\end{table*}

\begin{figure*}[tbp]
	\centering  
	 \makeatletter
	 \renewcommand{\@thesubfigure}{\hskip\subfiglabelskip}
	 \makeatother
	\subfigure[Initialization]{
		\includegraphics[height=1.9cm,width=2.2cm]{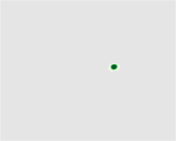}}
		 \makeatletter
	\renewcommand{\@thesubfigure}{ \thesubfigure\space}
	\makeatother
	\addtocounter{subfigure}{-1}
	\subfigure[GAN]{\label{subfig:gan}
		\includegraphics[height=2.2cm,width=2.2cm]{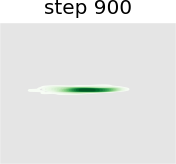} 
		\includegraphics[height=2.2cm,width=2.2cm]{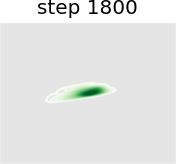} 
		\includegraphics[height=2.2cm,width=2.2cm]{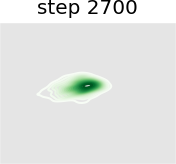}}
	\subfigure[WGAN]{\label{subfig:wgan}
		\includegraphics[height=2.2cm,width=2.2cm]{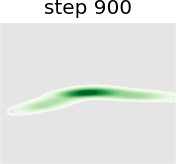} 
		\includegraphics[height=2.2cm,width=2.2cm]{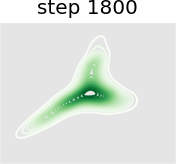} 
		\includegraphics[height=2.2cm,width=2.2cm]{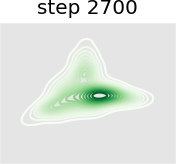}
		}
		 \makeatletter
	\renewcommand{\@thesubfigure}{\hskip\subfiglabelskip}
	\makeatother
	\subfigure[Target]{
	    \includegraphics[height=1.9cm,width=2.2cm]{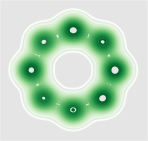}}
    		 \makeatletter
    	\renewcommand{\@thesubfigure}{\thesubfigure\space}
    \makeatother
    	\addtocounter{subfigure}{-1}
	\subfigure[Unrolled GAN]{\label{subfig:unrolled gan}
		\includegraphics[height=2.2cm,width=2.2cm]{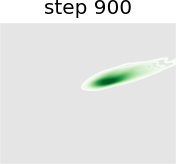} 
		\includegraphics[height=2.2cm,width=2.2cm]{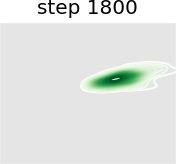} 
		\includegraphics[height=2.2cm,width=2.2cm]{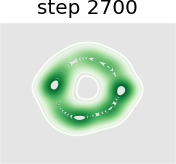}}
	\subfigure[BVFSM]{\label{subfig:BVFSM}
		\includegraphics[height=2.2cm,width=2.2cm]{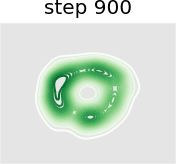} 
		\includegraphics[height=2.2cm,width=2.2cm]{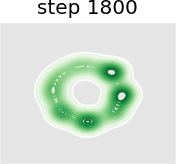} 
		\includegraphics[height=2.2cm,width=2.2cm]{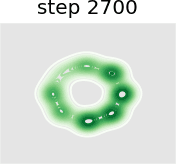}}
	\caption{Comparison of GAN training on a toy 2D mixture of Gaussians dataset. 
	We show the heat map of the generated distribution as the number of training steps increases. 
	Subfigure~\subref{subfig:gan} indicates vanilla GAN can capture only one distribution, 
	while in subfigure~\subref{subfig:wgan}, WGAN attempts to capture all distributions at the same time with one Gaussian distribution, but fails to achieve satisfactory performance.
	Unrolled GAN in subfigure~\subref{subfig:unrolled gan} can approximate all distributions simultaneously with the help of a leader-follower structure, but lacks further details.
	In contrast, our BVFSM fits all Gaussian distributions well with plenty of details, as shown in subfigure~\subref{subfig:BVFSM}.
    }
	\label{fig:7}
\end{figure*}

\subsection{Few-shot Learning}\label{sec:meta}

We then conduct experiments 
on the few-shot learning task. 
Few-shot learning is one of the most popular applications in meta-learning,
whose goal is to learn an algorithm that can also handle new tasks well. 
Specifically, each task is an $N$-way classification and it aims to learn the hyper-parameter $\x$ 
so that each task can be solved by only $M$ training samples (i.e., $N$-way $M$-shot).
Similar to works in~\cite{franceschi2018bilevel,liu2020generic,liu2022general}, 
we model the network with two parts: 
a four-layer convolution network $\x$ as a common feature extraction layer among tasks, 
and a logical regression layer $\y={\y^i}$ as the separated classifier for each task. 
We also set dataset as $\D=\{\D^i\}$, 
where $\D^i=\D^i_{\mathtt{tr}}\cup \D^i_{\mathtt{val}}$ for the $i$-th task.
By setting the loss function of the $i$-th task to be cross-entropy $g(\x,\y^i;\D^i_{\mathtt{tr}})$ for the LL problem, 
the LL objective can be defined as
\begin{equation*}
	f(\x,\y)=\sum_i g(\x,\y^i;\D^i_{\mathtt{tr}}).
\end{equation*}
As for the UL objective, we also utilize the cross-entropy function but define it based on $\{\D^i_{\mathtt{val}}\}$ as
\begin{equation*}
	F(\x,\y)=\sum_i g(\x,\y^i;\D^i_{\mathtt{val}}).
\end{equation*}
Our experiment is performed on the widely used benchmark dataset: Omniglot~\cite{lake2015human}, 
which contains examples of 1623 handwritten characters from 50 alphabets.

We compare our BVFSM with several approaches, such as MAML, Meta-SGD, Reptile, iMAML, RHG, TRHG and BDA~\cite{liu2020generic,liu2022general}.
From Table~\ref{table:meta},
BVFSM achieves slightly poorer performance than existing methods in the 5-way task,
because when dealing with small-scale LLC problems, 
the strength of regularization term by BVFSM to accelerate the convergence
cannot fully counteract its impact on the offset of solution.
However, for larger-scale LL problems (such as 20-way, 30-way and 40-way), thanks to the regularization term,
BVFSM reveals significant advantages over other methods.

\subsection{Generative Adversarial Networks}\label{sec:gan}

Next we perform intuitive experiments on GAN 
to illustrate the application of BVFSM for pessimistic BLO. 
GAN is a network used for unsupervised machine learning to build a min-max game between two players, i.e., the generator $\mathtt{Gen}(\x;\cdot)$ with the network parameter $\x$, 
and the discriminator $\mathtt{Dis}(\y;\cdot)$ with the network parameter $\y$.
We denote the standard Gaussian distribution as $\mathcal{N}(0,1)$ and the real data distribution as $p_{\mathtt{data}}$. 
The generator $\mathtt{Gen}$ tries to fool the discriminator $\mathtt{Dis}$ by producing data from random latent vector $\v \sim \mathcal{N}(0,1)$, 
while the discriminator $\mathtt{Dis}$ distinguishes between real data $\u \sim p_{\mathtt{data}}$ and generated data $\mathtt{Gen}(\x;\v)$ by outputting the probability that the samples are real.
The goal of GAN is to
$ 
	\textcolor{black}{\min_{\x}\max_{\y} }   \log(\mathtt{Dis}(\y;\u))+\log(1-\mathtt{Dis}(\y;\mathtt{Gen}(\x;\v))) 
$~\cite{goodfellow2014generative}.

However, this traditional modeling method regards $\mathtt{Dis}$ and $\mathtt{Gen}$ as equal status, and does not characterize the leader-follower relationship that $\mathtt{Gen}$ first generates data and after that $\mathtt{Dis}$ judges the data, which can be modeled by Stackelberg game and captured through BLO problems.
Specifically, from this perspective,
generative adversarial learning corresponds to a pessimistic BLO problem:
the UL objective $F$ of $\mathtt{Gen}$ tries to
generate adversarial samples, 
and the LL objective $f$ of $\mathtt{Dis}$ aims to learn a robust classifier which can maximize the UL objective.
Therefore, we reformulate GAN into the form in Eq.~\eqref{eq:PBO} discussed in Section~\ref{sec:pessimistic} to model this relationship, and call it bi-level GAN. 
Concretely, for the follower $\mathtt{Dis}(\y;\cdot)$, 
the LL objective is consistent with the original GAN:
\begin{equation*}
    f(\x,\y)=\log(\mathtt{Dis}(\y;\u))+\log(1-\mathtt{Dis}(\y;\mathtt{Gen}(\x;\v))).
\end{equation*}
As for UL,
considering the antagonistic goals of $\mathtt{Gen}$ and $\mathtt{Dis}$,
we model the UL problem as 
%
\begin{eqnarray}\nonumber
	F(\x,\y)=\log(\mathtt{Dis}(\y;\mathtt{Gen}(\x;\v))).
\end{eqnarray}

Note that the popular WGAN~\cite{arjovsky2017wasserstein} is a variation of the most classic vanilla GAN~\cite{goodfellow2014generative} (or simply GAN),
while unrolled GAN~\cite{metz2016unrolled} and the GAN generated by our BVFSM belong to bi-level GAN,
modeling from a BLO perspective. 
Our method has the following two advantages over other types of GAN.
On the one hand, compared with vanilla GAN and WGAN,
bi-level GAN can effectively model the leader-follower relationship between the generator and discriminator, rather than regard them as the same status.
On the other hand, in bi-level GAN, our method considers the situation that the objective has multiple solutions, from the viewpoint of pessimistic BLO, with theoretical convergence guarantee, which unrolled GAN cannot achieve.


In this experiment we train a simple GAN architecture on a 2D mixture of 8 Gaussians arranged on a circle. 
The dataset is sampled from a mixture of 8 Gaussians with standard deviation 0.02. 
The 8 points are the means of data
and are equally spaced around a circle with radius 2. 
The generator consists of a fully-connected network with 2 hidden layers of size 128 with ReLU activation followed by a linear projection to 2 dimensions. 
The discriminator first scales its input down by a factor of 4 (to roughly scale it to $(-1,1)$),
and is followed by a 1-layer fully-connected network from ReLU activation to a linear layer of size 1 to act as the logit.
As shown in Figure~\ref{fig:7}, we present a visual comparison of sample generation among GAN, WGAN, unrolled GAN, and our method.
It can be seen that vanilla GAN can capture only one distribution 
rather than all Gaussian distributions at a time, 
because it ignores the leader-follower structure. 
WGAN benefits from the improvement of distance function and uses one distribution to approximate all Gaussian distributions at the same time,
but it fails to display satisfying performance. 
Unrolled GAN shows the ability to capture all distributions simultaneously thanks to the leader-follower modeling by BLO, but it lacks further details of the distribution. 
However, the desirable treatment of non-convex problems by BVFSM brings about its ability to fit all distributions well with details.
In addition, we show the KL divergence between the generated and target image in Table~\ref{table:KL}. 
It can be seen that the traditional alternately optimized GAN and WGAN yield larger KL divergence, 
while unrolled GAN and our method, which consider GAN as a BLO model, 
produce smaller KL divergence, 
and our method further achieves the best result.

\textcolor{black}{
Figure~\ref{fig:ada} further validates the performance of BVFSM to be adaptive in large-scale GAN on real datasets. 
Specifically, we add BVFSM as a training strategy based on StyleGAN2~\cite{karras2019style} on the AFHQ dataset. 
It can be seen that our approach is effective in improving the generation quality and performance metrics Inception Score (IS) and Frechet Inception Distance (FID).
}

\begin{table}[tbp]
	\begin{center}  \footnotesize
	\caption{KL divergence by various GAN. 
	The divergence between the generated and target distribution by BVFSM is the smallest.} 
	\label{table:KL}
		\vspace {-0.3cm}
	\begin{tabular}{|c|c|c|c|c|}
			\hline
			&GAN&WGAN&Unrolled GAN&BVFSM\\
			\hline
			KL  Divergence&2.56&2.48&0.26&\textbf{0.15}\\
			\hline
		\end{tabular}
	\end{center}
\end{table}

\begin{figure} [tbp]
	\centering  
\color{black}{
	\subfigure[
		\begin{tabular}{c}
			StyleGAN2 \\ 
			(IS, FID)=(11.69, 5.93)
		\end{tabular}
		]{ 
			\includegraphics[height=4cm,width=4cm]{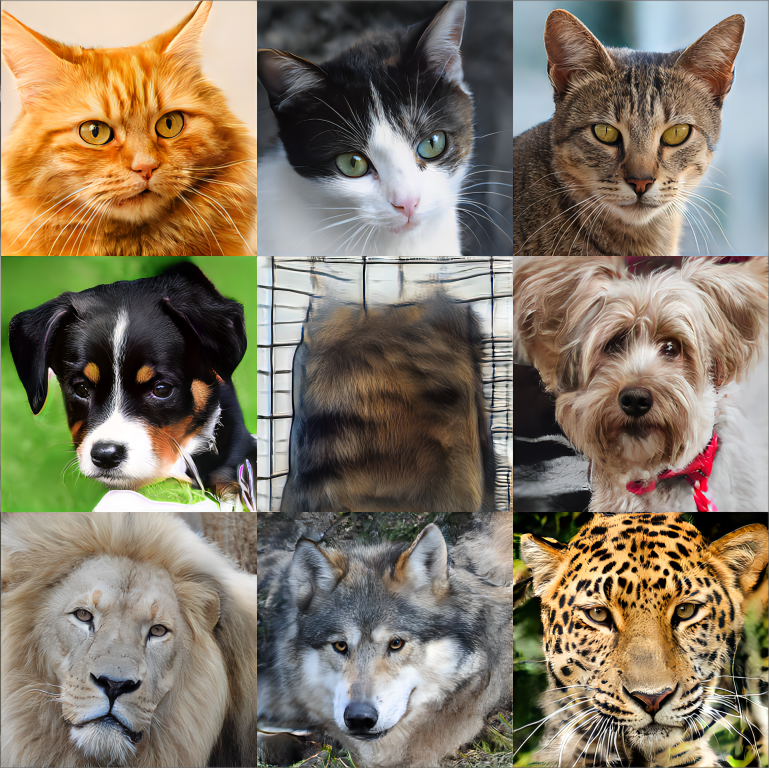}   }
	\subfigure[
		\begin{tabular}{c}
			StyleGAN2+BVFSM \\ 
			(IS, FID)=(11.94, 6.08)
		\end{tabular}
		]{ 
			\includegraphics[height=4cm,width=4cm]{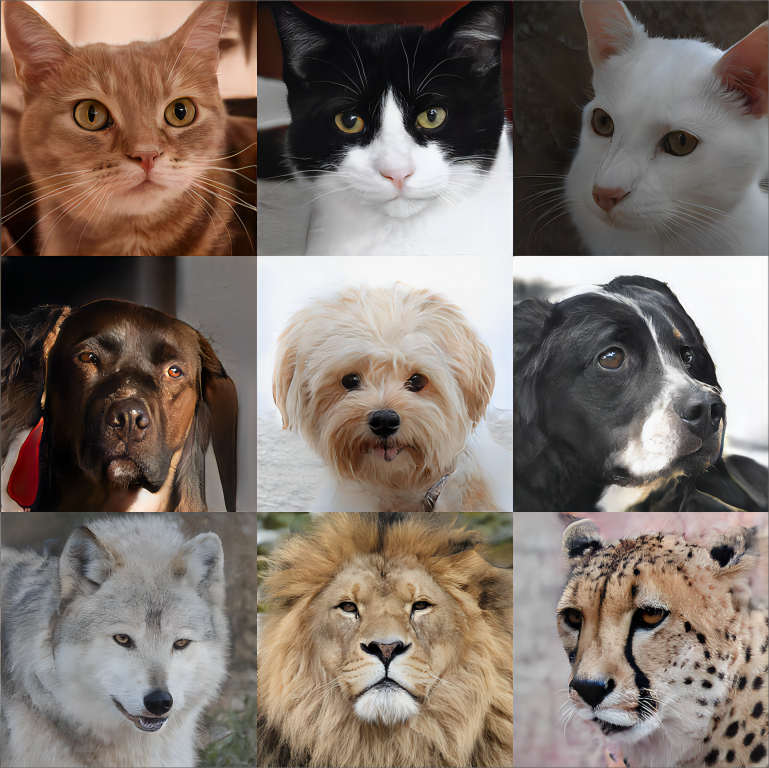}  }
\vspace {-0.3cm}	
	\caption{Visualization results and performance metrics by StyleGAN2 and by adding BVFSM to StyleGAN2 on the AFHQ dataset.}  
	\label{fig:ada} 
}
\end{figure}

\section{Conclusions}
In this paper, we propose a novel bi-level algorithm BVFSM 
\textcolor{black}{to provide an accessible path for large-scale problems with high dimensions from complex real-world tasks.
With the help of value-function which breaks the traditional mindset in gradient-based methods, BVFSM can} 
remove the LLC condition required by earlier works, and improve the efficiency of gradient-based methods,
to overcome the bottleneck caused by high-dimensional non-convex LL problems.
By transforming the regularized LL problem into UL objective by the value-function-based sequential minimization method,
we obtain a sequence of single-level unconstrained differentiable problems to approximate the original problem.
We prove the \textcolor{black}{asymptotic} convergence without LLC, 
and present our numerical superiority through complexity analysis and numerical evaluations for a variety of applications.
We also extend our method to BLO problems with constraints, and pessimistic BLO problems.

\ifCLASSOPTIONcompsoc
  \section*{Acknowledgments}
\else
  \section*{Acknowledgment}
\fi


This work is partially supported by the National Natural Science Foundation of China (Nos. U22B2052, 61922019, 12222106), 
the National Key R\&D Program of China (2020YFB1313503, 2022YFA1004101), 
Shenzhen Science and Technology Program (No. RCYX20200714114700072), the Guangdong Basic and Applied Basic Research Foundation (No. 2022B1515020082),
and Pacific Institute for the Mathematical Sciences (PIMS).

\ifCLASSOPTIONcaptionsoff
  \newpage
\fi



%
\bibliographystyle{IEEEtran}
\bibliography{output}

%

\begin{IEEEbiography}[{\includegraphics[width=1in,height=1.25in,clip,keepaspectratio]{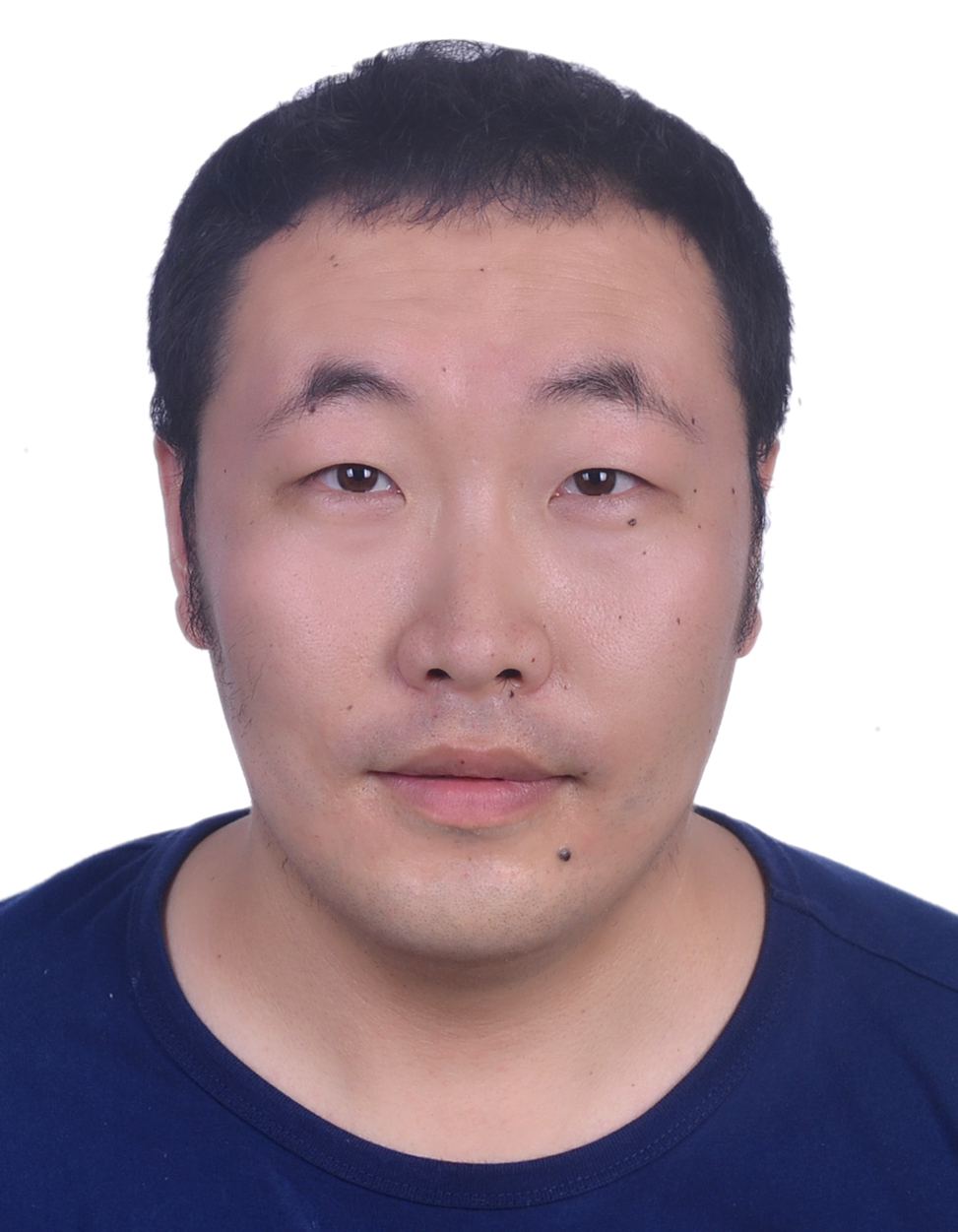}}]{Risheng Liu}
	 received the B.Sc.
	 and Ph.D. degrees in mathematics from Dalian
	 University of Technology in 2007 and 2012,
	 respectively. He was a Visiting Scholar with the
	 Robotics Institute, Carnegie Mellon University,
	 from 2010 to 2012. He served as a Hong Kong
	 Scholar Research Fellow at the Hong Kong Polytechnic University from 2016 to 2017. He is currently a Professor with the DUT-RU International
	 School of Information Science \& Engineering,
	 Dalian University of Technology. His research interests include machine learning, optimization, computer vision, and multimedia.
	 He is a member of the ACM, and was a co-recipient of the IEEE ICME Best
	 Student Paper Award in 2014 and 2015. Two papers were also selected as a
	 Finalist of the Best Paper Award in ICME~2017.
\end{IEEEbiography}

\begin{IEEEbiography}[{\includegraphics[width=1in,clip,keepaspectratio]{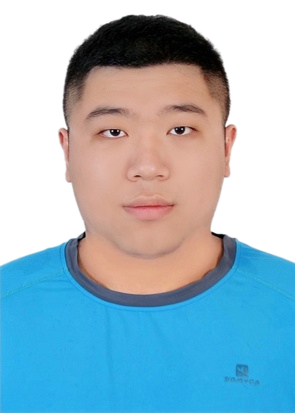}}]{Xuan Liu}
	received the B.Sc. degree in  mathematics from Dalian University of Technology in 2020.  He is currently an M.Phil. student in the Department of Software Engineering at Dalian University of Technology. His research interests include computer vision, machine learning, and control and optimization.
\end{IEEEbiography}

\begin{IEEEbiography}[{\includegraphics[width=1in,height=1.25in,clip,keepaspectratio]{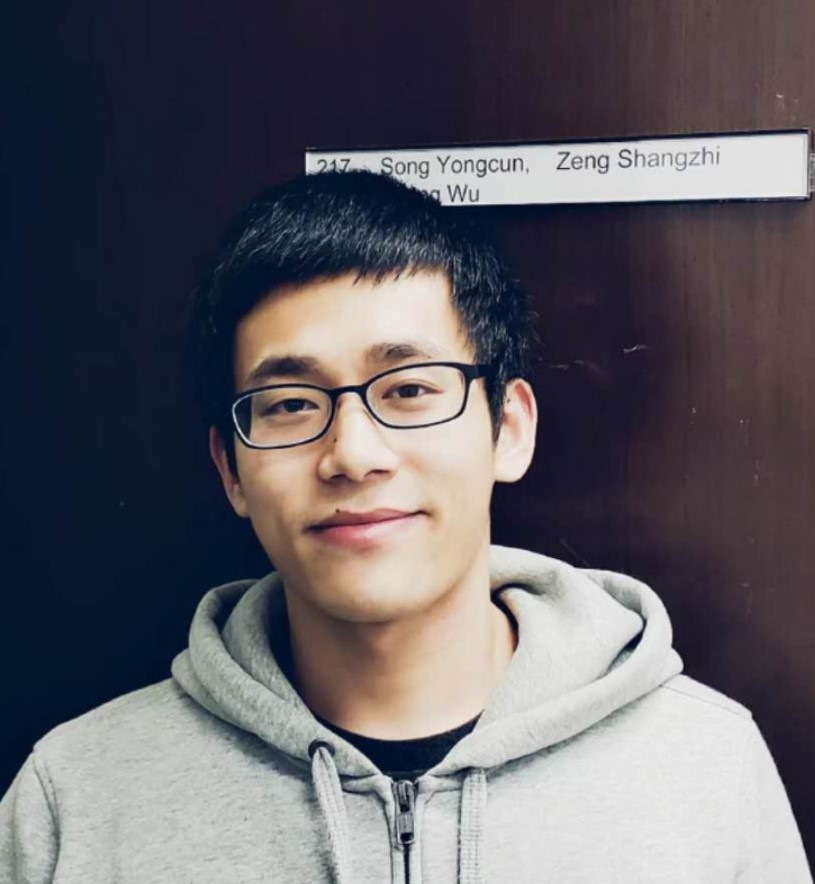}}]{Shangzhi Zeng}
	received the B.Sc. degree in
	Mathematics and Applied Mathematics from
	Wuhan University in 2015, the M.Phil. degree
	from Hong Kong Baptist University in 2017, and
	the Ph.D. degree from the University of Hong
	Kong in 2021. He is currently a PIMS postdoctoral
	fellow in the Department of Mathematics and
	Statistics at University of Victoria. His current
	research interests include variational analysis
	and bilevel optimization.
\end{IEEEbiography}

\begin{IEEEbiography}[{\includegraphics[width=1in,height=1.32in,clip,keepaspectratio]{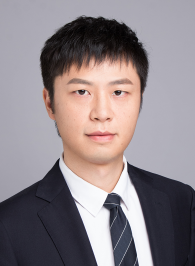}}]{Jin Zhang}
	received the B.A. degree in journalism
	and the M.Phil. degree in mathematics and operational
	research and cybernetics from Dalian University of Technology in 2007 and 2010, respectively,
	and the Ph.D. degree in applied mathematics from
	University of Victoria, Canada, in 2015. After
	working with Hong Kong Baptist University for
	three years, he joined Southern University of
	Science and Technology as a tenure-track Assistant
	Professor with the Department of Mathematics
	and promoted to an Associate Professor in 2022. His
	broad research area is comprised of optimization,
	variational analysis and their applications in economics, engineering, and
	data science. 
\end{IEEEbiography}

\begin{IEEEbiography}[{\includegraphics[width=1in,clip,keepaspectratio]{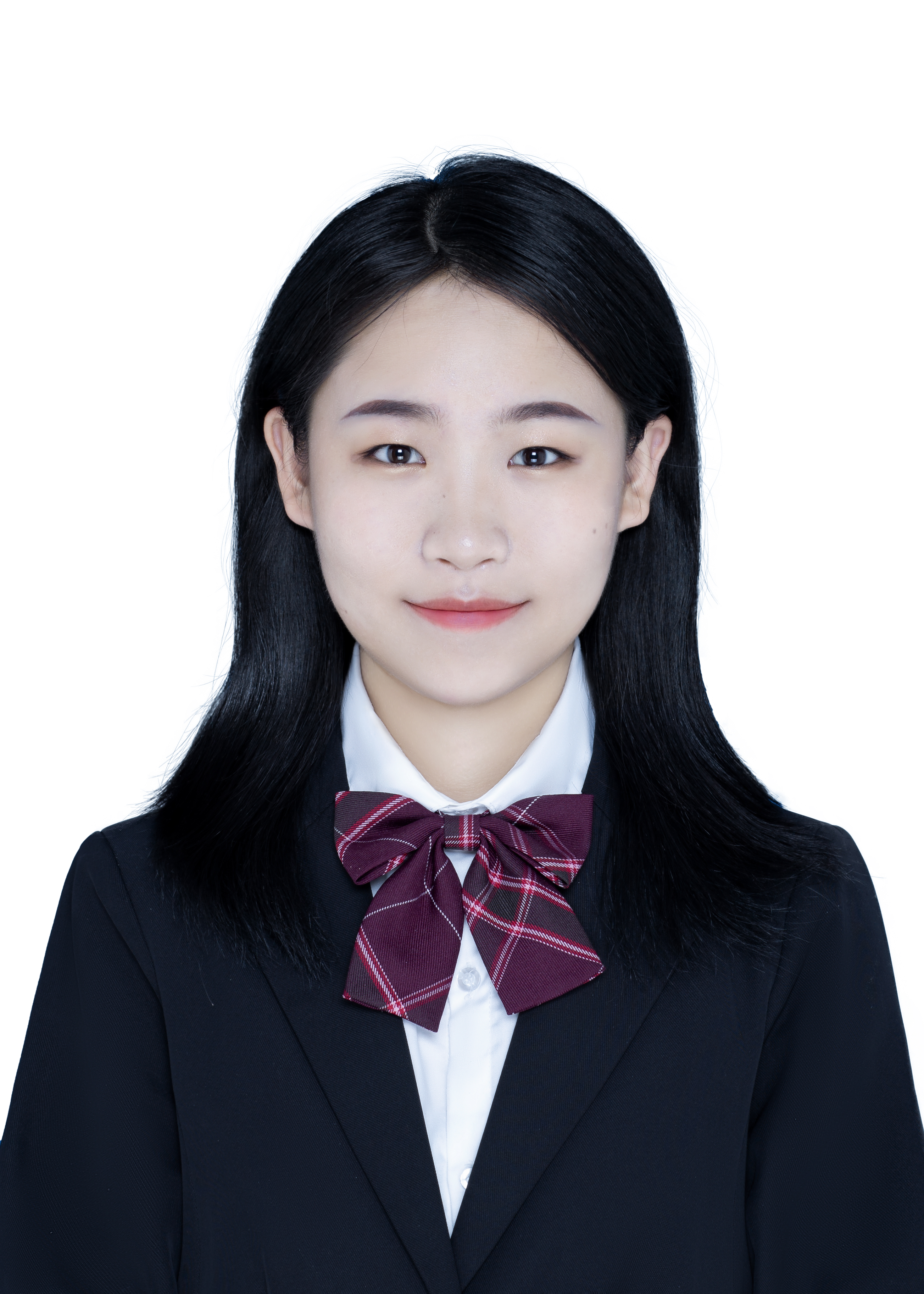}}]{Yixuan Zhang}
	received the B.Sc. degree in Mathematics and Applied Mathematics from Beijing Normal University in 2020,
	and the M.Phil. degree from Southern University of Science and Technology in 2022.
	She is currently a Ph.D. student in the Department of Applied Mathematics at the Hong Kong
	Polytechnic University.
	Her current research interests include optimization and machine learning.
	
\end{IEEEbiography}







\newpage

	
	\appendices

	\section{Proofs of Lemmas in Section~\ref{sec:theoretical investigations}} \label{sec:appendix lemmas}
	
	\subsection{Lemma 1}
		\textit{\textcolor{black}{Let $\{\sigma_k\}$ in $\PP_{k}(\omega) \! = \! \PP_{\sigma_k}(\omega)$ be a positive sequence such that $\lim_{k \rightarrow \infty}\sigma_k = 0$.}
		Additionally assume that $\lim_{k \rightarrow \infty}\rho( - \eta_k \; ; \sigma_k^{(1)}) = 0$ when $\rho$ is a modified barrier function.
		Then we have 
		\begin{enumerate}
			\item $\PP_{k}(\omega)$ is continuous, differentiable and non-decreasing, and satisfies $\PP_{k}(\omega) \geq 0$. 
			\item For any $\omega \leq 0 $, $\lim_{k \rightarrow \infty} \PP_k(\omega) = 0$.
			\item For any sequence $\{\omega_k\}$, $\lim_{k \rightarrow \infty} \PP_k(\omega_k) < + \infty$ implies that $\limsup_{k \rightarrow \infty} \omega_k \leq 0 $.
		\end{enumerate}}
		
		%
		%
	\begin{proof}
		From the definitions of 
		penalty and barrier functions (see, e.g., Definition~\ref{defi:barrier}), the statement (1) follows immediately.	
		
		When $\rho$ is a penalty function, $\lim_{\sigma \rightarrow 0} \rho(\omega; \sigma)$ is equal to $ + \infty$ for $\omega > 0$ and $0$ for $\omega \leq 0$. 
		Hence, as $k \rightarrow \infty$, $\sigma_k \rightarrow 0$, and we have $\PP_k(\omega) \rightarrow 0$, for $\omega \leq 0 $. 
		For any sequence $\{\omega_k\}$, if $\limsup_{k \rightarrow \infty} \omega_k > 0 $, there exists a subsequence $\{\omega_t\}$ of $\{\omega_k\}$ and $\varepsilon > 0$ such that $\omega_t \ge \varepsilon$ for all $t$. Then, it follows from the monotonicity of $\rho$ that $\lim_{t \rightarrow \infty} \PP_t(\omega_t)  = \lim_{t \rightarrow \infty} \rho(\omega_t; \sigma_t) \ge  \lim_{t \rightarrow \infty} \rho(\varepsilon; \sigma_t) = + \infty$. Thus,
		$\lim_{k \rightarrow \infty} \PP_k(\omega_k) < + \infty$ implies that $\limsup_{k \rightarrow \infty} \omega_k \leq 0 $.
		
		If $\rho$ is a modified barrier function, since $\rho$ is non-creasing, we have  $0 \leq \rho(\omega - \eta_k ; \sigma_k^{(1)}) \leq \rho(- \eta_k ; \sigma_k^{(1)})$
		when $\omega \leq 0$. 
		The assumption $\lim_{k \rightarrow \infty}\rho( - \eta_k ; \sigma_k^{(1)}) = 0$ implies $\rho(\omega - \eta_k ; \sigma_k^{(1)}) \rightarrow 0$
		when $\omega \leq 0$.
		Hence, as $k \rightarrow \infty$, we have $\sigma_k^{(1)} \rightarrow 0$, $\eta_k \rightarrow 0$,
		and $\PP_k(\omega) \rightarrow 0$, for $\omega \leq 0 $.  For any sequence $\{\omega_k\}$, if $\lim_{k \rightarrow \infty} \PP_k(\omega_k) < + \infty$, then it follows from the definition of modified barrier function that $\omega_k \le \eta_k$ and (3) follows immediately from $\eta_k \rightarrow 0$.
		
	\end{proof}

	\subsection{Lemma 2}	
		\textit{Let $\{ (\mu_k, \sigma_{k}) \}$ be a positive sequence such that $ (\mu_k, \sigma_{k}) \rightarrow 0$, also satisfying the same setting as in Lemma~\ref{lem1}.
		Then for any sequence $\{\x_k\}$ converging to $\bar{\x}$,
		\[
		\limsup_{k \rightarrow \infty} f_{k}^*(\x_k) \le f^*(\bar{\x}).
		\]}
	\begin{proof}
		Given any $\epsilon > 0$, there exists \textcolor{black}{$\bar{\y} \in \mathbb{R}^n$} such that $f(\bar{\x},\bar{\y}) < f^*(\bar{\x}) + 1/2\ \epsilon$, and $h(\bar{\x},\bar{\y}) \leq 0$. 
		If $h(\bar{\x},\bar{\y}) =~0$, by Assumption~\ref{assumption 1}.(4), 
		the minimum of $h$ w.r.t. $\y$ in any neighbourhood of $\bar{\y}$ is smaller than 0,
		so we can find a $\hat{\y}$ close enough to $\bar{\y}$ such that $f(\bar{\x},\hat{\y}) < f^*(\bar{\x}) + \epsilon$, and $h(\bar{\x},\hat{\y}) \leq -\delta$ for some $\delta > 0$.
		If $h(\bar{\x},\bar{\y}) < 0$, such $\hat{\y}$ exists~obviously.
		
		As $\{\x_k\}$ converges to $\bar{\x}$, 
		$h(\bar{\x}, \hat{\y}) \leq -\delta$ combining with the continuity of $h(\x,\y)$ implies the existence of $K_1 > 0$ such that $h(\x_k, \hat{\y}) \leq -\delta/2$ for all $k \ge K_1$. Since the barrier function is non-decreasing, it follows that $\PP_{B,k}(h(\x_k,\hat{\y})) \le \rho( -\delta/2; \sigma_k^{(1)})$. Then $\lim_{k \rightarrow \infty} \rho( -\delta/2; \sigma_k^{(1)}) = 0$ yields 
		\[
		\lim_{k \rightarrow \infty} \PP_{B,k} \! \left( h(\x_k,\hat{\y}) \right) = 0.
		\]
		Next, as $\{\x_k\}$ converges to $\bar{\x}$, 
		it follows from the continuity of $f(\x,\y)$ and $\mu_k \rightarrow 0$
		that there exists $K_2 \ge K_1$, such that for any $k \ge K_2$,
		\[
		\begin{aligned}
			f_{k}^*(\x_k) 
			\le &f(\x_k,\hat{\y}) + \PP_{B,k} \! \left( h(\x_k,\hat{\y}) \right) + \frac{\mu_{k}}{2}\|\hat{\y}\|^2  \\
			\le &f(\bar{\x},\hat{\y}) + \epsilon \\
			\le &f^*(\bar{\x}) + 2\epsilon.
		\end{aligned}
		\]
		By letting $k \rightarrow \infty$, we obtain
		\[
		\limsup_{k \rightarrow \infty} f_{k}^*(\x_k) \le f^*(\bar{\x}) + 2 \epsilon,
		\]
		and taking $\epsilon \rightarrow 0$ to the above yields the conclusion.
	\end{proof}

	\subsection{Lemma 3}
		\textit{\textcolor{black}{	
			Let $\{(\mu_k, \theta_k, \sigma_k)\}$ be a positive sequence such that $\lim_{k \rightarrow \infty}(\mu_k, \theta_k, \sigma_k) = 0$,
			and
			satisfy the same setting as in Lemma~\ref{lem1}.
			Given $\bar{\x}\in \X$, then for any sequence $\{ \x_k\}$ converging to $\bar{\x}$, we have
		}
		\begin{equation*}
			\liminf\limits_{k\rightarrow\infty} \varphi_{k}(\x_k)\geq \varphi(\bar{\x}).
		\end{equation*} }
	\begin{proof}
		We assume by contradiction that there exists $\bar{\x} \in \X$ and a sequence $\{ \x_k \}$,
		satisfying $\x_k \to\bar{\x}$ as $k \to \infty$ with the following inequality
		\begin{equation*}
			\lim_{k \rightarrow \infty} \varphi_{k}(\x_k) < \varphi(\bar{\x}).
		\end{equation*}
		Then, there exist $\epsilon > 0$ and a sequence $\{\y_k\}$ satisfying
		\begin{equation}
			\begin{aligned}
				&F(\x_k, \y_k) + \PP_{H,k} \! \left( H(\x_k, \y_k) \right)+ \PP_{h,k} \! \left( h(\x_k, \y_k) \right)\\ 
				& + \PP_{f,k} \! \left( f(\x_k, \y_k) - f_{k}^*(\x_k)   \right)  
				+ \frac{\theta_k}{2} \| \y_k \|^2
				< \varphi(\bar{\x}) - \epsilon.	 
				\label{lem3_eq1.5}
			\end{aligned}
		\end{equation} 
		Since $F(\x,\y)$ is level-bounded in $\y$ locally uniformly in $\bar{\x}$, we have that $\{ \y_k \}$ is bounded. 
		Take a subsequence $\{\y_t\}$ of $\{\y_k\}$ which satisfies there exists $\hat{\y}$, 
		such that $\y_t \rightarrow \hat{\y}$. 
		
		The inequality Eq.~\eqref{lem3_eq1.5} yields that 
		\[
		\begin{aligned}
			\PP_{f,t} \! \left( f(\x_t, \y_t) - f_{t}^*(\x_t)   \right) < \varphi(\bar{\x}) - \epsilon -  F(\x_t, \y_t).
		\end{aligned}		
		\]
		Taking $t \rightarrow \infty$  
		then $\lim_{t\rightarrow\infty}  \PP_{f,t} \! \left( f(\x_t, \y_t) - f_{t}^*(\x_t) \right) < + \infty$.
		From Lemma \ref{lem1}, we have $\limsup_{t \rightarrow \infty} \left\{ f(\x_t, \y_t) - f_{t}^*(\x_t) \right\} \le 0$, and hence by the continuity of $f$,
		\[
		\lim_{t \rightarrow \infty} f(\x_t, \y_t) \le \liminf_{t \rightarrow \infty}  f_{t}^*(\x_t).
		\]
		Then, by the continuity of $f$ and Lemma \ref{lem2}, we have
		\[
		f(\bar{\x},\hat{\y}) =	\lim_{t \rightarrow \infty} f(\x_t, \y_t) \le \limsup_{t \rightarrow \infty} f_{t}^*(\x_t) \le f^*(\bar{\x}).
		\]
		By using similar arguments and the continuity of $h$ and $H$, one can show 
		$
		h(\bar{\x},\hat{\y}) \le 0$ and $H (\bar{\x},\hat{\y}) \leq 0.
		$
		Thus, we have
		$
		\hat{\y} \in S(\bar{\x}),
		$
		and	$\hat{\y}$ is a feasible point to problem Eq.~\eqref{eq:constrain varphi} with $\x = \bar{\x}$. 
		Then Eq.~\eqref{lem3_eq1.5} yields
		\begin{equation*}
			\varphi(\bar{\x}) \le F(\bar{\x},\hat{\y}) \le \limsup_{k \rightarrow \infty} F(\x_k, \y_k) \le \varphi(\bar{\x}) - \epsilon,
		\end{equation*}
		which implies a contradiction.	
		Thus we get the conclusion. 
	\end{proof}

	\subsection{Lemma 4}
		\textit{\textcolor{black}{
			Let $\{(\mu_k, \theta_k, \sigma_k)\}$ be a positive sequence such that $\lim_{k \rightarrow \infty}(\mu_k, \theta_k, \sigma_k) = 0$,
			and
			satisfy the same setting as in Lemma~\ref{lem1}.
			Then for any $\x \in \X$, }
		\[
		\limsup_{k \rightarrow \infty} \varphi_k(\x) \le \varphi(\x).
		\]}
	\begin{proof}
		Given any $\bar{\x} \in \X$, for any $\epsilon > 0$, 
		there exists \textcolor{black}{$\bar{\y} \in \mathbb{R}^n$} satisfying $f(\bar{\x},\bar{\y}) \le f^*(\bar{\x})$, $h(\bar{\x}, \bar{\y}) \leq 0$, $H(\bar{\x},\bar{\y}) \leq 0$, 
		and $F(\bar{\x},\bar{\y}) \le \varphi(\bar{\x}) + \epsilon$. 
		
		By the definition of $\varphi_k$, we have
		\begin{equation}\label{lem3_eq1}
			\begin{aligned}
				\varphi_k(\bar{\x} ) 
				\le & F(\bar{\x} ,\bar{\y}) + \PP_{H,k} \! \left( H(\bar{\x} ,\bar{\y}) \right)
				+ \PP_{h,k} \! \left( h(\bar{\x} ,\bar{\y}) \right) \\ 
				& +  \PP_{f,k} \! \left( f(\bar{\x} ,\bar{\y}) - f_{k}^*(\bar{\x} )   \right)
				+ \frac{\theta_k}{2} \| \bar{\y} \|^2 .
			\end{aligned}
		\end{equation}
		From Lemma~\ref{lem1}, as $k \rightarrow \infty$, 
		we have
		$\PP_{H,k}\left( H(\bar{\x}, \bar{\y}) \right) \rightarrow 0$,
		and $\PP_{h,k} \! \left( h(\bar{\x} ,\bar{\y}) \right) \rightarrow 0$.
		As we choose the standard barrier function for $\PP_{B,k}$ in the definition of $f^*_{k}(\x)$,
		thus $f(\x,\y) +  \PP_{B,k} \! \left( h(\x,\y) \right) + \frac{\mu_k}{2} \| \y \|^2$ is always larger than $f(\x,\y)$ for any $\x$ and $\y$ feasible to the LL problem,
		and hence $f_{k}^*(\bar{\x}) \ge f^*(\bar{\x})$.
		Then we have $f(\bar{\x},\bar{\y}) - f_{k}^*(\bar{\x}) \le f(\bar{\x},\bar{\y}) - f^*(\bar{\x}) \le 0$. 
		Then it follows from the monotonicity of $\PP_{f,k} $ and Lemma~\ref{lem1} that $ \PP_{f,k} \! \left( f(\bar{\x} ,\bar{\y}) - f_{k}^*(\bar{\x} )   \right)  \rightarrow 0$.

		Therefore, as $\theta_k \rightarrow 0$, 
		by taking $k \rightarrow \infty$ in inequality Eq.~\eqref{lem3_eq1}, we have
		\[
		\limsup_{k \rightarrow \infty}\varphi_k(\bar{\x}) \le \varphi(\bar{\x}) + \epsilon.
		\]
		Then, we get the conclusion by letting $\epsilon \rightarrow 0$.
	\end{proof}

	\section{Closed-form Solution in Section~\ref{sec:8.1}}\label{sec:appendix closed-form solution}
	
{\color{black}
	
	Here we provide the detailed derivation of the closed-form solutions for the numerical examples in Section~\ref{sec:8.1}.
	
	\subsection{Optimistic BLO in Section~\ref{sec:toy optimistic}}
	
	We consider the following optimistic BLO:
	\begin{equation*}
		\begin{aligned}
			& \min_{ \x \in \mathbb{R}, \y \in \mathbb{R}^n} \| \x-a\|^2+\| \y-a-\c\|^2 \\
			& \text{\ s.t.\ }\;  [\y]_i \in \underset{ [\y]_{i} \in \mathbb{R}}{\mathrm{argmin}}\; \sin( \x+ [\y]_i-[\c]_i), \forall \ i,
		\end{aligned}
	\end{equation*}
	where $[\y]_i$ denotes the $i$-th component of $\y$, 
	while $a \in \mathbb{R}$ and $\c\in \mathbb{R}^n$ are adjustable parameters. 
	Note that here in the numerical example, $\x \in \mathbb{R}$ is a one-dimensional real number,
	and we still use the bold letter to represent the scalar in order to maintain the consistency of the context.
	The solution of such problem is
	$$
	\x^*=\frac{(1-n) a+n C}{1+n}, \ \text { and } \ [\y^*]_{i}=C+[\c]_i-\x, \forall \ i,
	$$
	where
	$$
	C=\underset{{k}}{\operatorname{argmin}}\left\{\|C_k-2 a\| : C_k=-\frac{\pi}{2}+2 k \pi, k \in \mathbb{Z}\right\} ,
	$$
	and the optimal value is $F^{*}=\frac{n(C-2 a)^{2}}{1+n}$.	
	Derivation of the optimal solution and optimal value is as follows.

	From the LL problem $$[\y]_i \in \underset{ [\y]_{i} \in \mathbb{R}}{\mathrm{argmin}}\; \sin( \x+ [\y]_i-[\c]_i), \forall \ i,$$
	we have $[\y]_i  \in \left\{ -\x + [\c]_i   -\frac{\pi}{2} + 2 k \pi: k \in \mathbb{Z} \right\}, \forall \ i$.
	Then the problem is to find the $\x$ and $k$ to minimize
	\[
	\begin{aligned}
		F &= \| \x-a\|^2+\| \y-a-\c\|^2 \\
		&= (\x - a)^2 + \sum_{i=1}^n ( [\y]_i - a - [\c]_i)^2 \\
		&= (\x - a)^2 + n (-\x - \frac{\pi}{2} + 2 k \pi - a)^2 \\
		&= (n+1) \x^2 + 2\left[ n \left(a+ \frac{\pi}{2} - 2k\pi \right) -a \right] \x \\
		& \quad + a^2 +n\left(a + \frac{\pi}{2} - 2k\pi\right)^2.
	\end{aligned}
	\]
	
	For a given $k$, denote $C_k = - \frac{\pi}{2} + 2 k \pi$,
	then 
	\[
	F_k = (n+1) \x^2 + 2\left[ n(a - C_k) -a \right] \x + a^2 +n\left(a -C_k \right)^2,
	\]
	which is strongly convex and quadratic w.r.t. $\x$.
	Thus, it is easy to obtain
	$\x_k^* = \frac{(1-n)a + n C_k}{1+n}$, and 
	\[
	\begin{aligned}
		F_k^* &= a^2 + n(a-C_k)^2 - \frac{[n(a-C_k) - a]^2}{n+1} \\
		&= \frac{n}{n+1} (C_k - 2a)^2.
	\end{aligned}	
	\]
	
	Hence, by denoting 
	$$
	C=\underset{{k}}{\operatorname{argmin}}\left\{\|C_k-2 a\| : C_k=-\frac{\pi}{2}+2 k \pi, k \in \mathbb{Z}\right\} ,
	$$
	we have the optimal value $F^* = \frac{n}{n+1} (C - 2a)^2$,
	and the corresponding solution
	\[
	\x^* = \frac{(1-n)a + n C}{1+n} , \ \text { and } \ [\y^*]_{i}=C+[\c]_i-\x^*, \forall \ i.
	\]

	\subsection{BLO with constraints in Section~\ref{sec:toy constraint}}
	
	We consider the following constrained BLO problem with non-convex LL:
	\begin{equation*}\small
		\begin{aligned}
			& \min _{\x \in \mathbb{R}, \y \in \mathbb{R}^{n}} \|\x-a\|^{2}+\left\|\y-a\right\|^{2} \\ 
			& \text { s.t. } [\y]_{i} \in \mathop{\arg\min}_{[\y]_i \in \mathbb{R}} \Big\{ \sin \left(\x+[\y]_i- [\c]_{i}\right) : \x+[\y]_i \in[0,1] \Big\},  \forall \ i,\\ 
		\end{aligned}
	\end{equation*}
	where $a \in \mathbb{R}$ and $\c \in \mathbb{R}^{n}$ are any fixed given constant and vector satisfying $ [\c ]_{i} \in[0,1] \text{ for any } i=1, \cdots, n$. The optimal solution is
	$$
	\x^*=\frac{1-n}{1+n} a, \text { and } [\y^*]_{i}=-\x^*, \forall \ i, 
	$$
	and the optimal value is $F^{*}=\frac{4 n}{1+n} a^{2}$.
	Derivation of the optimal solution and optimal value is as follows.

	From $\x+[\y]_i \in[0,1],$ along with $[\c ]_{i} \in[0,1],  \forall \ i,$
	it is easy to obtain
	\[
		\x+[\y]_i- [\c]_{i} \in \big[ \ [\c]_i, 1-[\c]_i \ \big] \subset [-1,1] \subset \left[ - \frac{\pi}{2}, \frac{\pi}{2} \right].
	\]
	Hence, $\sin \left(\x+[\y]_i- [\c]_{i}\right)$ is increasing w.r.t. $[\y]_i$ under the constraints for all $i$.
	Thus, from the LL problem we have
	$\x+[\y]_i - [\c]_{i} = - [\c]_{i}$,
	i.e., $[\y]_i = -\x, \forall \ i.$ 
	Then the problem is to find the $\x$ to minimize
	\[
	\begin{aligned}
		F &= \|\x-a\|^{2}+\left\|\y-a\right\|^{2} \\
		&= (\x-a)^2 + n(-\x-a)^2 \\
		&= (n+1)\x^2 + 2a(n-1)\x + (n+1)a^2.
	\end{aligned}
	\]
	Therefore, the optimal solution $\x^*=\frac{1-n}{1+n} a$,
	$[\y^*]_{i}=-\x^*, \forall \ i$,
	and by substituting~$\x^*$ into $F$,
	we have the optimal value~$F^{*}=\frac{4 n}{1+n} a^{2}$.

	\subsection{Pessimistic BLO in Section~\ref{sec:toy pessimistic}}
	
	For the pessimistic BLO we use the following example:
	\begin{equation}
		\label{}
		\begin{aligned}
			& \min_{ \x \in \mathbb{R}} \max_{\y \in \mathbb{R}^n} \| \x-a\|^2 - \| \y-a-\c\|^2 \\
			& \text{\ s.t.\ }\;  [\y]_i \in \underset{ [\y]_{i} \in \mathbb{R}}{\mathrm{argmin}}\; \sin( \x+ [\y]_i-[\c]_i), \forall \ i,
		\end{aligned}
	\end{equation}
	where $[\y]_i$ denotes the $i$-th component of $\y$, 
	while $a \in \mathbb{R}$ and~$\c\in \mathbb{R}^n$ are adjustable parameters. 	
	In our experiment, we set $a=2$ and $[\c]_i=2 \text{ for any }i = 1,2,\cdots,n$, and consider the 2-dimensional case (LL dimension $n=2$).
	The optimal solution to this problem is  
	$$(\x^*,\y^*)= \left(-2+\frac{\pi}{2}, 4 \pm \pi, 4 \pm \pi \right),$$ 
	and the optimal value is 
	$$F^* = \left(-4+\frac{\pi}{2} \right)^2 - 2\pi^2 = -\frac{7}{4} \pi^2 - 4 \pi +16.$$
	Derivation of the optimal solution and optimal value is as follows.

	From the LL problem $$[\y]_i \in \underset{ [\y]_{i} \in \mathbb{R}}{\mathrm{argmin}}\; \sin( \x+ [\y]_i-[\c]_i), \forall \ i,$$
	we have $[\y]_i  \in \left\{ -\x + [\c]_i   -\frac{\pi}{2} + 2 k \pi: k \in \mathbb{Z} \right\}, \forall \ i$.
	Then the problem is transferred to
	\[
	\min_{\x \in \mathbb{R}} \max_{k \in \mathbb{Z}} \ (\x-a)^2 - n \left[\x+a-\left( -\frac{\pi}{2} + 2k\pi \right) \right]^2.
	\]
	Denote $C_k = - \frac{\pi}{2} + 2 k \pi$.
	For a given $\x$, suppose $$\x+ a \in [C_{\widehat{k}} - \pi, C_{\widehat{k}} + \pi] \ (\widehat{k} \in \mathbb{Z}).$$
	Then to maximize $(\x-a)^2 - n \left[\x+a-\left( -\frac{\pi}{2} + 2k\pi \right) \right]^2$ w.r.t.~$k$ is to minimize $\left[\x+a-\left( -\frac{\pi}{2} + 2k\pi \right) \right]^2 = (\x + a - C_k)^2$ w.r.t.~$k$,
	and 
	\[
	\mathop{\arg\min}_k \ (\x + a - C_k)^2 = \widehat{k}.
	\]
	Thus, the problem is transformed into
	\[
		\min_{\x} \ (\x-a)^2 - n \left(\x + a - C_{\widehat{k}} \right)^2,
		\text{ if } \x+a \in [C_{\widehat{k}} - \pi, C_{\widehat{k}} + \pi],
	\]
	i.e.,
	\[
		\min_{\x \in \mathbb{R}} \varphi(\x)
	\]
	where
	\[
	\begin{aligned}
		\varphi(\x) := (\x-a)^2 &- n \left(\x + a - C_{\widehat{k}} \right)^2, \\
		&\text{ if } \x\in [-a+C_{\widehat{k}} - \pi, -a+ C_{\widehat{k}} + \pi] \ (\widehat{k} \in \mathbb{Z}).
	\end{aligned}
	\]
	It is easy to obtain that $\varphi(\x)$ is continuous on $\mathbb{R}$ (on interval endpoints $\varphi(\x) = (\x - a)^2 -n \pi^2$),
	and 
	\[
	\frac{1}{2} \varphi'(\x) = \x -a -n(\x+a- C_{\widehat{k}}),
	\]
	when $\x\in [-a+C_{\widehat{k}} - \pi, -a+ C_{\widehat{k}} + \pi] \ (\widehat{k} \in \mathbb{Z})$.

	Because we set $n=2$ and $a=2$,
	then if $\x\in [-2+C_{\widehat{k}} - \pi, -2+ C_{\widehat{k}} + \pi]$,
	\[
		\varphi(\x) := (\x-2)^2 - 2\left(\x + 2 - C_{\widehat{k}} \right)^2,
	\]
	\[
		\frac{1}{2} \varphi'(\x) = -\x-6 +2 C_{\widehat{k}}.
	\]
	Hence, 
	\[
		\frac{1}{2} \varphi'(\x) \in \left[ -4+C_{\widehat{k}} -\pi, -4+C_{\widehat{k}} +\pi \right].
	\]
	Therefore,
	\begin{enumerate}
		\item if $\widehat{k} \le 0$, then
		$$\frac{1}{2} \varphi'(\x) \le -4 + C_{\widehat{k}} +\pi \le -4+ \frac{\pi}{2} < 0, $$
		and $\varphi(\x)$ is decreasing.
		In this case, $C_{\widehat{k}} \le - \frac{\pi}{2}$, and $\x \le -2+ \frac{\pi}{2}$.
		
		\item if $\widehat{k} \ge 2$, then
		$$\frac{1}{2} \varphi'(\x) \ge -4 + C_{\widehat{k}} -\pi \ge -4+ \frac{7\pi}{2} > 0, $$
		and $\varphi(\x)$ is increasing.
		In this case, $C_{\widehat{k}} \ge \frac{7\pi}{2}$, and $\x~\ge~-2+ \frac{5\pi}{2}$.
		
		\item when $\x \in \left[ -2 + \frac{\pi}{2}, -2 + \frac{5\pi}{2}  \right]$,
		the corresponding $\widehat{k} = 1$, and
		\[
		\varphi(\x) := (\x-2)^2 - 2\left(\x + 2 - \frac{3\pi}{2} \right)^2,
		\]
		\[
		\frac{1}{2} \varphi'(\x) = -\x-6 + 3\pi.
		\]
		Thus, when $\x \in \left( -2 + \frac{\pi}{2}, 3\pi -6  \right)$, $ \varphi'(\x) > 0$, and $\varphi(\x)$ is increasing;
		when $\x \in \left( 3\pi-6, -2 + \frac{5\pi}{2}  \right)$, $ \varphi'(\x) < 0$, and $\varphi(\x)$ is decreasing.
	\end{enumerate}

	To sum up, $\varphi(\x)$ is increasing on $\left[-2 + \frac{\pi}{2}, 3\pi -6 \right]$ and $\left[-2 + \frac{5\pi}{2}, +\infty \right)$,
	and decreasing on $ \left(-\infty, -2 + \frac{\pi}{2} \right]$ and $\left[3\pi-6, -2 + \frac{5\pi}{2} \right] $.
	Therefore, the optimal $\x$ occurs at either $-2 + \frac{\pi}{2}$ or $-2 + \frac{5\pi}{2}$.
	Since at these two local minimizers,
	\[
	\begin{aligned}
		&\varphi\left(-2+\frac{\pi}{2}\right) = \left(-4+\frac{\pi}{2} \right)^2 - 2\pi^2 	\\
		< \ &\varphi \left(-2+\frac{5\pi}{2} \right) = \left(-4+\frac{5\pi}{2} \right)^2 - 2\pi^2 ,
	\end{aligned}	
	\]
	we have $\x^* = -2+\frac{\pi}{2}$,
	and $$F^* = \left(-4+\frac{\pi}{2} \right)^2 - 2\pi^2 = -\frac{7}{4} \pi^2 - 4 \pi +16.$$
	For $\x^* = -2+\frac{\pi}{2}$, 
	it means $\x^* \in \left[-2-\frac{\pi}{2}-\pi, -2-\frac{\pi}{2}+\pi \right]$
	and $\x^* \in \left[-2+\frac{3\pi}{2}-\pi, -2+\frac{\pi}{2}+\pi \right]$,
	so $C_{\widehat{k}} = - \frac{\pi}{2}$ or $\frac{3\pi}{2}$, i.e., $\widehat{k} = 0$ or $1$.
	Thus, 
	\[
		[\y^*]_i = -\x^* + [\c]_i - \frac{\pi}{2} = 4-\pi
	\]
	or
	\[
		[\y^*]_i = -\x^* + [\c]_i + \frac{3\pi}{2} = 4+\pi
	\]
	for $i = 1,2$.
	Hence, 
	$$(\x^*,\y^*)= \left(-2+\frac{\pi}{2}, 4 \pm \pi, 4 \pm \pi \right).$$

} 

%
%
%
%
%

\end{document}